\documentclass{article}

\usepackage[main, final, nonatbib]{neurips_2026}

\usepackage[T1]{fontenc}    
\usepackage{url}            
\usepackage{booktabs}       
\usepackage{amsfonts}       
\usepackage{nicefrac}       
\usepackage{microtype}      
\usepackage{xcolor}         
\usepackage[numbers,sort&compress]{natbib}
\usepackage{bm}
\usepackage{bbm}
\usepackage{amsmath}
\usepackage{pifont}
\usepackage{graphicx}
\usepackage{enumitem}
\usepackage{fontspec}
\usepackage{amsthm}
\usepackage{amsfonts}
\usepackage{algorithm}
\usepackage{algpseudocode}
\newtheorem{theorem}{Theorem}
\usepackage{subcaption}
\usepackage{newtxtext,newtxmath}
\usepackage{xcolor}
\usepackage[colorlinks]{hyperref}
\usepackage{multirow}
\usepackage{colortbl}
\usepackage[most]{tcolorbox}
\usepackage{tabularx}
\usepackage{array}
\usepackage{mathtools}
\usepackage{adjustbox}
\usepackage{wrapfig}
\usepackage{tikz}
\newcommand{\circnum}[1]{%
  \tikz[baseline=(char.base)]{
    \node[shape=circle, fill=black, text=white,
          inner sep=0.9pt, font=\scriptsize\sffamily\bfseries] (char) {#1};
  }%
}
\usepackage{float}
\usepackage{booktabs,tabularx}
\newcolumntype{Y}{>{\raggedright\arraybackslash}X}

\newtheorem{proposition}{Proposition}
\newtheorem{assumption}{Assumption}

\definecolor{MyCustomBlue}{HTML}{92CE6A}
\definecolor{MyCustomPink}{HTML}{C7513E}
\hypersetup{
    citecolor=MyCustomBlue,
    linkcolor=MyCustomPink,
}

\title{DRIVE: Fine-tuning via Data Contribution- and Diversity-aware Weighting with Prior Regularization}

\author{%
  \textbf{Qing Liu}\textsuperscript{1} \quad
  \textbf{Xinrui Chen}\textsuperscript{1} \quad
  \textbf{Yi Du}\textsuperscript{1,2*} \quad
  \textbf{Ou Wu}\textsuperscript{1} \quad
  \textbf{Weiyao Zhu}\textsuperscript{3} \\[0.08in]
  \textsuperscript{1}Hangzhou Institute for Advanced Study,\\
  University of Chinese Academy of Sciences, Hangzhou, China\\
  \textsuperscript{2}Computer Network Information Center,
  Chinese Academy of Sciences, Beijing, China\\
  \textsuperscript{3}School of Mathematics, Tianjin University,
  Tianjin, China\\[0.05in]
  \texttt{\{liuqing25,chenxinrui25\}@mails.ucas.ac.cn} \\[0.05in]
\texttt{duyi@cnic.cn} \quad 
\texttt{wuou@ucas.ac.cn} \quad \texttt{weiyaozhu@tju.edu.cn}
}

\begin{document}

\maketitle

\begingroup
\renewcommand{\thefootnote}{\fnsymbol{footnote}}
\footnotetext[1]{Corresponding author.}
\endgroup

\begin{abstract}

Single-stage fine-tuning is pervasive, yet catastrophic forgetting persists when the source pre-training data is inaccessible. Data weighting affords a natural lever to modulate sample influence. Yet existing weighting methods lack a principled mechanism to explicitly arbitrate knowledge retention against new-task adaptation, and provide no unified differentiable objective to accommodate auxiliary signals such as sample diversity and sample characteristics. This is because current schemes privilege one goal over the other, and no explicit gradient-compatible encoding exists for such signals. To address this, we propose $\textbf{DRIVE}$, a principled and efficient weight assignment framework. It derives per-sample weights via a unified differentiable objective where three coupled signals co-evolve: Shapley additivity splits contribution into anti-forgetting and adaptation components, which are reshaped by a covariance-based diversity term and stabilized by a sample-characteristic prior. Far from a mere combination, this yields an end-to-end gradient-based scheme balancing stability against plasticity. Theoretically, we prove that DRIVE enjoys superior stability-plasticity guarantees over existing weighting schemes and enhances parameter estimation robustness. Empirically, it attains superior overall performance, faster convergence, and stronger robustness. Code at \url{https://github.com/luciabcd/DRIVE}.

\end{abstract}

\section{Introduction}
\label{submission}
Fine-tuning is a key paradigm for improving downstream  performance in modern deep learning~\cite{he2016deep,devlin2019bert}. Recent studies show that even in a single round of fine-tuning, catastrophic forgetting can occur~\cite{wu2026mitigating,wang2026on,harmon2026mapping,lin2026sft}. Because the original pre-training data is typically unavailable in practice, prior work has explored data-oblivious strategies in both parameter space and sample space. Parameter-space methods mitigate forgetting by constraining update magnitudes~\cite{lubana2022quadratic,chen2024mofo}, updating only a subset of parameters, introducing additional parameter modules~\cite{panda2024lottery,hu2022lora,biderman2024lora}, or using model merging~\cite{ilharco2023editing,kleiman2025soup}. Sample-space methods mainly include sample weighting, selection, and augmentation~\cite{li2025meta,thakkar-etal-2023-self,sow2025dynamic}.

Data weighting enables fine-grained control over sample influence by upweighting high-contribution samples~\cite{zhao2026rethinking,zhou2022understanding,shu2023cmw}. In practice, most studies measure contribution by a sample's effect on new task learning, for example through weighting based on validation loss, gradient magnitude, or sample loss~\cite{zhao2026rethinking,li2025meta,thakkar-etal-2023-self,sow2025dynamic}, whereas FLOW measures it by the sample's role in preserving pre-trained knowledge during fine-tuning~\cite{sanyal2025upweighting}. While effective, these methods still ground sample contribution in either adaptation-oriented or retention-oriented signals, rather than explicitly balancing stability and plasticity, and thus they cannot fully characterize each sample's contribution to the fine-tuning.

Beyond sample contribution, sample diversity and sample characteristics have also been widely considered in practice. 
Regarding sample diversity, existing studies usually rely on clustering, DPP, or similarity coverage for selection~\cite{liu2024what,das2024deft,li2025unlocking,zhang2025survey,wang2025diversity}, making it difficult to encode diversity in a continuous weighting objective. For sample characteristics, studies typically use signals such as sample length and entropy to reflect the  quality or complexity of samples~\cite{liu2024selectit,li-etal-2024-one,li-etal-2024-quantity,luo2025semi,zhang2025d3}, and some directly treat them as proxies for sample contribution to determine sample weights~\cite{han2022umix,zhang2025dorm}. However, these signals are inherently task-agnostic properties with no direct correspondence to a sample's actual contribution during fine-tuning. We therefore treat them as complementary to sample contribution rather than substitutes for it. Overall, within sample weighting, existing methods usually incorporate only part of the above information into weight assignment. More importantly, because diversity is difficult to encode as a differentiable weighting term, existing sample-weighting methods lack a unified differentiable objective that jointly incorporates sample contribution, sample diversity, and sample characteristics. This not only makes it difficult to realize an end-to-end weighting scheme, limiting full data utilization, but also hinders a more principled analysis of optimization behavior and convergence.

\begin{figure*}[t]
\centering
\includegraphics[width=\textwidth]{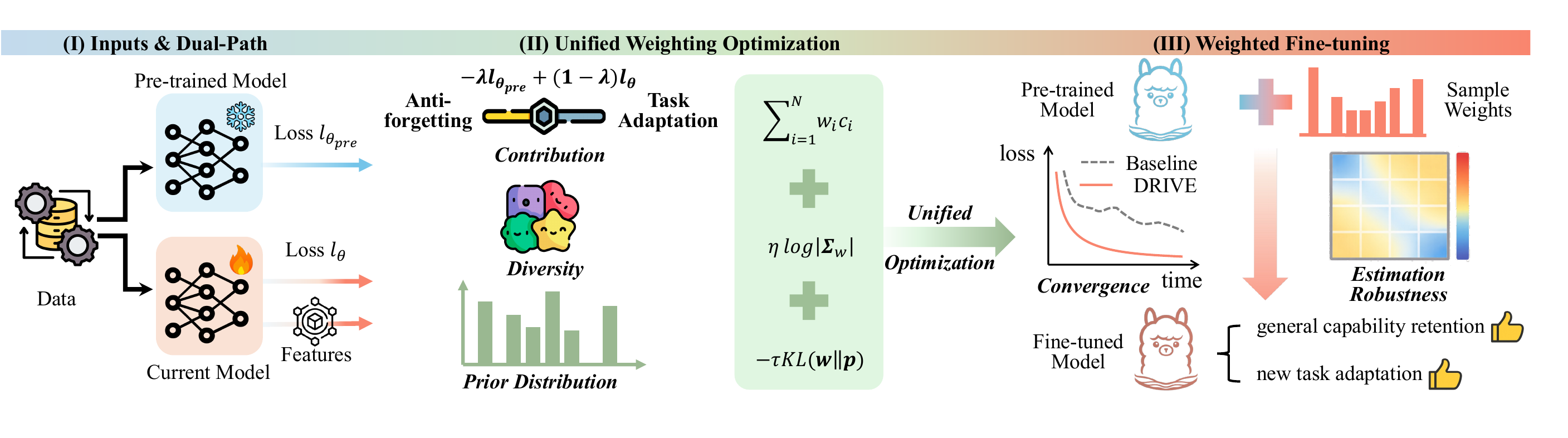}
\caption{Overview of DRIVE's unified differentiable data weighting objective.}
\label{fig:DRIVE_framework}
\vspace{-0.1in}
\end{figure*}

To address the above issues, we propose \textbf{DRIVE} (Fine-tuning via \textbf{D}ata Cont\textbf{RI}bution- and Di\textbf{V}ersity-aware W\textbf{E}ighting with prior Regularization), a principled and efficient weighting framework. Rather than stacking existing signals, \textbf{DRIVE reformulates each ingredient for joint compatibility within one differentiable objective}. Sample contribution is recast through a Shapley-additivity decomposition that disentangles anti-forgetting from task adaptation and yields a tunable stability--plasticity trade-off. Diversity is made differentiable via a covariance-determinant term, while sample characteristics serve as a KL-anchored prior over weights. These reformulations enable end-to-end weighting. Theoretically, DRIVE enjoys improved stability--plasticity guarantees and enhanced parameter-estimation robustness. Empirically, across vision models and LLMs, it attains stronger overall performance. Fig.~\ref{fig:DRIVE_framework} illustrates the overall framework. Our key contributions are summarized as follows:



\begin{itemize}[leftmargin=*,nosep]
\item We propose DRIVE, a novel weighting framework that jointly reformulates components rather than off-the-shelf aggregation. Shapley additivity yields an anti-forgetting and task-adaptation decomposition; a covariance-determinant term renders diversity differentiable, to our knowledge the first such term for sample weighting; and KL anchors weights to a sample-characteristic prior.
    \item We develop a provably convergent iterative solver for the coupled DRIVE objective. Theoretical analysis in linear/local nonlinear regimes shows contribution decomposition tightens the stability--plasticity trade-off, while diversity provably accelerates convergence and enhances robustness.

    \item Across vision models and LLMs, DRIVE consistently outperforms strong baselines, with better overall performance, faster convergence, and stronger robustness.
\end{itemize}

\section{Related Work}
\textbf{Mitigating Catastrophic Forgetting in Fine-tuning.}
Fine-tuning is pivotal for pre-trained models~\cite{hu2023llm,wu2025llm,kuang2024federatedscope}, yet balancing new skills with prior capabilities remains challenging. To mitigate catastrophic forgetting, methods target parameter or sample spaces~\cite{chen2026algorithmgoalsdualscoring,zhang2026nads}: regularization constrains weight drift~\cite{kirkpatrick2017overcoming,schwarz2018progress,aljundi2018memory,lee2020continual,sliwa2025mitigating}; optimization ensures orthogonal updates~\cite{chaudhry2018efficient,zeng2019continual,farajtabar2020orthogonal}; architectures isolate parameter subsets~\cite{rusu2016progressive,wang2023incorporating,gurbuz2022nispa,jung2020continual}; and replay injects past samples~\cite{aljundi2019gradient,bang2021rainbow,tiwari2022gcr,de2021continual}. In contrast, we introduce a data weighting perspective, optimizing the trade-off between anti-forgetting and task-adaptation.

\textbf{Data Weighting in Fine-tuning.}
Recent data weighting strategies  adopt fine-grained evaluations. Existing works optimize training dynamics via loss alignment, meta-learning, or robust dynamic weighting to mitigate noise and forgetting~\cite{zhou2022understanding, shu2023cmw, chen2024take, sow2025dynamic}. Others utilize quality proxies, such as reward scores, consistency metrics, or information entropy, to guide data selection or weighting~\cite{mukherjee2025learning, kuo2025not, liu2024selectit, luo2025semi, han2025automatic, zhang2025d3}. Further studies isolate fine-tuning contributions through distribution shifts, trajectory analysis, or orthogonal model decomposition~\cite{kang2024get, zeng2025lensllm, ju2024mitigating, nuti2025tuco}. Unlike these isolated perspectives, we propose a unified weighting objective that jointly models data contribution, diversity, and sample prior.

\textbf{Data Diversity in Fine-tuning.}
Data diversity is vital for model generalization and handling unseen semantics~\cite{stap2025effect,wang2025multilingual,li2024entropic,zhang2025diversification}, supporting performance across contexts.
Existing strategies rely on discrete data selection, employing clustering-based sampling to prune redundancy~\cite{li2025unlocking}, coreset selection~\cite{chai2023goodcore,das2024deft}, determinantal selection maximizing volume~(DPP)~\cite{wang2025diversity}, or complexity-aware filtering like DEITA~\cite{liu2024what,zhang2025survey}.
Others emphasize response-level diversity for alignment~\cite{li2025macro,song2024scaling,chen2025enhancing}.
Unlike these, we integrate a covariance determinant regularizer into the loss-weighting objective.

\section{Methodology}
\label{sec:method}
We consider a realistic deployment scenario: the single-task, data-oblivious setting of fine-tuning on a dataset $\mathcal{D} = \{(\bm{x}_i, \bm{y}_i)\}_{i=1}^N$ without access to pre-training data or continual task sequences (unlike continual learning with sequential task boundaries). We denote the parameters of pre-trained, current, and ideal models by $\bm{\theta}_\text{pre}$, $\bm{\theta}$, and $\bm{\theta}^*$, respectively. With model output $f(\bm{x}_i,\bm{\theta})$ and loss $\ell(\cdot,\cdot)$, we assign a weight $w_i (\ge 0)$ to $\bm{x}_i$, satisfying $\sum_i w_i = 1$. This section infers the optimal weights via the DRIVE method, which are then employed to fine-tune the current model as shown in Fig.~\ref{fig:DRIVE_framework}.

\subsection{Data Contribution and Diversity Measures}
\label{sec:contribution}

\textbf{Measuring Contribution via Shapley-inspired Scores.}
Data contribution quantifies the utility of individual samples for task performance.
 We disentangle fine-tuning contribution as balancing two distinct cooperative games over dataset
$\mathcal{D}$: the anti-forgetting game $v_{\text{pre}}$, whose utility encodes 
retention of pre-trained knowledge, and the task-adaptation game $v_{\text{task}}$, encoding  target-task performance.

Let \(\phi_i(v)\) denote the Shapley value assigned to sample \(\bm{x}_i\) under game \(v\). By the \textbf{Additivity Axiom of Shapley values}~\citep{shapley1953value},
the contribution of sample $\bm{x}_i$ in the composite game
$v_{\lambda} = \lambda v_{\text{pre}} + (1-\lambda)v_{\text{task}}$ with
trade-off coefficient $\lambda \in [0,1]$ must satisfy
\begin{equation}
    \phi_i(v_{\lambda})
    =
    \lambda \, \phi_i(v_{\text{pre}})
    +
    (1-\lambda) \, \phi_i(v_{\text{task}}).
    \label{eq:additivity}
\end{equation}
This axiomatically justifies linear aggregation. We denote the anti-forgetting
contribution by $a_i = \phi_i(v_{\text{pre}})$ and the task-adaptation contribution
by $t_i = \phi_i(v_{\text{task}})$, and define the contribution as
\begin{equation}
\label{eq:joint-contrib2}
c_i = \mathcal{C}(a_i, t_i) \coloneqq \lambda\, a_i + (1-\lambda)\, t_i.
\end{equation}
Estimating \(a_i\) and \(t_i\) requires a per-sample contribution measure. The canonical
instrument, gradient inner products against a held-out validation
set~\citep{wang2025data} (see Appendix~\ref{app:why_not_shapley}), is foreclosed by
our strict data-oblivious constraint, where neither pre-training data nor original
validation sets are available. We therefore use a loss-based proxy computable from the
model and sample alone. For cross-entropy objectives, loss measures the model's
deviation from the target and provides a scalar, summary of the
optimization signal carried by a sample. Thus, near-zero-loss samples typically induce
limited further updates, whereas high-loss samples remain active learning signals~\cite{chen2024take}.
Empirically, loss-based estimates correlate with validation-based Shapley contributions
at Spearman's \(\rho = 0.83\) and match gradient-based variants within \(0.3\%\) in
accuracy while running \(\sim 2\times\) faster
(Appendices~\ref{app:proxy_support} and~\ref{app:alternative_metrics}). We therefore
adopt the efficient contribution estimates
\begin{equation}
\label{eq:utilities}
a_i = -\ell(f(\bm{x}_i,\bm{\theta}_{\text{pre}}), \bm{y}_i),
\qquad
t_i = \ell(f(\bm{x}_i,\bm{\theta}), \bm{y}_i).
\end{equation}
\textbf{Measuring Diversity via the Covariance Determinant.}
Diversity captures set-level characteristics across samples and is critical for reducing redundancy and improving downstream generalization~\citep{liu2024what,das2024deft,li2025unlocking,zhang2025survey}. The most natural choice, clustering-based selection, is discrete and non-differentiable and thus incompatible with our continuous weight optimization. Yet a differentiable, principled set-level alternative is nontrivial to construct: combinatorial coverage terms are non-smooth, pairwise-repulsion surrogates scale quadratically and conflate global coverage with local separation, and even DPPs encode set-level volume only over discrete subsets. We therefore lift this set-level volume principle to continuous data weighting via D-optimal design~\citep{chaloner1995bayesian,alexanderian2021optimal}, to our knowledge the first such adaptation. 

Under weights $\bm{w}=\{w_i\}$, we define $\bm{\Sigma}_w=\sum_{i=1}^{N} w_i \bar{\bm{x}}_i \bar{\bm{x}}_i^\top$ with $\bar{\bm{x}}_i$ the fixed, unweightedly centered feature vector from the current model, and set $\mathcal{D}(\bm{w})\coloneqq \log|\bm{\Sigma}_w|$, which tracks the evolving representation rather than frozen pre-trained features. The objective is smooth in $\bm{w}$ with closed-form gradient $\frac{\partial}{\partial w_i}\log|\bm{\Sigma}_w| = \bar{\bm{x}}_i^\top \bm{\Sigma}_w^{-1} \bar{\bm{x}}_i$, the statistical leverage of sample $i$, so upweighting yields a large marginal diversity gain precisely when its features are poorly explained by the rest. This makes $\mathcal{D}(\bm{w})$ a principled, differentiable plug-in for the weighting objective introduced next.

\subsection{Data Weight Optimization Objective via Prior Regularization}
\label{sec:prior_objective}

Beyond contribution and diversity, samples possess intrinsic characteristics such as quality and complexity that are properties of the data itself and largely task-agnostic. Folding such information into $c_i$ would conflate these sample-intrinsic properties with the task-specific contribution it is meant to capture. We instead encode them as a prior distribution $\bm{p}=\{p_i\}$ over samples, which serves as a KL-anchored reference that regularizes the weight distribution toward sample-intrinsic preferences and jointly determines sample weights together with contribution and diversity.

Given the above, data weights are obtained by solving the following optimization objective
\begin{equation}
\label{eq:weight-optim}
\{w_i^*\}
=
\arg\max_{\bm{w}} J(\bm{w}),
\quad
J(\bm{w})
=
\sum_{i=1}^{N} w_i c_i
+
\eta \log |\bm{\Sigma}_w|
-
\tau KL(\bm{w}\|\bm{p})
\quad
\text{s.t.}
\sum_{i=1}^{N} w_i = 1,\;
w_i \ge 0.
\end{equation}
Here, the first term of $J(\bm{w})$ captures sample contribution, the second promotes diversity through feature-volume expansion with strength $\eta$, and the third incorporates a prior reference distribution $\bm{p}$ with strength $\tau$. The stationarity condition of Eq.~\eqref{eq:weight-optim} admits a closed-form expression for $w_i$, but it cannot be evaluated directly because $\bm{\Sigma}_w$ on its right-hand side is itself a function of $\bm{w}$. We therefore solve this fixed-point equation by iteration, which we observe to converge within a few steps. Starting from the maximum-entropy uniform $\bm{w}^{(0)}$ with induced covariance $\bm{\Sigma}_w^{(0)}$, each iteration $t\geq 1$ uses the previous covariance $\bm{\Sigma}_w^{(t-1)}$ in place of $\bm{\Sigma}_w$ to update the weights as
\begin{equation}
\label{eq:DRIVEr_fixed_point}
w_i^{(t)}
\propto
p_i \exp\left(
\frac{
c_i+\eta\,\bar{\bm{x}}_i^\top(\bm{\Sigma}_w^{(t-1)})^{-1}\bar{\bm{x}}_i
}{\tau}
\right).
\end{equation}
The updated weights are normalized to satisfy the simplex constraint, and the procedure iterates until convergence or the maximum iteration count is reached. This update admits the interpretation of entropic mirror ascent under generalized KL regularization, with the KL term serving as the proximal regularizer. Moreover, for \(\tau>0\), \(J(\bm{w})\) is strictly concave on the simplex, so Eq.~\eqref{eq:weight-optim} admits a unique global maximizer; the fixed-point iteration converges linearly under the standard contraction condition. Detailed derivations and proofs are provided in Sec.~\ref{sec:theory_weight_convergence} and Appendix~\ref{sec:appendix_optimization}.

\begin{wraptable}{r}{0.7\columnwidth}
\vspace{-10pt}
\caption{Special cases of DRIVE under a uniform prior}
\vspace{-7pt}
\label{tab:weight_special_cases}
\centering
\scriptsize
\setlength{\tabcolsep}{2.2pt}
\renewcommand{\arraystretch}{1.05}
\begin{tabular}{@{}l >{\centering\arraybackslash}p{0.38\columnwidth} >{\centering\arraybackslash}p{0.20\columnwidth}@{}}
\toprule
Case & Weight form & Meaning \\
\midrule

$\eta=0$
& $w_i^* \propto \exp[(-\lambda l_i^{\theta_{\mathrm{pre}}}+(1-\lambda)l_i^\theta)/\tau]$
& Without diversity \\

\hspace{0.8em}- $\lambda=1$
& $w_i^* \propto \exp(-l_i^{\theta_{\mathrm{pre}}}/\tau)$
& Anti-forgetting~\citep{sanyal2025upweighting} \\

\hspace{0.8em}- $\lambda=0$
& $w_i^* \propto \exp(l_i^\theta/\tau)$
& Task adaptation~\citep{sow2025dynamic} \\

\hspace{0.8em}- $\lambda=0.5$
& $\bm{\theta}\!=\!\bm{\theta}_{\mathrm{pre}}\Rightarrow w_i^* \propto 1/N$
& Standard FT (init.) \\

\midrule

$\eta>0$
& $w_i^{(t)} \propto \exp[(c_i+\eta\,\bar{\bm{x}}_i^\top(\bm{\Sigma}_w^{(t-1)})^{-1}\bar{\bm{x}}_i)/\tau]$
& Full DRIVE \\
\bottomrule
\end{tabular}
\vspace{-1pt}
\end{wraptable}

In this study, we adopt a uniform prior (i.e., $p_i \equiv \frac{1}{N}$) as the default non-informative choice in our main experiments, since it introduces minimal assumptions. When other signals are available, however, the framework naturally allows more informative priors; recent studies suggest that sample length or entropy can serve as proxies for sample quality in sample selection~\citep{liu2024selectit, luo2025semi, zhang2025d3}, and our ablation in Sec.~\ref{sec:ablation} further supports the effectiveness of entropy-informed priors. Under a uniform prior, the objective admits several special cases (Table~\ref{tab:weight_special_cases}).

We primarily anneal $\lambda$ via a linear decay ($1\to 0$) over fine-tuning, with a task-specific constant as alternative. The decay enacts a curriculum: large $\lambda$ early preserves pretrained capability, while $\lambda\to 0$ late prioritizes task adaptation. This mirrors the recall-then-learn principle of \cite{chen-etal-2020-recall}, which anneals a pretraining-recall loss against the task loss, thereby reconciling early stabilization with late adaptation without per-task tuning (Table~\ref{tab:lambda_ablation}).

\begin{wrapfigure}{r}{0.5\columnwidth}
\centering
\includegraphics[width=0.48\columnwidth, trim=4mm 0mm 4mm 4mm, clip]{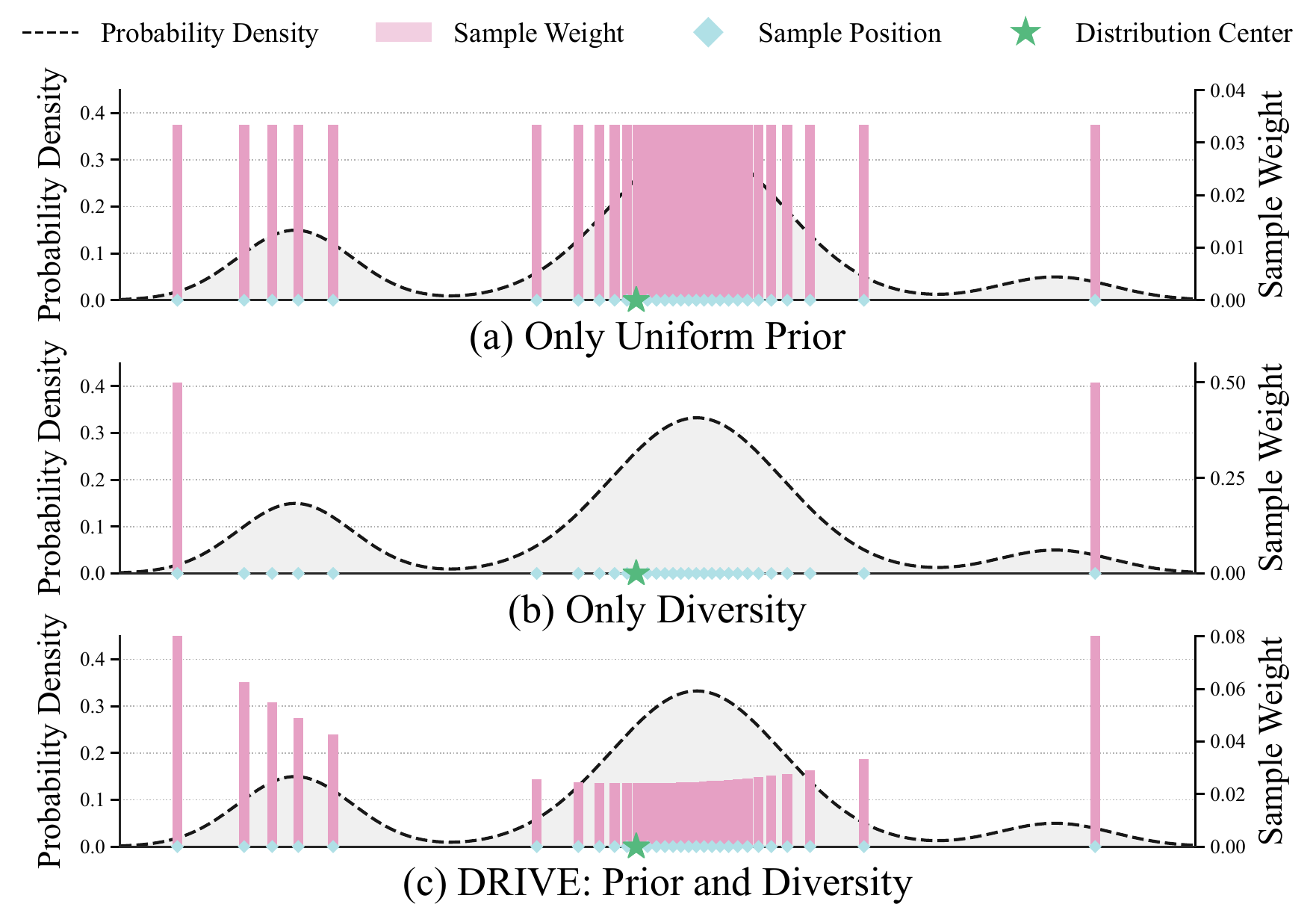}
\caption{Impact of Uniform Prior and Determinant Regularization on Weight Distribution.}
\vspace{-0.1in}
\label{fig:diversity_viz}
\end{wrapfigure}

Fig.~\ref{fig:diversity_viz} shows the impacts of the uniform prior and diversity terms, as well as their synergistic effect in a 1D scenario. In Fig.~\ref{fig:diversity_viz}(a), the prior term enforces equal sample weights, causing the model to over-concentrate on dense regions dominated by redundant samples. Conversely, maximizing the diversity term alone (Fig.~\ref{fig:diversity_viz}(b)) biases weights toward peripheral samples at the expense of the majority of data. DRIVE (Fig.~\ref{fig:diversity_viz}(c)) reconciles these tendencies by integrating both; it up-weights peripheral samples to promote diversity while anchoring to the uniform prior to ensure  coverage. We prove in Sec.~\ref{sec:theory} that incorporating diversity accelerates convergence and enhances estimation robustness. Conventional diversity strategies, such as clustering, are typically discrete and hard to incorporate into unified objectives. Our determinant-based approach offers a differentiable alternative that empirically outperforms the heuristic clustering method DEFT-UCS~\cite{das2024deft} in Sec.~\ref{sec:main_results}.

\begin{wrapfigure}{R}{0.6\columnwidth}
\vspace{-0.34in}
\begin{minipage}{\linewidth}
\begin{algorithm}[H]
\caption{DRIVE~(Given step-wise inputs $c_i$ and $\bar{\bm{x}}_i$.)}
\label{alg:reweighting}
\begin{algorithmic}[1]
\Require $\mathcal{D}=\{(\bm{x}_i,\bm{y}_i)\}_{i=1}^S$, $c_i$, $\bar{\bm{x}}_i$, $k$, $\tau$, $\eta$, $\epsilon$, $\tilde{w}_i^{(0)} \leftarrow 1/S$.
\Ensure Optimal weights $\bm{w}^*$.
\State Generate a Gaussian random matrix $\mathbf{R} \in \mathbb{R}^{d \times k}$.
\State Project features: $\tilde{\bm{x}}_i \leftarrow \mathbf{R}^\top \bar{\bm{x}}_i$ for all $i=1,\dots,S$.
\For{$t=1$ to $T_{\max}$}
    \State $\bm{\Sigma}_w^{(t-1)} \leftarrow \sum_{i=1}^S \tilde{w}_i^{(t-1)} \tilde{\bm{x}}_i \tilde{\bm{x}}_i^\top + \epsilon \mathbf{I}$
    \State Compute $w_i^{(t)}$ via Eq.~\eqref{eq:DRIVEr_fixed_point} for all $i$
    \State $\tilde{w}_i^{(t)} \leftarrow w_i^{(t)} \big/ \sum_{j=1}^S w_j^{(t)}$
    \State If converged, then break
\EndFor
\State \textbf{return} $\bm{w}^* \leftarrow \tilde{\bm{w}}^{(t)}$
\end{algorithmic}
\end{algorithm}
\end{minipage}
\vspace{-15pt}
\end{wrapfigure}

\subsection{Efficient Data Weighting for Fine-tuning}
\label{sec:efficient_reweighting_finetuning}

While the iterative solver for Eq.~(\ref{eq:weight-optim}) is theoretically convergent,
using it in fine-tuning faces a key challenge: inverting $\bm{\Sigma}_w$ costs
$\mathcal{O}(d^3)$, prohibitive for high-dimensional
$\bm{x}\in\mathbb{R}^d$.

To bound the inversion cost, we sketch the centered feature $\bar{\bm{x}}_i$,
extracted from the current model's penultimate layer (Vision)~\cite{Wei_2024_CVPR}
or final-block layer (LLM)~\cite{khosla2025magnet}, through a Gaussian random
matrix $\mathbf{R}\in\mathbb{R}^{d\times k}$ with $k\ll d$, giving
$\tilde{\bm{x}}_i=\mathbf{R}^\top\bar{\bm{x}}_i\in\mathbb{R}^k$.

The Johnson--Lindenstrauss lemma suggests that the Gaussian sketch approximately
preserves inner products and pairwise distances on the dominant low-rank local
subspace (Appendix~\ref{app:random_projection}), while reducing inversion to
$\mathcal{O}(k^3)$. $\mathbf{R}$ is generated once
and shared across all steps, and ensembling projections yields negligible gains
(Appendix~\ref{sec:appendix_ensemble}).
Since a single mini-batch $\mathcal{B}$ may not contain enough samples for
$\bm{\Sigma}_w$ to be invertible, we augment it with a FIFO buffer of the $128$
most recently seen samples, forming a local support set of size $S$ on which both
$\bm{\Sigma}_w$ and the centering mean are estimated. We set
$k=128$ to match the buffer size, aligning the sketch dimension with the local
support scale and avoiding overly aggressive compression. Sensitivity to $k$ is
reported in Appendix~\ref{app:k_sensitivity}. For numerical stability, we add a
small ridge $\epsilon\mathbf{I}$ to $\bm{\Sigma}_w$ before inversion. Algorithm~\ref{alg:reweighting} summarizes the per-step procedure. At step $s$, we compute $c_i$ and $\tilde{\bm{x}}_i$ for the current batch and reuse cached FIFO values to form the local support set, then run the projected fixed-point iteration to obtain $\bm{w}^*$. We restrict and renormalize $\bm{w}^*$ within $\mathcal{B}$, form the weighted loss, and update $\bm{\theta}^{(s)}$ to $\bm{\theta}^{(s+1)}$.

\section{Theoretical Analysis}
\label{sec:theory}

Following FLOW~\citep{sanyal2025upweighting} and related analyses of weighted fine-tuning, we study DRIVE in both linear and nonlinear regimes. The linear case serves as a standard analytical abstraction, widely adopted in prior work~\cite{deshpande2021linearized,tomihari2024understanding}. To better reflect deep fine-tuning in practice, we further extend the analysis to the nonlinear regime under local smoothness and local linearization assumptions. In addition, we analyze the iterative solver for Eq.~\eqref{eq:weight-optim} and establish the linear convergence of the resulting fixed-point iteration. The empirical results in Secs.~\ref{sec:main_results} and~\ref{convergence_and_robustness} are consistent with these theoretical insights.

\subsection{Linear and Local Nonlinear Regimes}
\label{sec:theory_regimes}

\textbf{Linear regime.}
We first consider a linear regression abstraction, following prior analyses of FLOW~\cite{sanyal2025upweighting}. Under sample weights $w(\bm{x})$, the training dynamics are governed by the weighted covariance
$\bm{\Sigma}_w=\mathbb{E}[w(\bm{x})\bm{x}\bm{x}^\top]$,
and convergence along direction $\bm{u}$ is controlled by
$\mu_{\bm{u}}=\bm{u}^\top\bm{\Sigma}_w\bm{u}$.
For directions where pre-training and fine-tuning conflict, stability-oriented weighting such as FLOW can suppress the spectrum and stall adaptation toward the ideal target $\bm{\theta}^{*}$. In contrast, DRIVE uses
$c_i=\lambda a_i+(1-\lambda)t_i$,
so decreasing $\lambda$ reallocates weight to task-relevant conflicting samples, restores learning in those directions, and reduces the fitting gap to $\bm{\theta}^{*}$ when the local approximation error is controlled. Moreover, maximizing $\log|\bm{\Sigma}_w|$ is equivalent to D-optimal design under the simplex constraint: it enlarges the weighted feature volume, reduces generalized variance, and improves the condition number of $\bm{\Sigma}_w$, thereby accelerating convergence and enhancing estimator stability. Therefore, in the linear regime, DRIVE improves both adaptation toward $\bm{\theta}^{*}$  and optimization robustness.

\textbf{Local nonlinear regime.}
We consider deep fine-tuning under local smoothness around $\bm{\theta}_{\mathrm{pre}}$. In this setting, the weighted Hessian decomposes into a Gauss--Newton component, a non-Gauss--Newton residual, and a local nonlinear remainder. The principal Gauss--Newton block is represented by the weighted matrix $\bm{G}_w(\bm{\theta})$. Stability-oriented weighting can assign insufficient mass to task-relevant conflicting directions, slowing movement toward $\bm{\theta}^{*}$. Shifting $\lambda$ from anti-forgetting toward task adaptation reallocates weight to these directions and restores the spectrum of $\bm{G}_w(\bm{\theta})$, reducing the local residual gradient toward $\bm{\theta}^{*}$ when the nonlinear remainder is controlled. Meanwhile, the diversity term $\log|\bm{\Sigma}_w|$, defined on the centered feature covariance, emphasizes under-covered feature directions and discourages redundancy. This promotes a better-covered representation geometry, which supports better-conditioned local optimization through the weighted Gauss--Newton block. 

\subsection{Convergence Analysis of the Iterative Solver for Eq.~(\ref{eq:weight-optim})}
\label{sec:theory_weight_convergence}

With the ridge-regularized covariance used in the solver, $J(\bm{w})$ is $\tau$-strongly concave on $\Delta_{N-1}$ for $\tau>0$ and therefore Eq.~\eqref{eq:weight-optim} admits a unique maximizer $\bm{w}^*$. Let $\{\bm{w}^{(t)}\}_{t\geq 0}$ be the sequence generated by Eq.~\eqref{eq:DRIVEr_fixed_point}. If the diversity-gradient term $\bm h(\bm w)$ is $L_h$-Lipschitz and $\eta L_h<\tau$, then the fixed-point map is contractive, so there exists $\xi=\eta L_h/\tau<1$ such that
$\|\bm{w}^{(t)}-\bm{w}^*\|_1
\le
\xi^t \|\bm{w}^{(0)}-\bm{w}^*\|_1$.
Therefore, the iterate after $T$ steps satisfies $\|\bm{w}^{(T)}-\bm{w}^*\|_1=\mathcal{O}(\xi^T)$. Hence, under the stated contraction condition, the proposed solver returns a near-optimal weight vector after finitely many iterations. Proofs are deferred to Appendix~\ref{sec:appendix_optimization} and Appendix~\ref{sec:appendix_proofs}.

\section{Experiments}
\label{sec:exp}

\begin{table*}[t]
\caption{Performance comparison of models fine-tuned on vision benchmarks~(Accuracy in \%). Best in \textbf{bold}, second in \underline{underlined}.}
\vspace{-2pt}
\label{tab:DRIVE}
\centering
\small
\setlength{\tabcolsep}{4pt}

\renewcommand{\arraystretch}{1.1}
\definecolor{highlightblue}{HTML}{E0F7FA} 
\begin{adjustbox}{max width=\textwidth}
\begin{tabular}{@{}c|c|ccc|ccc|ccc@{}}
\toprule
\multirow{2}{*}{\textbf{Model}} & \multirow{2}{*}{\textbf{Method}}
& \multicolumn{3}{c|}{\textbf{CIFAR10}}
& \multicolumn{3}{c|}{\textbf{CIFAR100}}
& \multicolumn{3}{c}{\textbf{Flowers102}} \\
\cmidrule(lr){3-5}\cmidrule(lr){6-8}\cmidrule(lr){9-11}
& & \textbf{IN-1K Acc.} & \textbf{Target Acc.} & \textbf{Avg Acc.}
  & \textbf{IN-1K Acc.} & \textbf{Target Acc.} & \textbf{Avg Acc.}
  & \textbf{IN-1K Acc.} & \textbf{Target Acc.} & \textbf{Avg Acc.} \\
\midrule
\multirow{8}{*}{\rotatebox{90}{\textbf{ResNet-18}}}
& Pre-trained
& \textbf{69.76$_{\pm 0.00}$} & -- & --
& \textbf{69.76$_{\pm 0.00}$} & -- & --
& \textbf{69.76$_{\pm 0.00}$} & -- & -- \\
\cmidrule(lr){2-11}

& Standard FT
&34.53$_{\pm 0.03}$ & \textbf{96.13}$_{\pm 0.06}$ & 65.33$_{\pm 0.04}$
& 3.28 $_{\pm 0.06}$ &\textbf{84.12$_{\pm 0.04}$} & 43.70$_{\pm 0.05}$
& 5.28$_{\pm 0.01}$ &\underline{92.00$_{\pm 0.03}$} & 48.64$_{\pm 0.02}$ \\

& Linear Probe
& \textbf{69.76$_{\pm 0.00}$} & 81.24$_{\pm 0.12}$ & 75.50$_{\pm 0.06}$
& \textbf{69.76$_{\pm 0.00}$} &59.01$_{\pm 0.02}$ & 64.39$_{\pm 0.01}$
& \textbf{69.76$_{\pm 0.00}$ }&86.86$_{\pm 0.02}$ & 78.31$_{\pm 0.01}$ \\

& WiSE-FT
& 57.04$_{\pm 0.07}$ &87.91$_{\pm 0.03}$ & 72.48$_{\pm 0.03}$
& 49.31$_{\pm 0.04}$ & 62.18$_{\pm 0.01}$ & 55.75$_{\pm 0.02}$
& 49.60$_{\pm 0.01}$ & 75.69$_{\pm 0.05}$ & 62.65$_{\pm 0.02}$ \\

& $\ell_2$-Reg.
& 37.86$_{\pm 0.05}$ & \underline{95.53$_{\pm 0.04}$} & 66.70$_{\pm 0.05}$
& 29.86$_{\pm 0.04}$ & \underline{81.82$_{\pm0.07}$} & 55.84$_{\pm 0.06}$
& 19.34$_{\pm 0.06}$ &\textbf{92.11$_{\pm 0.06}$} & 55.73$_{\pm 0.04}$ \\

& DEFT-UCS
& 47.56$_{\pm 0.05}$ & 87.26$_{\pm 0.06}$ &  67.41$_{\pm 0.06}$
& 49.63$_{\pm 0.07}$ & 79.63$_{\pm 0.05}$ &  64.63$_{\pm 0.06}$
& 62.45$_{\pm 0.04}$ & 87.94$_{\pm0.07}$ & 75.20$_{\pm 0.06}$ \\

& FLOW
& \underline{68.83$_{\pm 0.02}$} & 88.44$_{\pm 0.05}$ & \underline{78.64$_{\pm 0.04}$}
&53.05$_{\pm 0.07}$ & 77.07$_{\pm 0.02}$ & \underline{65.06$_{\pm 0.02}$}
& \underline{69.47$_{\pm 0.06}$} & 87.64$_{\pm 0.02}$ & \underline{78.56$_{\pm 0.04}$} \\

\rowcolor{highlightblue} 
\cellcolor{white} 
& \textbf{DRIVE (Ours)}
& 67.17$_{\pm 0.03}$ & 92.38$_{\pm 0.04}$ & \textbf{79.78$_{\pm 0.03}$}
& \underline{54.66$_{\pm 0.01}$} & 80.39$_{\pm 0.01}$ & \textbf{67.53$_{\pm 0.01}$}
& 68.20$_{\pm 0.02}$ & 90.65$_{\pm 0.05}$ & \textbf{79.43$_{\pm 0.03}$} \\
\midrule

\multirow{8}{*}{\rotatebox{90}{\textbf{ResNet-50}}}
& Pre-trained
& \textbf{80.10$_{\pm 0.00}$} & -- & --
& \textbf{80.10$_{\pm 0.00}$} & -- & --
& \textbf{80.10$_{\pm 0.00}$} & -- & -- \\
\cmidrule(lr){2-11}
& Standard FT
& 29.99$_{\pm 0.04}$ & \textbf{97.19$_{\pm 0.03}$} & 63.59$_{\pm 0.03}$
& 2.50$_{\pm 0.06}$ & \textbf{86.10$_{\pm 0.07}$} & 44.30$_{\pm 0.05}$
& 6.70$_{\pm 0.07}$ & \textbf{92.70$_{\pm 0.06}$} & 49.70$_{\pm 0.05}$ \\

& Linear Probe
& \textbf{80.10$_{\pm 0.00}$} &82.46$_{\pm 0.03}$& 81.28$_{\pm 0.02}$
& \textbf{80.10$_{\pm 0.00}$} &61.11$_{\pm 0.05}$ & 70.61$_{\pm 0.02}$
&  \textbf{80.10$_{\pm 0.00}$} & 85.84$_{\pm 0.04}$ & \underline{82.97$_{\pm0.02}$ }\\

& WiSE-FT
& 63.17$_{\pm 0.04}$ & 91.71$_{\pm 0.02}$ & 77.44$_{\pm 0.02}$
& 55.42$_{\pm 0.03}$ & 71.33$_{\pm 0.05}$ & 63.38$_{\pm 0.03}$
& 58.07$_{\pm 0.03}$ &85.82$_{\pm 0.07}$ & 71.95$_{\pm 0.03}$ \\

& $\ell_2$-Reg.
& 33.98$_{\pm 0.05}$ & \underline{97.00$_{\pm 0.03}$ }& 65.49$_{\pm 0.02}$
& 47.16$_{\pm 0.03}$ & \underline{85.77$_{\pm 0.01}$} &  66.47$_{\pm 0.02}$
& 62.85$_{\pm 0.06}$ & 91.67$_{\pm 0.03}$ &77.26$_{\pm 0.05}$ \\

& DEFT-UCS
& 61.73$_{\pm 0.07}$ & 90.92$_{\pm 0.03}$ & 76.33$_{\pm 0.05}$
& 57.92$_{\pm 0.06}$ & 82.11$_{\pm 0.08}$ & 70.02$_{\pm 0.07}$
&70.42$_{\pm 0.05}$ &89.77$_{\pm 0.09}$ & 80.10$_{\pm 0.07}$ \\

& FLOW
& \underline{75.15$_{\pm 0.04}$} & 89.18$_{\pm 0.02}$ & \underline{82.17$_{\pm 0.02}$}
& 60.85$_{\pm 0.05}$ & 80.74$_{\pm 0.05}$ &\underline{70.80$_{\pm 0.03}$}
&72.45$_{\pm 0.01}$ & 88.93$_{\pm 0.07}$ & 80.69$_{\pm 0.04}$ \\

\rowcolor{highlightblue} 
\cellcolor{white} 
& \textbf{DRIVE (Ours)} 
&73.89$_{\pm 0.03}$ & 94.45$_{\pm 0.03}$ & \textbf{84.17$_{\pm 0.03}$}
& \underline{61.64$_{\pm 0.05}$} &83.03$_{\pm 0.06}$ & \textbf{72.34$_{\pm 0.02}$}
& \underline{75.02$_{\pm 0.01}$} & \underline{92.67$_{\pm 0.04}$} & \multicolumn{1}{c@{}}{\cellcolor{highlightblue}\textbf{83.85$_{\pm 0.01}$}} \\
\bottomrule
\end{tabular}
\end{adjustbox}
\end{table*}

Our comprehensive empirical analysis is structured to address five core research questions: \circnum{1} Does DRIVE outperform SOTAs in balancing task adaptation and anti-forgetting capabilities? \circnum{2} Do individual components of DRIVE contribute effectively to the fine-tuning objectives? \circnum{3} How sensitive is DRIVE to choices of essential hyper-parameters, including $\lambda$, $\tau$, and $\eta$? \circnum{4} Does diversity regularization accelerate convergence and enhance estimation robustness? \circnum{5} Do loss-based contribution scores effectively guide weighting?

\subsection{Setup}
\textbf{Baselines.} We compare DRIVE with methods in the data-oblivious setting: Standard FT, linear probing, $\ell_2$-regularization~\cite{kirkpatrick2017overcoming}, WiSE-FT~\cite{wortsman2022robust}, LoRA~\cite{hu2022lora}, DEFT-UCS~\cite{das2024deft}, FLOW~\cite{sanyal2025upweighting}, and TALR~\cite{lin2026sft}.

\textbf{Datasets.} Vision uses ImageNet-1K (anti-forgetting) and nine benchmarks: CIFAR-10/100~\cite{krizhevsky2009learning}, Flowers102~\cite{nilsback2008automated}, Caltech101~\cite{fei2004learning}, Cars/Dogs~\cite{krause20133d,khosla2011novel}, Food101~\cite{bossard2014food}, Beans~\cite{elfatimi2022beans}, and EuroSAT~\cite{helber2019eurosat}. LLMs fine-tune on MetaMathQA~\cite{yu2024metamath}/Magicoder~\cite{wei2024magicoder} targeting GSM8K~\cite{cobbe2021training}/HumanEval~\cite{chen2021evaluating}, evaluated on MMLU~\cite{hendrycks2021measuring}, Commonsense, and MBPP~\cite{austin2021program}/GSM8K.

\textbf{Evaluation metrics.} For vision, we report fine-tuning accuracy and ImageNet-1K top-1 accuracy (restoring pre-trained heads) for anti-forgetting, alongside their average. For language, we employ \texttt{lm-evaluation-harness}~\cite{gao12608602framework} to assess target and general tasks, reporting the Overall metric~(the average of target accuracy and general capability). Details in Appendix~\ref{sec:llm_eval_details}.

\textbf{Implementation details.} For vision tasks, we use ImageNet-1K-pretrained ResNet-18 and ResNet-50~\cite{wightman2021resnet}. For language, we use Gemma-2-2B~\cite{team2024gemma}, Llama-3.2-3B and Llama-3.1-8B~\cite{grattafiori2024llama}. Across domains, we decay $\lambda$ from 1 to 0, set $\tau$ to the 10th percentile of pre-training loss, $\eta=1$, $k=128$, $T_{\text{max}}=10$, and $\epsilon=1e-5$; see Appendix~\ref{full_details}. Experiments used four A100 80GB GPUs.

\begin{table*}[t]
\caption{Performance comparison of models fine-tuned on the MetaMathQA dataset~(Accuracy in \%). Best in \textbf{bold}, second in \underline{underlined}.}
\vspace{-2pt}
\label{tab:main-language}
\footnotesize
\setlength{\tabcolsep}{12pt}
\renewcommand{\arraystretch}{1.1}
\definecolor{highlightblue}{HTML}{E0F7FA}
\centering
\begin{adjustbox}{max width=\textwidth}
\begin{tabular}{@{}c|c|cccc|c|>{\columncolor{white}[\tabcolsep][0pt]}c@{}}
\toprule
\multirow{2}{*}[-0.6ex]{\centering\textbf{Model}} &
\multirow{2}{*}[-0.6ex]{\centering\textbf{Method}} &
\multicolumn{4}{c|}{\textbf{General Capability Score~(\%)}} &
\textbf{Target Acc.~(\%)} &
\textbf{Overall~(\%)} \\
\cmidrule(lr){3-6} \cmidrule(lr){7-7} \cmidrule(lr){8-8}
& & \textbf{CS} & \textbf{MMLU} & \textbf{MBPP} & \textbf{Avg.} & \textbf{GSM8K} & \textbf{Avg.} \\
\midrule

\multirow{9}{*}{\rotatebox{90}{\textbf{Gemma-2-2B}}}
& Pre-trained
& 57.23$_{\pm 0.00}$ & 53.14$_{\pm 0.00}$ & 28.60$_{\pm 0.00}$ & 46.32$_{\pm 0.00}$ & 23.65$_{\pm 0.00}$ & 34.99$_{\pm 0.00}$ \\
\cmidrule(lr){2-8}
& Standard FT
& 54.89$_{\pm 1.01}$ & 49.26$_{\pm 0.22}$ & 23.40$_{\pm 2.29}$ & 42.52$_{\pm 0.82}$ & \textbf{63.81$_{\pm 0.67}$} & 53.17$_{\pm 0.73}$ \\
& WiSE-FT
& 55.28$_{\pm 0.62}$ & 50.42$_{\pm 0.50}$ & 26.80$_{\pm 1.41}$ & 44.17$_{\pm 0.65}$ & 54.00$_{\pm 0.51}$ & 49.09$_{\pm 0.44}$ \\
& $\ell_2$-Reg.
& \textbf{58.76$_{\pm 0.97}$} & \textbf{53.92$_{\pm 0.15}$} & 30.72$_{\pm 1.08}$ & \textbf{47.80$_{\pm 0.48}$} & 58.50$_{\pm 0.44}$ & 53.15$_{\pm 0.40}$ \\
& LoRA
& 55.58$_{\pm 0.47}$ & 50.34$_{\pm 0.39}$ & 27.20$_{\pm 1.55}$ & 44.37$_{\pm 0.67}$ & 61.50$_{\pm 0.21}$ & 52.94$_{\pm 0.36}$ \\
& DEFT-UCS
& 52.05$_{\pm 0.75}$ & 44.55$_{\pm 0.23}$ & 22.80$_{\pm 1.44}$ & 39.80$_{\pm 0.67}$ & 56.63$_{\pm 0.44}$ & 48.22$_{\pm 0.42}$ \\
& FLOW
& \underline{58.00$_{\pm 2.35}$} & 53.50$_{\pm 0.32}$ & 31.00$_{\pm 2.38}$ & 47.50$_{\pm 1.46}$ & 60.20$_{\pm 0.49}$ & 53.85$_{\pm 0.82}$ \\
& TALR
& 57.98$_{\pm 1.87}$ & 53.58$_{\pm 0.26}$ & \underline{31.04$_{\pm 1.22}$} & 47.53$_{\pm 0.67}$ & 60.20$_{\pm 0.48}$ & \underline{53.87$_{\pm 0.35}$} \\
\rowcolor{highlightblue}
\cellcolor{white}
& \textbf{DRIVE (Ours)}
& 57.92$_{\pm 1.46}$ & \underline{53.59$_{\pm 0.14}$} & \textbf{31.40$_{\pm 1.90}$} & \underline{47.64$_{\pm 1.05}$} & \underline{63.46$_{\pm 0.43}$} & \textbf{55.55$_{\pm 0.45}$} \\
\midrule

\multirow{9}{*}{\rotatebox{90}{\textbf{Llama-3.2-3B}}}
& Pre-trained
& 54.62$_{\pm 0.00}$ & 56.54$_{\pm 0.00}$ & 38.60$_{\pm 0.00}$ & 49.92$_{\pm 0.00}$ & 27.75$_{\pm 0.00}$ & 38.84$_{\pm 0.00}$ \\
\cmidrule(lr){2-8}
& Standard FT
& 52.00$_{\pm 1.03}$ & 52.50$_{\pm 0.43}$ & 33.52$_{\pm 2.20}$ & 46.01$_{\pm 1.13}$ & \textbf{65.81$_{\pm 1.05}$} & 55.91$_{\pm 0.08}$ \\
& WiSE-FT
& 54.57$_{\pm 1.40}$ & 52.01$_{\pm 0.94}$ & 33.80$_{\pm 1.58}$ & 46.79$_{\pm 0.70}$ & 56.01$_{\pm 0.41}$ & 51.40$_{\pm 0.49}$ \\
& $\ell_2$-Reg.
& \textbf{56.01$_{\pm 0.60}$} & 51.99$_{\pm 0.45}$ & 40.52$_{\pm 1.59}$ & 49.51$_{\pm 0.62}$ & 60.00$_{\pm 0.89}$ & 54.76$_{\pm 0.44}$ \\
& LoRA
& 54.17$_{\pm 1.16}$ & 53.33$_{\pm 0.22}$ & 35.92$_{\pm 1.41}$ & 47.81$_{\pm 0.48}$ & 63.50$_{\pm 0.48}$ & 55.66$_{\pm 0.16}$ \\
& DEFT-UCS
& 51.28$_{\pm 1.55}$ & 46.03$_{\pm 0.38}$ & 23.80$_{\pm 1.58}$ & 40.37$_{\pm 0.83}$ & 60.38$_{\pm 0.58}$ & 50.38$_{\pm 0.49}$ \\
& FLOW
& \underline{56.00$_{\pm 1.19}$} & \textbf{56.50$_{\pm 0.33}$} & \underline{41.12$_{\pm 2.33}$} & \underline{51.21$_{\pm 1.06}$} & 61.20$_{\pm 0.48}$ & 56.21$_{\pm 0.48}$ \\
& TALR
& 55.98$_{\pm 2.35}$ & \underline{56.43$_{\pm 0.58}$} & \textbf{41.96$_{\pm 1.93}$} & \textbf{51.46$_{\pm 1.46}$} & 61.05$_{\pm 0.38}$ & \underline{56.26$_{\pm 0.61}$} \\
\rowcolor{highlightblue}
\cellcolor{white}
& \textbf{DRIVE (Ours)}
& 55.92$_{\pm 2.68}$ & 55.71$_{\pm 0.28}$ & 41.00$_{\pm 1.08}$ & 50.88$_{\pm 1.05}$ & \underline{64.50$_{\pm 0.53}$} & \textbf{57.69$_{\pm 0.36}$} \\
\midrule

\multirow{9}{*}{\rotatebox{90}{\textbf{Llama-3.1-8B}}}
& Pre-trained
& 61.59$_{\pm 0.00}$ & 65.32$_{\pm 0.00}$ & 47.40$_{\pm 0.00}$ & 58.10$_{\pm 0.00}$ & 50.95$_{\pm 0.00}$ & 54.53$_{\pm 0.00}$ \\
\cmidrule(lr){2-8}
& Standard FT
& 56.50$_{\pm 1.03}$ & 58.50$_{\pm 0.37}$ & 33.52$_{\pm 1.14}$ & 49.51$_{\pm 0.33}$ & \underline{76.91$_{\pm 0.56}$} & 63.21$_{\pm 0.22}$ \\
& WiSE-FT
& 57.46$_{\pm 0.58}$ & 62.37$_{\pm 0.58}$ & 40.60$_{\pm 1.39}$ & 53.48$_{\pm 0.80}$ & 72.89$_{\pm 0.66}$ & 63.19$_{\pm 0.36}$ \\
& $\ell_2$-Reg.
& 60.09$_{\pm 1.27}$ & 62.06$_{\pm 0.56}$ & 44.60$_{\pm 2.32}$ & 55.58$_{\pm 1.14}$ & 73.21$_{\pm 0.78}$ & 64.40$_{\pm 0.83}$ \\
& LoRA
& \underline{60.90$_{\pm 1.20}$} & \underline{63.40$_{\pm 0.44}$} & 46.40$_{\pm 0.51}$ & 56.90$_{\pm 0.59}$ & 74.50$_{\pm 0.42}$ & \underline{65.70$_{\pm 0.49}$} \\
& DEFT-UCS
& 60.78$_{\pm 0.41}$ & 63.35$_{\pm 0.19}$ & 38.20$_{\pm 2.37}$ & 54.11$_{\pm 0.87}$ & 74.37$_{\pm 0.43}$ & 64.24$_{\pm 0.43}$ \\
& FLOW
& 60.00$_{\pm 0.65}$ & 63.00$_{\pm 0.44}$ & \textbf{48.00$_{\pm 1.70}$} & 57.00$_{\pm 0.70}$ & 73.90$_{\pm 0.16}$ & 65.45$_{\pm 0.40}$ \\
& TALR
& 60.82$_{\pm 0.77}$ & 63.36$_{\pm 0.26}$ & 47.48$_{\pm 2.16}$ & \underline{57.22$_{\pm 0.81}$} & 73.84$_{\pm 0.41}$ & 65.53$_{\pm 0.51}$ \\
\rowcolor{highlightblue}
\cellcolor{white}
& \textbf{DRIVE (Ours)}
& \textbf{60.95$_{\pm 0.54}$} & \textbf{63.98$_{\pm 0.33}$} & \underline{47.60$_{\pm 2.25}$} & \textbf{57.51$_{\pm 0.71}$} & \textbf{77.03$_{\pm 0.49}$} & \textbf{67.27$_{\pm 0.52}$} \\
\bottomrule
\end{tabular}
\end{adjustbox}

\par\vspace{2pt}
\begin{minipage}{\textwidth}
\scriptsize
\end{minipage}
\vspace{-0.2in}
\end{table*}

\subsection{Comparison on Vision and Language Tasks~(Q1)}
\label{sec:main_results}
In the vision domain, where Standard FT often incurs severe catastrophic forgetting, DRIVE better reconciles adaptation and retention. As detailed in Table~\ref{tab:DRIVE}, Standard FT excels in adaptation but suffers severe forgetting; while FLOW attains competitive results  overall average. Specifically, using ResNet-50, DRIVE achieves average accuracy gains of 2.0\%, 1.54\%, and 3.16\% over FLOW on CIFAR-10, CIFAR-100, and Flowers102, respectively. These improvements are significant, stemming from higher target accuracy while largely retaining ImageNet-1K performance across  benchmarks. Extensive experiments on six additional fine-grained datasets (Appendix~\ref{app:extended_vision_experiments})  corroborate these findings, showing consistent performance gains across diverse visual domains.

Table~\ref{tab:main-language} shows that Standard FT achieves strong target performance but suffers from general capability degradation, dropping by up to 8.59\% on Llama-3.1-8B. Retention-oriented baselines such as FLOW and TALR better preserve general capabilities, but lag behind in target accuracy. In contrast, DRIVE improves the general capability average over Standard FT by 5.12\%, 4.87\%, and 8.00\% on Gemma-2-2B, Llama-3.2-3B, and Llama-3.1-8B, respectively, while maintaining competitive target accuracy. It ranks second on Gemma-2-2B and Llama-3.2-3B, and slightly surpasses Standard FT by 0.12\% on Llama-3.1-8B.

\begin{wraptable}{r}{0.5\columnwidth}
    \vspace{-0.15in}
    \caption{DRIVE acts as a plug-and-play module, boosting performance in $\ell_2$-Reg. and LoRA.}
    \vspace{-7pt}
    \label{tab:comprehensive_compatibility}
    \centering
    \footnotesize
    \setlength{\tabcolsep}{4pt}
    \renewcommand{\arraystretch}{0.9}
    \resizebox{\linewidth}{!}{%
        \begin{tabular}{@{}l|c|c|c@{}}
        \toprule
        \textbf{Method} & \textbf{Gen Acc.} & \textbf{Target Acc.} & \textbf{Overall Avg.} \\
        \midrule
        $\ell_2$-Reg. & 55.58$_{(+0.00)}$  & 73.21$_{(+0.00)}$ & 64.40$_{(+0.00)}$ \\
        \quad + FLOW & 56.08$_{(+0.50)}$ & 74.80$_{(+1.59)}$ & 65.44$_{(+1.04)}$ \\
        \quad + \textbf{DRIVE} & \textbf{56.92$_{(+1.34)}$} & \textbf{76.19$_{(+2.98)}$} & \textbf{66.56$_{(+2.16)}$} \\
        \midrule
        LoRA & 56.90$_{(+0.00)}$ & 74.50$_{(+0.00)}$ & 65.70$_{(+0.00)}$ \\
        \quad + FLOW & 57.52$_{(+0.62)}$ & 74.91$_{(+0.41)}$ & 66.22$_{(+0.52)}$ \\
        \quad + \textbf{DRIVE} & \textbf{57.54$_{(+0.64)}$} & \textbf{76.41$_{(+1.91)}$} & \textbf{66.98$_{(+1.28)}$} \\
        \bottomrule
        \end{tabular}%
    }
    \vspace{-10pt}
\end{wraptable}

These stability and plasticity gains yield the best overall performance across all three LLMs. DRIVE surpasses the strongest baseline by 1.68\%, 1.43\%, and 1.57\% overall on Gemma-2-2B, Llama-3.2-3B, and Llama-3.1-8B, respectively. Overall, DRIVE secures high target performance while substantially reducing general capability degradation. Appendix~\ref{app:extended_llm_experiments} extends our evaluation to a code dataset, and Appendix~\ref{app:alternative_metrics} reports a replacement study on contribution scoring.

DRIVE serves as a plug-and-play module that can be directly integrated with parameter-space fine-tuning algorithms. Table~\ref{tab:comprehensive_compatibility} validates DRIVE as a versatile module for parameter-space algorithms like $\ell_2$-Regularization and LoRA on Llama-3.1-8B. In the $\ell_2$-Regularization setting, DRIVE unlocks adaptation potential suppressed by rigid penalties, boosting target accuracy by 3.0\% over the baseline, whereas FLOW improves it by 1.6\%. For LoRA, DRIVE provides complementary regularization, surpassing LoRA+FLOW by 0.76\% in overall average.

These stability and plasticity gains yield the best overall
performance across all three LLMs. DRIVE exceeds the strongest
baseline by 1.68, 1.43, and 1.57 percentage points on Gemma-2-2B,
Llama-3.2-3B, and Llama-3.1-8B, respectively.
To assess the statistical reliability of these improvements,
we perform two-sided paired $t$-tests on the Overall scores
from five runs with matched random seeds and data orders,
using the existing training logs without additional training.
All three improvements remain statistically significant after
Holm correction across the three model-wise comparisons
(adjusted $p \leq 0.0152$), and each individual 95\% confidence
interval for the mean paired difference lies entirely above zero.
These results support the reliability of the observed gains
under the evaluated matched-seed protocol.
Appendix~\ref{app:paired_significance} provides the complete
analysis, including confidence intervals and paired effect sizes.

Overall, DRIVE maintains competitive target performance while
substantially reducing general capability degradation.
Appendix~\ref{app:extended_llm_experiments} extends our evaluation
to code fine-tuning, and
Appendix~\ref{app:alternative_metrics} reports a replacement study
on contribution scoring.

\subsection{Ablation Study~(Q2)}
\label{sec:ablation}

\begin{wraptable}{r}{0.5\columnwidth}
    \vspace{-8pt}
    \caption{Ablation study of DRIVE components.}
    \vspace{-7pt}
    \label{tab:ablation_llama_8b}
    \centering
    \footnotesize
    \setlength{\tabcolsep}{4pt}
    \renewcommand{\arraystretch}{0.9}
    \resizebox{\linewidth}{!}{%
        \begin{tabular}{@{}lccc@{}}
        \toprule
        \textbf{Method} & \textbf{General} & \textbf{Target} & \textbf{Avg} \\
        \midrule
        \textbf{DRIVE} & 57.51 & 77.03 & \textbf{67.27} \\
        \quad - w/o Task-adaptation ($\lambda=1$) & \textbf{58.42} & 73.19 & 65.81 \\
        \quad - w/o Anti-forgetting ($\lambda=0$) & 52.24 & \textbf{79.83} & 66.04 \\
        \quad - w/o Diversity ($\eta=0$) & 56.18 & 74.78 & 65.48 \\
        \quad - w/o KL Prior & 53.96 & 70.61 & 62.29 \\
        \bottomrule
        \end{tabular}%
    }
    \vspace{-2pt}
\end{wraptable}

Table~\ref{tab:ablation_llama_8b} details the ablation results on Llama-3.1-8B using MetaMathQA. 
Removing any component degrades performance. 
Exclusively focusing on anti-forgetting or task adaptation fails to reach an optimal equilibrium; the former reduces target accuracy by 3.84\%, while the latter degrades general capabilities by 5.27\%.
Likewise, removing uniform prior regularization causes the most severe impact, reducing average accuracy by 4.98\%; see Appendix~\ref{app:no_prior_optimization} for its no-prior optimization.
Removing diversity regularization results in a 1.79\% accuracy drop. 
Consistent trends also appear in vision tasks (Appendix~\ref{app:ablation_vision}).

We examine whether DRIVE can benefit from an informative prior. Specifically, on Llama-3.1-8B using MetaMathQA, we compare the default uniform prior with an entropy-informed prior computed from the pre-trained model's predictive entropy. Let $H_i$ denote the predictive entropy of sample $\bm{x}_i$. We assign larger prior mass to lower-entropy samples by setting $p_i \propto \exp(-H_i)$, followed by normalization over the training set. Table~\ref{tab:entropy_prior_comparison} shows that the entropy-informed prior achieves slightly better general and overall performance than the uniform prior, while maintaining comparable target accuracy. This suggests that DRIVE can benefit from simple auxiliary signals when available, while the uniform prior remains a strong default without additional estimators.

\begin{figure}[h]
   \vspace{-2pt}
\centering
\begin{minipage}[t]{0.38\linewidth}
    \vspace{0pt}
    \begin{minipage}{\linewidth}
        \centering
        \captionof{table}{Entropy vs.\ uniform prior.}
        \vspace{-2pt}
        \label{tab:entropy_prior_comparison}
        \fontsize{6.3pt}{7.0pt}\selectfont
        \setlength{\tabcolsep}{5.2pt}
        \renewcommand{\arraystretch}{0.95}
        \resizebox{\linewidth}{!}{%
        \begin{tabular}{@{}lccc@{}}
            \toprule
            \textbf{Prior} & \textbf{General} & \textbf{Target} & \textbf{Avg.} \\
            \midrule
            Uniform Prior & 57.51 & \textbf{77.03} & 67.27 \\
            Entropy Prior & \textbf{58.07} & 76.94 & \textbf{67.51} \\
            \bottomrule
        \end{tabular}}
    \end{minipage}

    \vspace{5pt}

    \begin{minipage}{\linewidth}
        \centering
        \captionof{table}{Diversity comparison.}
        \label{tab:diversity_comparison}
        \vspace{-2pt}
        \fontsize{6.3pt}{7.0pt}\selectfont
        \setlength{\tabcolsep}{5.2pt}
        \renewcommand{\arraystretch}{0.95}
        \resizebox{\linewidth}{!}{%
        \begin{tabular}{@{}lccc@{}}
            \toprule
            \textbf{Method} & \textbf{General} & \textbf{Target} & \textbf{Avg.} \\
            \midrule
            DEFT-UCS & 54.11 & 74.37 & 64.24 \\
            DRIVE-Cluster & 55.20 & 74.80 & 65.00 \\
            \textbf{DRIVE} & \textbf{57.51} & \textbf{77.03} & \textbf{67.27} \\
            \bottomrule
        \end{tabular}}
    \end{minipage}
\end{minipage}
\hfill
\begin{minipage}[t]{0.58\linewidth}
    \vspace{0pt}
    \centering
    \includegraphics[width=\linewidth, trim=1mm 4mm 1mm 2mm, clip]{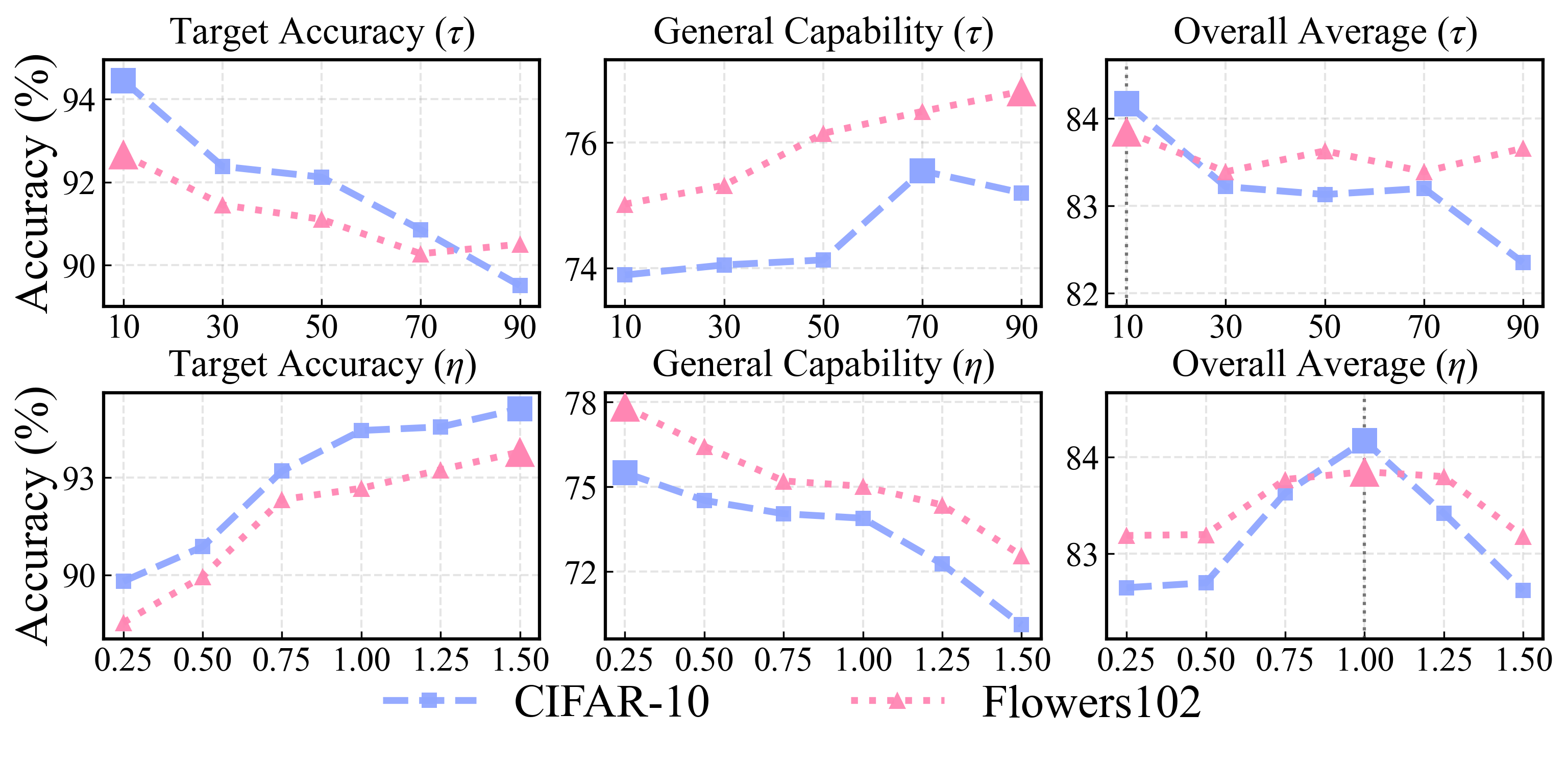}
    \vspace{-10pt}
    \captionof{figure}{Sensitivity analysis of $\tau$ and $\eta$ on ResNet-50.}
    \label{fig:hyper_sensitivity}
\end{minipage}
    \vspace{-10pt}
\end{figure}

To validate the effectiveness of our differentiable covariance diversity regularizer, we compare DRIVE against DEFT-UCS and a discrete variant, DRIVE-Cluster, on Llama-3.1-8B. 
We construct DRIVE-Cluster by replacing the differentiable diversity regularization with a discrete clustering mechanism. Specifically, it partitions data into $K=7$ semantic groups. For a sample $\bm{x}_i$ assigned to the $k$-th cluster $C_k$, we adopt a balanced prior $p_i \equiv 1/(|C_k|\cdot K)$, to enforce equal aggregate cluster weights while determining within-cluster weights (details in Appendix~\ref{app:DRIVE_cluster}). 
As shown in Table~\ref{tab:diversity_comparison}, DRIVE-Cluster outperforms DEFT-UCS, highlighting the benefit of explicit diversity control. Crucially, DRIVE further surpasses DRIVE-Cluster by 2.27\%, indicating that differentiable diversity regularization provides a more adaptive mechanism than rigid cluster balancing.

\subsection{Hyperparameter Sensitivity Analysis~(Q3)}

To analyze how DRIVE mediates the stability-plasticity trade-off, we evaluate   $\lambda$ under static and dynamic settings. For static settings (0.0 to 1.0 at intervals of 0.2), Table~\ref{tab:lambda_ablation} on ResNet-50 shows that lower values favor plasticity but worsen forgetting, whereas higher values prioritize retention, achieving an optimal static balance at $\lambda=0.8$. Among dynamic schedules, the decaying strategy ($1 \to 0$) yields superior average accuracy, outperforming the best static configuration ($\lambda=0.8$).

Regarding hyperparameters in Fig.~\ref{fig:hyper_sensitivity}, $\tau$
controls weight sharpness through smoothing. Small values favor target adaptation, moderate values
give the best overall trade-off, and overly large values make weights nearly
uniform and can weaken retention. Increasing $\eta$ improves target performance up to
$\eta=1.0$, after which generalization degrades.

\subsection{Analysis of Convergence and Robustness~(Q4)}
\label{convergence_and_robustness}
We investigate the optimization dynamics induced by diversity regularization and evaluate estimation robustness by analyzing convergence curves, measuring pairwise parameter distances across four seeds, and examining performance on the top 10\% samples with the highest pre-training loss.

\begin{figure*}[t]
    \centering
    \vspace{8pt}

    \begin{minipage}[t]{0.49\textwidth}
        \centering
        \includegraphics[width=\linewidth, trim=2mm 2mm 0mm 0mm, clip]{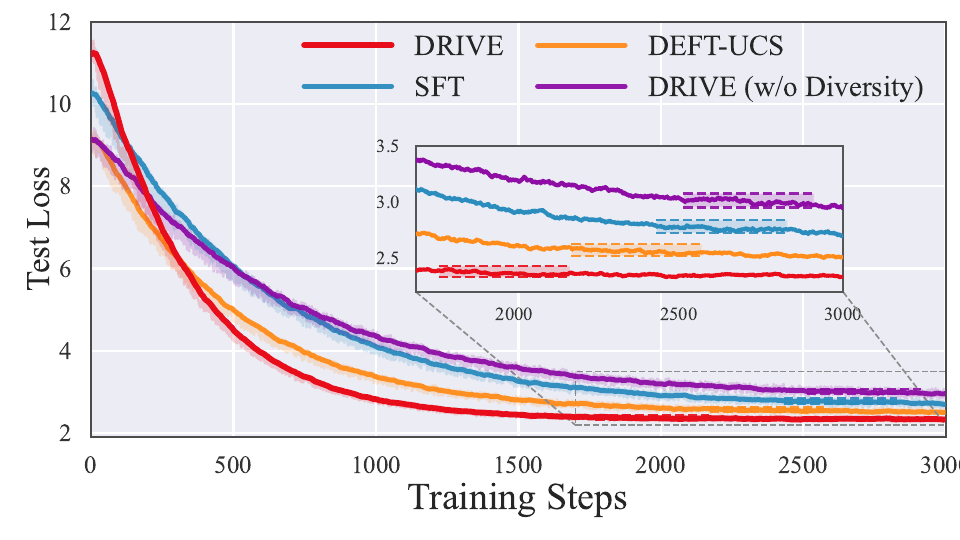}
        \vspace{-18pt}
        \captionof{figure}{Convergence curves across baselines.}
        \label{fig:convergence_panel}
    \end{minipage}
    \hfill
    \begin{minipage}[t]{0.49\textwidth}
        \centering
        \includegraphics[width=\linewidth, trim=2mm 2mm 1mm 0mm, clip]{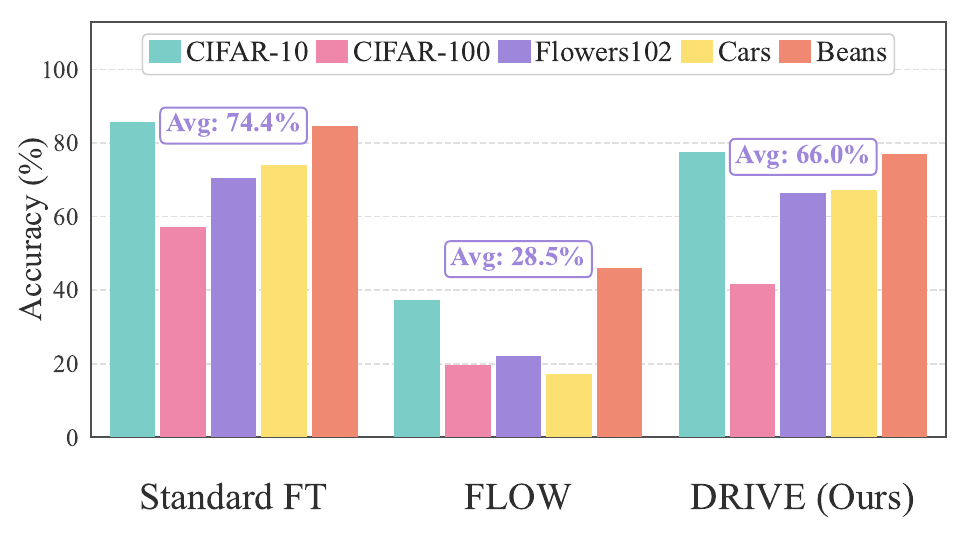}
        \vspace{-18pt}
        \captionof{figure}{Hard-sample accuracy across datasets}
        \label{fig:accuracy_panel}
    \end{minipage}

    \vspace{-10pt}
\end{figure*}

\begin{figure}[h]
\vspace{0.1in}
\centering
\begin{minipage}[t]{0.42\columnwidth}
    \centering
    \captionof{table}{Sensitivity Analysis of $\lambda$.}
    \vspace{-2pt}
    \label{tab:lambda_ablation}
    \footnotesize
    \definecolor{highlightblue}{HTML}{E0F7FA}
    \setlength{\tabcolsep}{3pt}
    \renewcommand{\arraystretch}{1}
    \resizebox{\linewidth}{!}{%
        \begin{tabular}{lcccccc}
        \toprule
        & \multicolumn{3}{c}{\textbf{CIFAR-10}} & \multicolumn{3}{c}{\textbf{Flowers102}} \\
        \cmidrule(lr){2-4} \cmidrule(lr){5-7}
        \textbf{$\lambda$} & \textbf{IN-1K} & \textbf{Target} & \textbf{Avg.} & \textbf{IN-1K} & \textbf{Target} & \textbf{Avg.} \\
        \midrule
        0.0 & 61.84 & \textbf{97.10} & 79.47 & 71.98 & \textbf{93.67} & 82.83 \\
        0.2 & 64.92 & \underline{96.08} & 80.50 & 72.47 & 92.63 & 82.55 \\
        0.4 & 67.31 & 95.21 & 81.26 & 73.39 & 92.55 & 82.97 \\
        0.6 & 70.28 & 94.63 & 82.46 & 75.06 & 91.53 & 83.30 \\
        0.8 & 73.02 & 94.19 & \underline{83.61} & \underline{75.42} & 91.69 & \underline{83.56} \\
        1.0 & \textbf{76.10} & 90.02 & 83.06 & \textbf{76.26} & 90.73 & 83.50 \\
        \midrule
        0 $\to$ 1 & 70.95 & 94.88 & 82.92 & 73.67 & \underline{93.01} & 83.34 \\
        \rowcolor{highlightblue}
        1 $\to$ 0 & \underline{73.89} & 94.45 & \textbf{84.17} & 75.02 & 92.67 & \textbf{83.85} \\
        \bottomrule
        \end{tabular}%
    }
\end{minipage}
\hfill
\begin{minipage}[t]{0.56\columnwidth}
    \vspace{2pt}
    \centering
    \includegraphics[width=\linewidth]{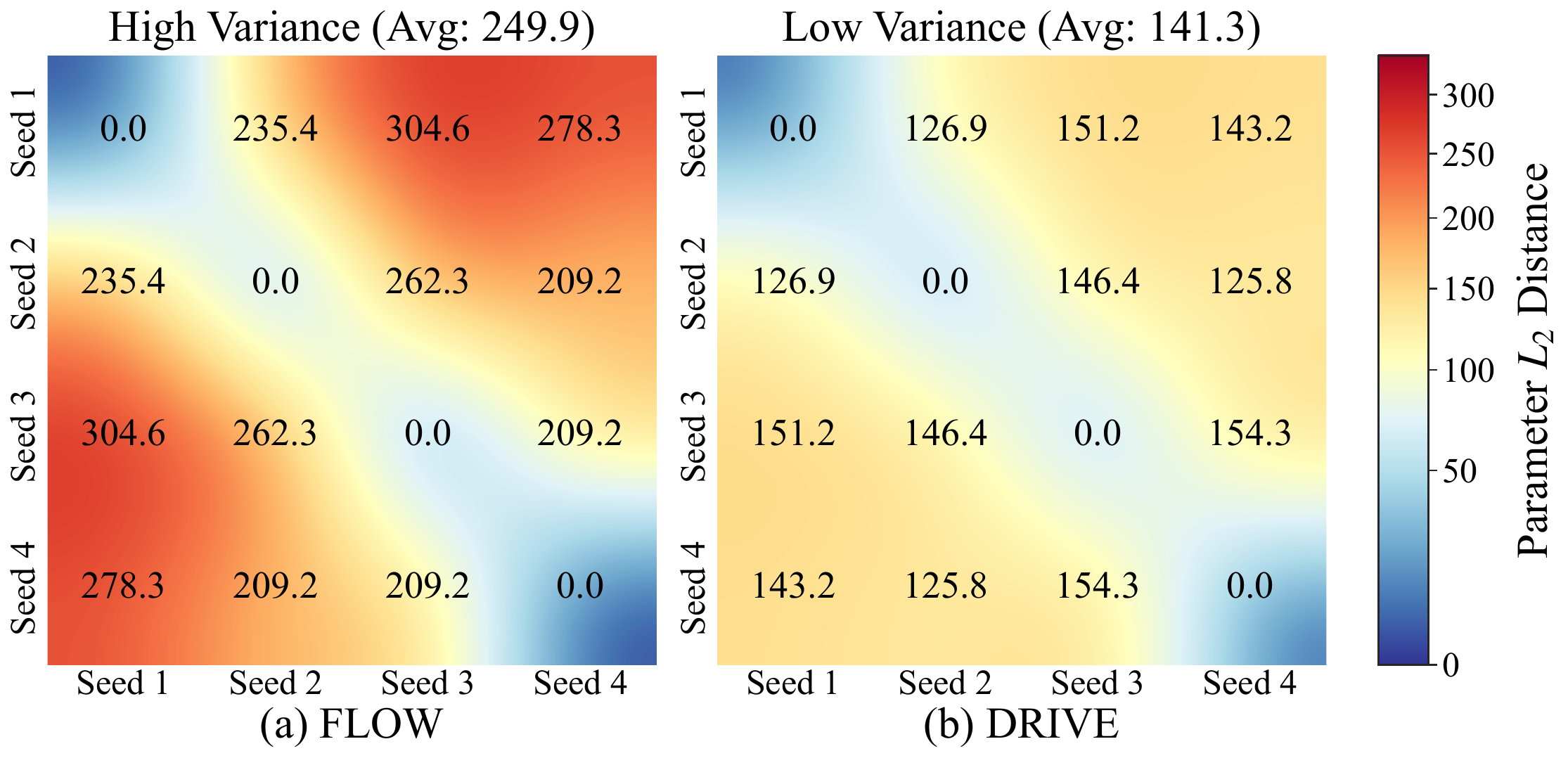}
    \vspace{-12pt}
    \captionof{figure}{Pairwise parameter distance heatmap.}
    \label{fig:param_heatmap}
\end{minipage}
\vspace{-2pt}
\end{figure}

Fig.~\ref{fig:convergence_panel} illustrates diversity term's impact on training dynamics using Llama-3.1-8B on MetaMathQA. We define convergence as the point where the relative fluctuation of the smoothed loss remains within 2\% over a 400-step window. DRIVE stabilizes significantly earlier than baselines lacking diversity constraints.
This supports our claim that diversity accelerates convergence.

To verify DRIVE mitigates the under-fitting inherent in stability-oriented methods like FLOW, we further tested ResNet-50 on the most challenging 10\% samples by pre-training loss. In Fig.~\ref{fig:accuracy_panel}, FLOW performs poorly on these challenging samples. Standard FT excels but suffers catastrophic forgetting on the retention set. Conversely, DRIVE maintains satisfactory performance on difficult samples due to its explicit contribution weighting mechanism. More detailed analysis is provided in Appendix~\ref{sec:appendix_hard_sample_dynamics}.

To quantify  estimation robustness, we analyzed geometric consistency across four random seeds via the Euclidean ($L_2$) distance between flattened fine-tuned parameter vectors. As shown in Fig.~\ref{fig:param_heatmap} for ResNet-50 on CIFAR-100, FLOW exhibits volatility and pronounced seed sensitivity. In contrast, DRIVE enhances geometric stability, reducing the average pairwise parameter distance by 43.46\%. This contraction further confirms that our diversity regularization imposes an effective constraint, minimizing the estimator's generalized variance and yielding more reliable parameter estimates.

\subsection{Validation of Loss-Based Contribution Scores~(Q5)}
\label{sec:proxy_validation}

\begin{wraptable}{r}{0.4\columnwidth}
    \vspace{-8pt}
    \centering
    \captionsetup{font=footnotesize,skip=3pt}
   \caption{Effect of score perturbations.}
    \label{tab:contribution_perturbation}
    \fontsize{7.5pt}{9pt}\selectfont
    \setlength{\tabcolsep}{2pt}
    \renewcommand{\arraystretch}{1.0}
    \begin{tabular*}{\linewidth}{@{\extracolsep{\fill}}lccc@{}}
        \toprule
        \textbf{Setting} & \textbf{General}
        & \textbf{Target} & \textbf{Overall} \\
        \midrule
        DRIVE & \textbf{57.48}
        & \textbf{77.10} & \textbf{67.29} \\
        Shuffled & 56.31 & 73.92 & 65.12 \\
        Reversed & 50.02 & 70.43 & 60.23 \\
        \bottomrule
    \end{tabular*}
    \vspace{-6pt}
\end{wraptable}

On ResNet-50/Flowers102 (1,020 training samples; $\lambda=0.8$), $c_i$ correlates with validation-gradient-based Shapley approximations~\cite{wang2025data}, achieving Spearman's $\rho=0.83$ and Pearson's $r=0.81$ (both $p<0.001$). Reference validation sets are used only for diagnostics (Appendix~\ref{app:proxy_support}). In a single-seed diagnostic on Llama-3.1-8B/MetaMathQA, shuffling or reversing score rankings while preserving score values and other settings reduces Overall by 2.17 and 7.06 percentage points, respectively (Table~\ref{tab:contribution_perturbation}), supporting the utility of sample-specific contribution scores.

\section{Conclusion}
We presented DRIVE, a novel data weighting framework for data-oblivious fine-tuning. Rather than aggregating off-the-shelf signals, DRIVE recasts contribution, diversity, and sample prior into a single coherent differentiable objective. Concretely, contribution is partitioned into retention and adaptation via a tunable Shapley-additive decomposition, diversity is captured through a differentiable covariance-determinant term, and weights are anchored by a prior regularizer. In linear and local nonlinear regimes, this unified modeling tightens the stability--plasticity trade-off. Diversity provably accelerates convergence and bolsters estimation robustness; the solver converges linearly. Across vision and language benchmarks, DRIVE consistently surpasses SOTA methods. Future work explores samples beneficial to both contributions and token-level extensions.

\section*{Acknowledgments and Disclosures}
This work was supported by the National Natural Science Foundation of China under Grant Nos. 62476191 and 42550187, with computing resources provided by our research group. The authors declare no competing interests.

{
\small
\bibliographystyle{IEEEtran}
\bibliography{reference}
}

\newpage
\appendix
\onecolumn

\newpage

\newcommand{\appendixentry}[2]{%
  \noindent\hyperref[#1]{\ref*{#1}\quad #2}%
  \dotfill
  \hyperref[#1]{\pageref*{#1}}%
  \par
}

\newcommand{\appendixsubentry}[2]{%
  \noindent\hspace{2.2em}\hyperref[#1]{\ref*{#1}\quad #2}%
  \dotfill
  \hyperref[#1]{\pageref*{#1}}%
  \par
}

{
\hypersetup{linkcolor=black}

\section*{Organization of the Appendix}
\vspace{-0.5em}

\appendixentry{app:more_details_contribution}{More Details Supporting Our Contribution Function $\mathcal{C}$}
\appendixsubentry{app:why_not_shapley}{Inapplicable Contribution Measures}
\appendixsubentry{app:proxy_support}{Statistical Support for the Loss-Based Contribution Proxies}
\appendixsubentry{app:alternative_metrics}{Alternative Validation-Free Metrics}

\appendixentry{sec:appendix_optimization}{Solution for Objective~\eqref{eq:weight-optim} and Convergence Analysis}
\appendixsubentry{sec:appendix_convexity}{Concavity and Uniqueness of Eq.~\eqref{eq:weight-optim}}
\appendixsubentry{sec:appendix_derivation}{Derivation of the Iterative Update Rule}
\appendixsubentry{app:no_prior_optimization}{Optimization for the No-Prior Ablation}
\appendixsubentry{sec:appendix_weight_solver}{Convergence of the Weight Solver}
\appendixsubentry{app:random_projection}{Justification for Randomized Dimensionality Reduction}

\appendixentry{sec:appendix_proofs}{Detailed Theoretical Proofs}
\appendixsubentry{app:proof_linear_regime}{Proof for the Linear Regime in Sec.~\ref{sec:theory_regimes}}
\appendixsubentry{app:proof_local_nonlinear_regime}{Local Nonlinear Analysis with Hessian Remainder Control}

\appendixentry{full_details}{More Experimental Details}
\appendixsubentry{app:baselines}{Baseline Details}
\appendixsubentry{sec:evaluation_details}{Vision Datasets and Evaluation Details}
\appendixsubentry{sec:llm_eval_details}{LLM Datasets and Evaluation Details}
\appendixsubentry{app:impl_details}{Implementation Details}

\appendixentry{app:extended_experiments}{Additional Experimental Results and Analysis}
\appendixsubentry{app:extended_vision_experiments}{Analysis of Vision Tasks}
\appendixsubentry{app:vit_experiments}{Experiments on Vision Transformers}
\appendixsubentry{app:extended_llm_experiments}{Experiments on Language Models}
\appendixsubentry{app:lambda_schedule}{Impact of Balance Parameter $\lambda$ on Language Tasks}
\appendixsubentry{app:tau_impact}{Temporal Dynamics of $\lambda$: Increasing vs. Decaying Schedule}
\appendixsubentry{app:regularization_tau}{Impact of Prior Regularization $\tau$}
\appendixsubentry{app:eta_impact}{Impact of Diversity Regularization $\eta$ on Language Tasks}
\appendixsubentry{app:ablation_vision}{Ablation Study on Vision Tasks}
\appendixsubentry{appendix:efficiency}{Computational Efficiency Analysis}

\appendixentry{app:mechanism_analysis}{Analysis of Contribution Assessment and Diversity Mechanisms}
\appendixsubentry{sec:appendix_shapley}{Shapley Analysis of Anti-Forgetting vs. Task Adaptation}
\appendixsubentry{app:diversity_metrics}{Sample Diversity}
\appendixsubentry{sec:appendix_ensemble}{Analysis of Covariance Approximation Strategies}
\appendixsubentry{sec:appendix_hard_sample_dynamics}{Empirical Verification of Dynamic Adaptation on Hard Samples}
\appendixsubentry{app:k_sensitivity}{Sensitivity to the Projection Dimension $k$}

\appendixentry{sec:limitations}{Limitations}
}

\newpage
\section{More Details Supporting Our Contribution Function $\mathcal{C}$}
\label{app:more_details_contribution}

This appendix provides additional support for the contribution function $\mathcal{C}$.
Sec.~\ref{app:why_not_shapley} explains why exact Shapley values and
validation-dependent contribution estimates cannot be directly deployed under our
strict data-oblivious setting. Sec.~\ref{app:proxy_support} shows that the adopted
loss-based proxies remain statistically aligned with validation-based
Shapley-style reference contributions in the aspects most relevant to soft
weighting. Sec.~\ref{app:alternative_metrics} discusses several practical
validation-free alternatives and empirically compares representative variants,
showing that the standard loss-based instantiation offers the best
effectiveness--efficiency trade-off.
\subsection{Inapplicable Contribution Measures}
\label{app:why_not_shapley}

A natural theoretical foundation for sample contribution is the Shapley value~\cite{shapley1953value}. Let $U(\cdot)$ denote a utility function and $D$ be a training set of $N$ samples. For a sample $z\in D$, its Shapley value is
\begin{equation}
\phi_z(U)
:=
\frac{1}{N}
\sum_{k=1}^{N}
\binom{N-1}{k-1}^{-1}
\sum_{\substack{S\subseteq D\setminus\{z\}\\ |S|=k-1}}
\left[U(S\cup\{z\})-U(S)\right].
\end{equation}
This provides a rigorous conceptual foundation for contribution modeling. However, exact Shapley computation requires repeated retraining over exponentially many subsets, which is infeasible for large-scale fine-tuning. Thus, exact Shapley values can only serve as a theoretical reference, but not as the deployed contribution measure in DRIVE.

To avoid exhaustive retraining, recent work has proposed one-run approximations that estimate contribution by accumulating gradient alignment along a single training trajectory~\cite{wang2025data}. In this case, one  define validation-based reference contributions using retain and target validation gradients:
\begin{align}
a_i^{\mathrm{ref}}
&:=
\sum_{t=1}^{T}\eta_t
\left\langle
\nabla_{\theta_t}\ell(f(x_i;\theta_t),y_i),
\nabla_{\theta_t}\ell(f(D_{\mathrm{val}}^{\mathrm{retain}};\theta_t),y_{\mathrm{val}}^{\mathrm{retain}})
\right\rangle,\\
t_i^{\mathrm{ref}}
&:=
\sum_{t=1}^{T}\eta_t
\left\langle
\nabla_{\theta_t}\ell(f(x_i;\theta_t),y_i),
\nabla_{\theta_t}\ell(f(D_{\mathrm{val}}^{\mathrm{target}};\theta_t),y_{\mathrm{val}}^{\mathrm{target}})
\right\rangle.
\end{align}
These quantities are much more tractable than exact Shapley values and are useful as diagnostic references. However, they still require two external supervision channels: a retain validation set and a target validation set. Similarly, meta-learning-based contribution learning~\cite{ren2018learning} also depends on explicit validation supervision. In contrast, DRIVE is designed for a strict data-oblivious fine-tuning setting, where neither the original pre-training data nor auxiliary retain/target validation sets are assumed to be available during training.

Under the strict data-oblivious setting considered by DRIVE, more principled contribution measures are either computationally intractable (exact Shapley) or fundamentally incompatible with the available supervision (validation-based approximations and meta-learning). This motivates our use of a training-time computable, validation-free proxy in the main text.

\subsection{Statistical Support for the Loss-Based Contribution Proxies}
\label{app:proxy_support}

Although the loss-based definitions in Eq.~\eqref{eq:utilities} are adopted for practicality, the question is whether they preserve the information that matters for weighting. Since Eq.~(\ref{eq:DRIVEr_fixed_point}) depends on the relative ordering and separation of contribution scores, a proxy should preserve trend, ranking structure, head-sample identity, and the agreement of the unified contribution score with its reference counterpart. In this experiment, the unified score is computed with the same setting as the main text, $c_i=\lambda a_i+(1-\lambda)t_i$ with $\lambda=0.8$, and $c_i^{\mathrm{ref}}=\lambda a_i^{\mathrm{ref}}+(1-\lambda)t_i^{\mathrm{ref}}$. Following the validation-based Shapley-style formulation in \cite{wang2025data}, we use the sample contributions computed on the retain validation set and the target validation set as reference anti-forgetting and task-adaptation contributions, denoted by $a_i^{\mathrm{ref}}$ and $t_i^{\mathrm{ref}}$, respectively. These reference quantities are used only for evaluating the quality of the loss-based proxies.

We therefore report four complementary statistics. Quantile Monotonicity Score measures global trend preservation by sorting samples with the proxy score, partitioning them into equal-frequency quantile bins, and computing the fraction of adjacent bins whose mean reference score is monotonically non-decreasing. Top-50\% Overlap measures head-sample consistency as the overlap between the top half selected by the proxy and by the reference. Spearman's $\rho$ evaluates ranking agreement, while Pearson's $r$ serves as a supplementary linear amplitude-based statistic. All reported results are highly significant ($p<0.001$ for every metric under the corresponding tests).

Across these decision-oriented evaluations, the loss-based instantiation used in DRIVE shows strong agreement with the validation-based Shapley-style reference contributions. The quantile trends remain clearly monotone, the Top-50\% overlap is high for the influential samples that dominate weighting, and both Spearman and Pearson correlations are strong. Importantly, this conclusion holds not only for the decomposed anti-forgetting and task-adaptation components, but also for the unified contribution score $c_i$, which is the quantity directly used by DRIVE to determine sample weights.

\begin{table}[t]
\centering
\small
\caption{Agreement between loss-based proxy contributions and validation-based Shapley reference scores on ResNet-50 fine-tuned on Flowers102, using an ImageNet subset as the retain validation set.}
\label{tab:proxy_shapley_agreement}
\setlength{\tabcolsep}{6pt}
\renewcommand{\arraystretch}{1.15}
\begin{tabular}{lccc}
\toprule
\textbf{Metric} & \textbf{Anti-forgetting} & \textbf{Task adaptation} & \textbf{Unified score $c_i$} \\
\midrule
Quantile Monotonicity Score $\uparrow$ & 0.89 & 0.98 & 0.92 \\
Spearman $\rho$ $\uparrow$ & 0.76 & 0.87 & 0.83 \\
Top-50\% Overlap $\uparrow$ & 0.87 & 0.95 & 0.89 \\
Pearson $r$ $\uparrow$ & 0.65 & 0.87 & 0.81 \\
\bottomrule
\end{tabular}
\vspace{-14pt}
\end{table}

\subsection{Alternative Validation-Free Metrics}
\label{app:alternative_metrics}

While the main text adopts loss-based contributions, the DRIVE framework can accommodate other practical validation-free surrogates. Below we discuss several representative alternatives. Here, \(L\) denotes the number of prediction positions over which the surrogate score is averaged; for LLM fine-tuning, it is the sequence length, while for vision classification we set \(L=1\).

\paragraph{Gradient-based metrics.}
The gradient norm provides a more direct first-order measure of sample-level update influence than scalar loss~\cite{paul2021deep}. A natural alternative is
\begin{align}
a_i &= -\left\lVert \nabla_{\theta_{\mathrm{pre}}}\,\ell\!\bigl(f(x_i;\theta_{\mathrm{pre}}),y_i\bigr) \right\rVert_{2}, \\
t_i &= \left\lVert \nabla_{\theta}\,\ell\!\bigl(f(x_i;\theta),y_i\bigr) \right\rVert_{2}.
\end{align}
This formulation is theoretically appealing and closely connected to GraNd-style sample valuation~\cite{paul2021deep}, but it requires per-sample backward computation and is significantly more expensive in large-scale fine-tuning.

\paragraph{Error-based metrics.}
EL2N measures the $\ell_2$ discrepancy between the predicted probability vector and the target one-hot vector~\cite{paul2021deep}:
\begin{equation}
\mathrm{EL2N}(x_i;\theta)
=
\frac{1}{L}\sum_{j=1}^{L}\left\|p_j(x_i;\theta)-y_j\right\|_2,
\end{equation}
based on which we define
\begin{equation}
a_i=-\mathrm{EL2N}(x_i;\theta_{\mathrm{pre}}), \qquad
t_i=\mathrm{EL2N}(x_i;\theta).
\end{equation}
This yields a bounded and relatively robust signal, but in practice it is less directly aligned with optimization dynamics than the adopted loss-based formulation.

\paragraph{Uncertainty-based metrics.}
Prediction entropy provides another alternative, motivated by semantic uncertainty~\cite{kuhn2023semantic}:
\begin{align}
a_i &= - \frac{1}{L}\sum_{j=1}^{L}\left(-\sum_{k=1}^{|V|}p_{j,k}(x_i;\theta_{\mathrm{pre}})
\log p_{j,k}(x_i;\theta_{\mathrm{pre}})\right),\\
t_i &= \frac{1}{L}\sum_{j=1}^{L}\left(-\sum_{k=1}^{|V|}p_{j,k}(x_i;\theta)
\log p_{j,k}(x_i;\theta)\right).
\end{align}
Entropy-based scoring emphasizes uncertainty, but it may over-focus on ambiguous or difficult samples and is not always as stable for balancing retention and adaptation.

\paragraph{Margin-based metrics.}
A margin-based alternative is
\begin{align}
a_i &= \frac{1}{L}\sum_{j=1}^{L}\Bigl(z_{j,y_{i,j}}(x_i;\theta_{\mathrm{pre}})-\max_{k\neq y_{i,j}} z_{j,k}(x_i;\theta_{\mathrm{pre}})\Bigr),\\
t_i &= -\frac{1}{L}\sum_{j=1}^{L}\Bigl(z_{j,y_{i,j}}(x_i;\theta)-\max_{k\neq y_{i,j}} z_{j,k}(x_i;\theta)\Bigr),
\end{align}
which prioritizes boundary-near target samples while favoring confidently mastered prior knowledge.

\paragraph{Replacement study.}
To verify that the loss-based instantiation is a reasonable default, we compare DRIVE against two representative alternatives: a gradient-norm variant (\textsc{DRIVE-Gra}) and an entropy-based variant (\textsc{DRIVE-Entro}). As shown in Table~\ref{tab:contribution_replace_time}, \textsc{DRIVE-Gra} achieves the highest overall score, indicating that richer gradient signals can improve contribution estimation. However, this gain comes at a computational cost, requiring roughly twice the training time of standard DRIVE. By contrast, \textsc{DRIVE-Entro} underperforms both. Notably, the loss-based DRIVE remains competitive in accuracy while offering the strongest efficiency profile. This supports our choice of the loss-based contribution function as the default practical instantiation of $\mathcal{C}$.

\begin{table}[t]
\centering
\small
\caption{Replacement study of practical contribution mechanisms on Llama-3.2 3B. Gradient-based scoring is slightly stronger in accuracy but substantially more expensive, while the standard loss-based DRIVE offers the best effectiveness--efficiency trade-off.}
\label{tab:contribution_replace_time}
\setlength{\tabcolsep}{6pt}
\renewcommand{\arraystretch}{1.15}
\begin{adjustbox}{max width=\columnwidth}
\begin{tabular}{@{}c|cccc|c|c|c@{}}
\toprule
\multirow{2}{*}[-0.6ex]{\textbf{Method}} &
\multicolumn{4}{c|}{\textbf{General Capability Acc.~(\%)}} &
\textbf{Target Acc.~(\%)} &
\textbf{Overall} &
\textbf{Relative Time} \\
\cmidrule(lr){2-5} \cmidrule(lr){6-6} \cmidrule(lr){7-7} \cmidrule(lr){8-8}
& \textbf{CS} & \textbf{MMLU} & \textbf{MBPP} & \textbf{Avg} & \textbf{GSM8K} & \textbf{Avg~(\%)} & \textbf{Cost} \\
\midrule
\textsc{DRIVE-Gra}   & 55.11 & \textbf{55.94} & \textbf{41.48} & 50.84 & \textbf{65.13} & \textbf{57.99} & $2.12\times$ \\
\textsc{DRIVE-Entro} & 55.14 & 55.40 & 40.40 & 50.31 & 59.97 & 55.14 & $1.04\times$ \\
\textsc{DRIVE}       & \textbf{55.92} & 55.71 & 41.00 & \textbf{50.88} & 64.50 & 57.69 & $1.00\times$ \\
\bottomrule
\end{tabular}
\end{adjustbox}
\vspace{-14pt}
\end{table}

\section{Solution for Objective \eqref{eq:weight-optim} and Convergence Analysis}
\label{sec:appendix_optimization}

We analyze the optimization and efficient implementation of Eq.~\eqref{eq:weight-optim}. 
Sec.~\ref{sec:appendix_convexity} proves concavity and uniqueness, 
Sec.~\ref{sec:appendix_derivation} derives the KL-regularized iterative update rule, 
Sec.~\ref{app:no_prior_optimization} explains how the no-prior ablation is optimized, 
Sec.~\ref{sec:appendix_weight_solver} establishes solver convergence, and 
Sec.~\ref{app:random_projection} justifies the randomized projection used for efficient covariance inversion.

\subsection{Concavity and Uniqueness of Eq. \eqref{eq:weight-optim}}
\label{sec:appendix_convexity}

Recall the unified weight optimization objective introduced in Eq.~\eqref{eq:weight-optim}. 
Following the implementation and the solver analysis, we use the ridge-regularized 
weighted covariance \(\bm{\Sigma}_{\bm w}\), where the ridge term is absorbed into the 
notation. The objective is
\begin{equation}
\label{eq:appendix_obj}
\mathcal{J}(\bm{w})
=
\underbrace{\sum_{i=1}^{N} w_i c_i}_{\mathcal{J}_{\mathrm{linear}}}
-
\underbrace{\tau \sum_{i=1}^{N} w_i \log \frac{w_i}{p_i}}_{\mathcal{J}_{\mathrm{KL}}}
+
\underbrace{\eta \log| \bm{\Sigma}_{\bm w}|}_{\mathcal{J}_{\mathrm{div}}},
\end{equation}
subject to the simplex constraint
\[
\Delta^{N-1}
=
\left\{
\bm{w}\in\mathbb{R}^{N}
\mid
\sum_{i=1}^{N}w_i=1,\ 
w_i\ge 0
\right\}.
\]
We assume \(p_i>0\) for all \(i\), \(\tau>0\), and \(\eta\ge0\).

\begin{proposition}[Strict Concavity and Uniqueness]
The objective \(\mathcal{J}(\bm{w})\) is strictly concave on the simplex. 
Consequently, Eq.~\eqref{eq:weight-optim} admits a unique global maximizer 
\(\bm{w}^*\).
\end{proposition}

\begin{proof}
The linear term \(\mathcal{J}_{\mathrm{linear}}\) is affine in \(\bm w\), so it is both 
convex and concave. The KL term satisfies
\[
D_{\mathrm{KL}}(\bm w\|\bm p)
=
\sum_{i=1}^{N}w_i\log w_i
-
\sum_{i=1}^{N}w_i\log p_i .
\]
Since \(x\log x\) is strictly convex on \(x>0\), 
\(D_{\mathrm{KL}}(\bm w\|\bm p)\) is strictly convex on the relative interior of the 
simplex. Therefore, \(-\tau D_{\mathrm{KL}}(\bm w\|\bm p)\) is strictly concave for 
\(\tau>0\). Equivalently, its Hessian on the positive orthant is diagonal with entries 
\(-\tau/w_i\), which is negative definite.

The diversity term is also concave. Since the ridge-regularized covariance 
\(\bm{\Sigma}_{\bm w}\) is affine in \(\bm w\) and remains positive definite, and since 
\(\log\det(\cdot)\) is concave on the positive definite cone~\cite{boyd2004convex}, 
the composition \(\log|\bm{\Sigma}_{\bm w}|\) is concave in \(\bm w\).

Thus, \(\mathcal{J}(\bm w)\) is the sum of a strictly concave term, a concave term, and 
an affine term, so it is strictly concave on the simplex. The closed simplex is compact 
and convex, and the ridge term makes the diversity term continuous on it. Hence a global 
maximizer exists. Strict concavity further implies that this maximizer is unique.
\end{proof}
\subsection{Derivation of the Iterative Update Rule}
\label{sec:appendix_derivation}

To handle the simplex constraint $\sum w_i = 1$, we introduce a Lagrange multiplier $\nu$. The Lagrangian is:
\begin{equation}
\mathcal{L}(\bm{w}, \nu) =
\sum_{i=1}^N w_i c_i
+ \eta \log |\bm{\Sigma}_{\bm w}|
- \tau \sum_{i=1}^N w_i (\log w_i - \log p_i)
- \nu \left(\sum_{i=1}^N w_i - 1\right).
\end{equation}

Setting the partial derivative with respect to $w_k$ to zero:
\begin{equation}
\frac{\partial \mathcal{L}}{\partial w_k}
=
c_k
+ \eta \frac{\partial}{\partial w_k} \log |\bm{\Sigma}_{\bm w}|
- \tau (\log w_k + 1 - \log p_k)
- \nu
= 0.
\end{equation}

Using the derivative of the log-determinant term,
\begin{equation}
\frac{\partial}{\partial w_k}\log|\bm{\Sigma}_{\bm w}|
=
\mathrm{Tr}\!\left(\bm{\Sigma}_{\bm w}^{-1}
\frac{\partial \bm{\Sigma}_{\bm w}}{\partial w_k}\right)
=
\mathrm{Tr}\!\left(\bm{\Sigma}_{\bm w}^{-1}
\bar{\bm{x}}_k\bar{\bm{x}}_k^\top\right)
=
\bar{\bm{x}}_k^\top \bm{\Sigma}_{\bm w}^{-1}\bar{\bm{x}}_k
=
h_k(\bm{w}),
\end{equation}
we substitute this back into the stationarity condition:
\begin{equation}
c_k + \eta h_k(\bm{w}) - \tau \log w_k + \tau \log p_k - \tau - \nu = 0.
\end{equation}

Rearranging terms to solve for $\log w_k$:
\begin{equation}
\tau \log w_k = \tau \log p_k + c_k + \eta h_k(\bm{w}) - (\nu + \tau).
\end{equation}

Exponentiating both sides yields the general solution form:
\begin{equation}
w_k^* \propto p_k \exp \left( \frac{c_k + \eta h_k(\bm{w})}{\tau} \right).
\end{equation}

This reveals that the optimal weight is determined by the contribution score, the diversity score, and the prior probability $p_k$.

\paragraph{Solution under Uniform Prior.} 
In the standard DRIVE setting (as described in Algorithm~\ref{alg:reweighting}), we assume a non-informative, uniform prior where $p_k = 1/N$ for all samples. Substituting $p_k = 1/N$ into the equation above:
\begin{equation}
w_k^* \propto \frac{1}{N} \exp \left( \frac{c_k + \eta h_k(\bm{w})}{\tau} \right).
\end{equation}
Since the constant factor $1/N$ is common to all samples, it cancels out during the normalization step $\sum w_i = 1$. Thus, under the uniform prior assumption, the update rule simplifies to the standard Softmax form:
\begin{equation}
\label{eq:fixed_point}
w_k^{(t+1)} = \frac{\exp \left( \frac{c_k + \eta h_k(\bm{w}^{(t)})}{\tau} \right)}{\sum_{j=1}^N \exp \left( \frac{c_j + \eta h_j(\bm{w}^{(t)})}{\tau} \right)}.
\end{equation}
This confirms that our standard algorithm is a special case of the general KL-regularized framework where the prior distribution over the data is uniform.

\subsection{Optimization for the No-Prior Ablation}
\label{app:no_prior_optimization}

The ``w/o KL Prior'' ablation removes the prior regularization term from
Eq.~\eqref{eq:weight-optim}.  We
optimize the remaining contribution--diversity objective
\[
\bm w^*
=
\arg\max_{\bm w\in\Delta_{S-1}}
\sum_{i=1}^{S} w_i c_i+\eta\log|\bm{\Sigma}_{\bm w}|,
\qquad
\bm{\Sigma}_{\bm w}
=
\sum_{i=1}^{S}w_i\bar{\bm x}_i\bar{\bm x}_i^\top+\epsilon\bm I .
\]
Here, \(\Delta_{S-1}=\{\bm w:w_i\ge0,\sum_i w_i=1\}\). This objective has no
temperature parameter, and its stationarity condition is implicit because
\(h_i(\bm w)=\bar{\bm x}_i^\top\bm{\Sigma}_{\bm w}^{-1}\bar{\bm x}_i\) depends on
\(\bm w\). For active samples, the KKT condition is
\(c_i+\eta h_i(\bm w^*)=\nu\), with
\(c_i+\eta h_i(\bm w^*)\le\nu\) for inactive samples. Thus, the no-prior objective
does not yield a softmax-form closed update.

For consistency with DRIVE, we use the same iterative covariance-refresh
framework, but replace the KL-induced softmax step with simplex-constrained
ascent. At iteration \(t\), we compute
\(h_i^{(t)}=\bar{\bm x}_i^\top(\bm{\Sigma}_{\bm w}^{(t)})^{-1}\bar{\bm x}_i\) and
\(g_i^{(t)}=c_i+\eta h_i^{(t)}\), then update
\[
\bm w^{(t+1)}
=
\Pi_{\Delta_{S-1}}
\left(
\bm w^{(t)}+\alpha_t\bm g^{(t)}
\right),
\]
where \(\Pi_{\Delta_{S-1}}\) is Euclidean projection onto the simplex. This keeps
weights feasible and uses the same local support set, projected features,
ridge-regularized covariance, and iteration budget as full DRIVE.
\subsection{Convergence of the Weight Solver}
\label{sec:appendix_weight_solver}

We analyze Eq.~\eqref{eq:weight-optim} with the fixed-point solver in Eq.~\eqref{eq:DRIVEr_fixed_point}. To match the implementation, we absorb the ridge term into the covariance:
\[
\bm{\Sigma}_{\bm w}
=
\sum_{i=1}^{N} w_i \bar{\bm{x}}_i \bar{\bm{x}}_i^\top
+
\epsilon \bm I,
\qquad
\epsilon>0,
\]
and define \(h_i(\bm w)=\bar{\bm{x}}_i^\top\bm{\Sigma}_{\bm w}^{-1}\bar{\bm{x}}_i\), with \(\bm h(\bm w)=(h_1(\bm w),\ldots,h_N(\bm w))\). Setting \(\epsilon=0\) recovers the unregularized form when the covariance is nonsingular.

\begin{theorem}[Uniqueness and Linear Convergence]
\label{thm:weight_solver}
Assume \(p_i>0\) for all \(i\), \(\tau>0\), \(\eta\ge 0\), and \(\epsilon>0\). Then \(J(\bm w)\) is \(\tau\)-strongly concave on \(\Delta_{N-1}\), so Eq.~\eqref{eq:weight-optim} admits a unique maximizer \(\bm w^*\). Moreover, assume \(\bm h\) is \(L_h\)-Lipschitz from \((\Delta_{N-1},\|\cdot\|_1)\) to \((\mathbb{R}^N,\|\cdot\|_\infty)\). If \(\eta L_h<\tau\), then \(\bm w^*\) is the unique fixed point of Eq.~\eqref{eq:DRIVEr_fixed_point}, and
\[
\|\bm w^{(t)}-\bm w^*\|_1
\le
\xi^t\|\bm w^{(0)}-\bm w^*\|_1,
\qquad
\xi:=\eta L_h/\tau\in(0,1).
\]
Thus, \(\|\bm w^{(T)}-\bm w^*\|_1=\mathcal O(\xi^T)\).
\end{theorem}

\begin{proof}
The linear term in \(J(\bm w)\) is affine, \(\log|\bm{\Sigma}_{\bm w}|\) is concave because \(\bm{\Sigma}_{\bm w}\) depends affinely on \(\bm w\), and \(-\tau\mathrm{KL}(\bm w\|\bm p)\) is \(\tau\)-strongly concave on the simplex. Hence \(J(\bm w)\) is \(\tau\)-strongly concave and has a unique maximizer \(\bm w^*\).

Let \(\Psi(\bm w)=\tau\mathrm{KL}(\bm w\|\bm p)+\iota_{\Delta_{N-1}}(\bm w)\). Since \(\Psi\) is \(\tau\)-strongly convex, \(\|\nabla\Psi^*(\bm a)-\nabla\Psi^*(\bm b)\|_1\le \tau^{-1}\|\bm a-\bm b\|_\infty\). Writing the solver map as \(T(\bm w)=\nabla\Psi^*(\bm c+\eta\bm h(\bm w))\), for any \(\bm u,\bm v\in\Delta_{N-1}\),
\[
\|T(\bm u)-T(\bm v)\|_1
\le
\tau^{-1}\eta\|\bm h(\bm u)-\bm h(\bm v)\|_\infty
\le
(\eta L_h/\tau)\|\bm u-\bm v\|_1.
\]
Thus \(T\) is a contraction when \(\xi=\eta L_h/\tau<1\).

The optimality condition \(0\in-\bm c-\eta\bm h(\bm w^*)+\partial\Psi(\bm w^*)\) is equivalent to \(\bm w^*=T(\bm w^*)\). Hence \(\bm w^*\) is the unique fixed point, and recursively applying the contraction gives the stated linear rate. Finally, the Lipschitz assumption holds under bounded features: if \(\|\bar{\bm{x}}_i\|_2\le R_x\), then \(\bm{\Sigma}_{\bm w}\succeq\epsilon\bm I\) implies \(|h_i(\bm u)-h_i(\bm v)|\le R_x^4\epsilon^{-2}\|\bm u-\bm v\|_1\), so one may take \(L_h\le R_x^4/\epsilon^2\).
\end{proof}

\subsection{Justification for Randomized Dimensionality Reduction}
\label{app:random_projection}

To reduce the \(\mathcal O(d^3)\) cost of covariance inversion, DRIVE projects centered features by
\(\tilde{\bm{x}}_i=R^\top\bar{\bm{x}}_i\), where \(R\in\mathbb{R}^{d\times k}\) and \(k\ll d\), reducing the dominant cost to \(\mathcal O(k^3)\). We use a normalized Gaussian projection, \(R_{ab}\sim\mathcal N(0,1/k)\), equivalently \(R=G/\sqrt{k}\) with \(G_{ab}\sim\mathcal N(0,1)\). Since \(\bm{\Sigma}_{\bm w}\) is estimated on an augmented local support of size \(S=|\mathcal B|+M\) with FIFO buffer \(M=128\), we set \(k=128\) as a practical support-scale-matched sketch dimension. This choice should be interpreted as preserving the dominant low-dimensional geometry, rather than as a full-rank embedding guarantee when \(\operatorname{rank}(Z)>k\).

\begin{proposition}[Effective subspace embedding for weighted local geometry]
Let \(Z=W^{1/2}\bar X\in\mathbb{R}^{S\times d}\), where \(\bar X\) is the centered feature matrix and \(W=\operatorname{diag}(w_1,\ldots,w_S)\). Let \(U_{\mathrm{eff}}\in\mathbb{R}^{d\times r_{\mathrm{eff}}}\) be an orthonormal basis of an effective \(r_{\mathrm{eff}}\)-dimensional subspace of \(\operatorname{span}(Z^\top)\), with \(r_{\mathrm{eff}}\le k\). Let \(R\in\mathbb{R}^{d\times k}\) have entries \(R_{ab}\sim\mathcal N(0,1/k)\). There exists a universal constant \(C>0\) such that, for any \(\varepsilon,\delta\in(0,1)\), if
\[
k\ge C\varepsilon^{-2}\left(r_{\mathrm{eff}}+\log(1/\delta)\right),
\]
then with probability at least \(1-\delta\),
\[
(1-\varepsilon)I_{r_{\mathrm{eff}}}
\preceq
U_{\mathrm{eff}}^\top RR^\top U_{\mathrm{eff}}
\preceq
(1+\varepsilon)I_{r_{\mathrm{eff}}}.
\]
Equivalently, for every \(z\in\operatorname{span}(U_{\mathrm{eff}})\),
\[
(1-\varepsilon)\|z\|_2^2
\le
\|R^\top z\|_2^2
\le
(1+\varepsilon)\|z\|_2^2 .
\]
\end{proposition}

\noindent
\textit{Remark.}
When the full local rank \(r=\operatorname{rank}(Z)\) exceeds \(k\), the projected covariance is necessarily rank at most \(k\), so the proposition cannot hold over the entire \(\operatorname{span}(Z^\top)\). In this regime, our use of \(k=128\) should be viewed as a low-rank randomized approximation to the dominant local geometry. The approximation quality depends on the effective intrinsic rank of the centered features, and sensitivity to \(k\) is evaluated empirically in Appendix~\ref{app:k_sensitivity}.

\section{Detailed Theoretical Proofs}
\label{sec:appendix_proofs}
This appendix supports the theoretical claims in Sec.~\ref{sec:theory_regimes}.
Sec.~\ref{app:proof_linear_regime} proves the linear-regime results on contribution
reallocation, adaptation toward \(\bm{\theta}^{*}\), and diversity-induced conditioning.
Sec.~\ref{app:proof_local_nonlinear_regime} extends the argument to the local nonlinear
regime with explicit Hessian remainder control.

\subsection{Proof for the Linear Regime in Sec.~\ref{sec:theory_regimes}}
\label{app:proof_linear_regime}
 
\paragraph{Notation used in this subsection.}
We work in the linear regression abstraction of the main text, where
\[
\bm{\Sigma}_{\bm w}\;:=\;\sum_{i=1}^N w_i \bm{x}_i\bm{x}_i^\top
\]
denotes the (\emph{uncentered}) weighted Gram matrix. This is the natural object for the linear regression analysis below and matches the main text linear regime. \emph{Important:} in the local nonlinear regime
(App.~\ref{app:proof_local_nonlinear_regime}) the symbol $\bm{\Sigma}_{\bm w}$ is reused, but it instead denotes the \emph{centered} feature covariance
$\sum_i w_i\bar{\bm x}_i\bar{\bm x}_i^\top$ that appears in the diversity objective. The two usages are intentionally different (mirroring the two regimes in the main text), and we will not mix them within a single subsection.
 
We justify the two linear-regime claims in Sec.~\ref{sec:theory_regimes}:
(i) unified contribution reallocation improves adaptation along conflicting directions, and
(ii) the diversity term $\log|\bm{\Sigma}_{\bm w}|$ improves conditioning, accelerates convergence, and enhances estimation robustness.
 
Consider the weighted linear regression model $f(\bm{x})=\bm{\theta}^{\top}\bm{x}$ with loss
\[
\mathcal{L}_{\bm w}(\bm{\theta})=\frac12\sum_{i=1}^N w_i (y_i-\bm{\theta}^{\top}\bm{x}_i)^2.
\]
We inherit $\bm{\theta}^{*}$ from the main text as the ideal target parameter, and use
$\bm{\theta}_{\bm w}^{\mathrm{opt}}:=\arg\min_{\bm\theta}\mathcal L_{\bm w}(\bm\theta)$ to denote the weighted optimum.
Define the optimization error
$\bm{\epsilon}(t):=\bm{\theta}(t)-\bm{\theta}_{\bm w}^{\mathrm{opt}}$.
The gradient flow $\dot{\bm\theta}=-\nabla\mathcal L_{\bm w}(\bm\theta)$ satisfies
\[
\frac{d}{dt}\bm{\epsilon}(t)=-\bm{\Sigma}_{\bm w}\bm{\epsilon}(t),
\]
so the contraction rate along any unit direction $\bm{u}$ is governed by the effective eigenvalue
$\mu_{\bm{u}}=\bm{u}^{\top}\bm{\Sigma}_{\bm w}\bm{u}$.
 
\paragraph{Contribution-side analysis.}
Let $\bm{u}_{\mathrm{con}}$ denote a conflicting direction, i.e., a direction that is important for the target task but mainly supported by samples with high pre-training loss. Under stability-oriented weighting such as FLOW, samples aligned with $\bm{u}_{\mathrm{con}}$ receive insufficient mass, so
\[
\mu_{\mathrm{con}}
=
\bm{u}_{\mathrm{con}}^{\top}\bm{\Sigma}_{\bm w}\bm{u}_{\mathrm{con}}
=
\sum_{i=1}^N w_i (\bm{x}_i^{\top}\bm{u}_{\mathrm{con}})^2
\]
becomes small, slowing contraction along $\bm{u}_{\mathrm{con}}$ and possibly stalling adaptation.
 
DRIVE uses the unified contribution $c_i=\lambda a_i+(1-\lambda)t_i$. Since $a_i$ favors stability while $t_i$ favors task adaptation, decreasing $\lambda$ increases the relative score of target-relevant conflicting samples. Because the weight update is monotone in $c_i$, this reallocates mass toward samples that support $\bm{u}_{\mathrm{con}}$, thereby increasing $\mu_{\mathrm{con}}$ and accelerating error decay along the conflicting direction.
 
To connect this to the ideal target $\bm{\theta}^{*}$, let $r_i^{*}:=y_i-(\bm{\theta}^{*})^\top\bm{x}_i$ be the residual of $\bm{\theta}^{*}$. Along any unit direction $\bm u$, the projected residual gradient at $\bm{\theta}^{*}$ is
\[
g_{\bm u}^{*}\;:=\;\bm u^\top\nabla\mathcal L_{\bm w}(\bm{\theta}^{*})
=-\sum_i w_i r_i^{*}(\bm{x}_i^\top\bm u).
\]
Restrict $\bm\theta$ to the one-dimensional slice $\{\bm\theta^{*}+s\bm u:s\in\mathbb R\}$. The loss is quadratic in $s$:
\[
\mathcal L_{\bm w}(\bm\theta^{*}+s\bm u)
=\mathcal L_{\bm w}(\bm\theta^{*})+s\,g_{\bm u}^{*}+\tfrac{s^2}{2}\mu_{\bm u},
\]
whose minimizer is at $s^{*}=-g_{\bm u}^{*}/\mu_{\bm u}$, yielding slice-deviation $|g_{\bm u}^{*}|/\mu_{\bm u}$. Whenever $\bm u$ is an eigenvector of $\bm{\Sigma}_{\bm w}$, this slice-deviation coincides exactly with the projection of the full deviation $\bm\theta_{\bm w}^{\mathrm{opt}}-\bm\theta^{*}$ onto $\bm u$. Thus, when the ideal target fits the task-relevant samples well so that $g_{\bm u}^{*}$ is small, increasing $\mu_{\bm u}$ reduces the projected fitting gap to $\bm{\theta}^{*}$. By assigning larger weights to target-relevant conflicting samples, DRIVE can therefore move the model closer to $\bm{\theta}^{*}$ along directions that would otherwise be under-weighted by purely stability-oriented weighting. This proves the contribution-side claim.
 
\paragraph{Diversity-side analysis.}
Under the standard D-optimal design abstraction, \(w_i\) is interpreted as a
relative precision or allocation weight, with conditionally independent noise
satisfying \(\mathrm{Var}(y_i\mid\bm{x}_i)=\sigma^2/w_i\). Then the weighted
least-squares estimator has covariance
\(\mathrm{Cov}(\hat{\bm{\theta}})=\sigma^2\bm{\Sigma}_{\bm w}^{-1}\).
More generally, without this precision-weight interpretation, the exact covariance
takes a sandwich form. We use the D-optimal abstraction to analyze how enlarging
\(|\bm{\Sigma}_{\bm w}|\) reduces generalized variance and improves estimator stability.
 
To see why it also improves conditioning, by Jacobi's formula
\[
\frac{\partial}{\partial w_i}\log|\bm{\Sigma}_{\bm w}|
=
\mathrm{tr}\!\big(\bm{\Sigma}_{\bm w}^{-1}\bm x_i\bm x_i^\top\big)
=
\bm{x}_i^\top \bm{\Sigma}_{\bm w}^{-1}\bm{x}_i.
\]
Let $\bm{\Sigma}_{\bm w}=\sum_{j=1}^d \lambda_j \bm{u}_j\bm{u}_j^\top$ be the eigendecomposition and write $\bm{x}_i=\sum_j \alpha_{ij}\bm{u}_j$. Then
\[
\frac{\partial}{\partial w_i}\log|\bm{\Sigma}_{\bm w}|
=
\sum_{j=1}^d \frac{\alpha_{ij}^2}{\lambda_j}.
\]
Hence directions with small eigenvalues receive larger gradient pressure. The log-determinant term therefore preferentially lifts under-covered directions, flattens the spectrum, and reduces the condition number $\kappa(\bm{\Sigma}_{\bm w})=\lambda_{\max}(\bm{\Sigma}_{\bm w})/\lambda_{\min}(\bm{\Sigma}_{\bm w})$.
 
Finally, under discrete-time gradient descent with step size $\eta_{\mathrm{gd}}<2/\lambda_{\max}(\bm{\Sigma}_{\bm w})$, the worst-case convergence factor is monotonically improved as $\kappa(\bm{\Sigma}_{\bm w})$ decreases. Therefore, by enlarging small eigenvalues and improving conditioning, the diversity term accelerates optimization and enhances robustness in the linear regime.
 
\subsection{Local Nonlinear Analysis with Hessian Remainder Control}
\label{app:proof_local_nonlinear_regime}
 
\paragraph{Notation used in this subsection.}
We extend the local analysis in Sec.~\ref{sec:theory_regimes} by explicitly accounting for the non-Gauss--Newton Hessian component and the local nonlinear remainder. Throughout this subsection, $\bm H_{\bm w}(\bm\theta):=\nabla^2\mathcal L_{\bm w}(\bm\theta)$ is the full weighted Hessian. We use
\[
\bm G_{\bm w}\;:=\;\bm G_{\bm w}(\bm\theta_{\mathrm{pre}})
\;=\;\sum_{i=1}^N w_i\,\bm J_i^\top\bm B_i\bm J_i
\]
for the weighted Gauss--Newton block evaluated at $\bm\theta_{\mathrm{pre}}$, where $\bm G_{\bm w}(\bm\theta)$ is the GN block introduced in the main text, $\bm J_i:=\nabla_{\bm\theta}f_i(\bm\theta_{\mathrm{pre}})$ is the per-sample Jacobian of the model output, and $\bm B_i$ is the corresponding output-space loss Hessian. Following the main text local nonlinear regime, we reserve
\[
\bm{\Sigma}_{\bm w}\;:=\;\sum_i w_i\bar{\bm{x}}_i\bar{\bm{x}}_i^\top
\]
for the \emph{centered} feature covariance used in the diversity objective $\log|\bm{\Sigma}_{\bm w}|$, where $\bar{\bm x}_i$ is the (mini-batch) centered feature. \emph{Note that this $\bm{\Sigma}_{\bm w}$ is intentionally distinct from the (uncentered) weighted Gram matrix denoted by the same symbol in App.~\ref{app:proof_linear_regime}; the two regimes use the natural object for their respective analyses.} The contribution scores $a_i,t_i$, the prior $p_i$, the diversity score $h_i$, the trade-off $\lambda$, the diversity coefficient $\eta$ and the temperature $\tau$ are inherited from the main text.
 
Consider the local neighborhood
$\mathcal B_R=\{\bm{\theta}:\|\bm{\theta}-\bm{\theta}_{\mathrm{pre}}\|_2\le R\}$.
 
\begin{assumption}[Local Hessian smoothness]
\label{ass:hessian_smooth}
There exists $L_H\ge 0$ such that
$\|\bm H_{\bm w}(\bm\theta)-\bm H_{\bm w}(\bm\theta_{\mathrm{pre}})\|_2
\le L_H\|\bm\theta-\bm\theta_{\mathrm{pre}}\|_2$
for all $\bm\theta\in\mathcal B_R$.
\end{assumption}
 
\begin{assumption}[Bounded non-GN residual]
\label{ass:non_gn}
At $\bm\theta_{\mathrm{pre}}$ we have the decomposition $\bm H_{\bm w}(\bm\theta_{\mathrm{pre}})=\bm G_{\bm w}+\bm E_{\bm w}$ with $\|\bm E_{\bm w}\|_2\le\kappa$.
\end{assumption}
 
\begin{assumption}[Spectral comparison on the local subspace]
\label{ass:spec_comp}
Let $\mathcal S$ be the local tangent subspace relevant to the update. There exist $\alpha,\beta>0$ such that for any unit vector $\bm v\in\mathcal S$,
\[
\alpha\,\lambda_{\min}(\bm{\Sigma}_{\bm w})
\;\le\;\bm v^\top\bm G_{\bm w}\bm v
\;\le\;\beta\,\lambda_{\max}(\bm{\Sigma}_{\bm w}).
\]
\end{assumption}
 
Assumption~\ref{ass:spec_comp} formalizes the link between the diversity-side object $\bm{\Sigma}_{\bm w}$ (centered feature covariance) and the GN block $\bm G_{\bm w}$, which is the actual driver of local curvature. It is not an a priori statement that the two are equal but rather that they share a comparable spectrum on the directions that the optimizer can move along.
 
\paragraph{Local strong convexity and smoothness.}
For any $\bm\theta\in\mathcal B_R$, write
\[
\bm H_{\bm w}(\bm\theta)
\;=\;
\bm G_{\bm w}+\bm E_{\bm w}+\bm R_{\bm w}(\bm\theta),
\qquad
\bm R_{\bm w}(\bm\theta)
\;:=\;
\bm H_{\bm w}(\bm\theta)-\bm H_{\bm w}(\bm\theta_{\mathrm{pre}}).
\]
By Assumption~\ref{ass:hessian_smooth}, $\|\bm R_{\bm w}(\bm\theta)\|_2\le L_H R$. Combining with Assumptions~\ref{ass:non_gn}--\ref{ass:spec_comp} and the triangle inequality, for any unit $\bm v\in\mathcal S$ and $\bm\theta\in\mathcal B_R$,
\[
\alpha\,\lambda_{\min}(\bm{\Sigma}_{\bm w})-\kappa-L_H R
\;\le\;
\bm v^\top\bm H_{\bm w}(\bm\theta)\bm v
\;\le\;
\beta\,\lambda_{\max}(\bm{\Sigma}_{\bm w})+\kappa+L_H R.
\]
Define
\[
m_{\bm w}\;:=\;\alpha\,\lambda_{\min}(\bm{\Sigma}_{\bm w})-\kappa-L_H R,
\quad
M_{\bm w}\;:=\;\beta\,\lambda_{\max}(\bm{\Sigma}_{\bm w})+\kappa+L_H R.
\]
If $m_{\bm w}>0$, then on $\mathcal S\cap\mathcal B_R$ the weighted objective is locally $m_{\bm w}$-strongly convex and $M_{\bm w}$-smooth, with local condition number bounded by $M_{\bm w}/m_{\bm w}$.
 
This yields two standard consequences. First, if gradient descent uses step $0<\zeta\le 1/M_{\bm w}$ and stays in $\mathcal B_R\cap\mathcal S$, the local objective gap contracts at rate $1-m_{\bm w}/M_{\bm w}$ per step. Second, if the empirical gradient is perturbed by $\bm e$ with $\|\bm e\|_2\le\delta$, the corresponding perturbed local stationary point $\tilde{\bm\theta}_{\bm w}$ satisfies
\[
\|\tilde{\bm\theta}_{\bm w}-\bm\theta_{\bm w}^{\mathrm{loc}}\|_2
\;\le\;\delta/m_{\bm w},
\]
where $\bm\theta_{\bm w}^{\mathrm{loc}}$ denotes the unperturbed local stationary point. Thus, a larger $m_{\bm w}$ improves both local descent and robustness, provided that the nonlinear residual $\kappa+L_H R$ is not dominant.
 
\paragraph{Connection to the ideal target.}
Suppose $\bm\theta^{*}\in\mathcal B_R\cap\mathcal S$. By local $m_{\bm w}$-strong convexity on $\mathcal S$,
\[
\|\bm\theta_{\bm w}^{\mathrm{loc}}-\bm\theta^{*}\|_2
\;\le\;
m_{\bm w}^{-1}\,
\|\Pi_{\mathcal S}\nabla\mathcal L_{\bm w}(\bm\theta^{*})\|_2,
\]
where $\Pi_{\mathcal S}$ is the projection onto the local update subspace. Therefore DRIVE can approach the ideal target when (a) the projected weighted residual gradient at $\bm\theta^{*}$ is small and (b) the effective lower curvature $m_{\bm w}$ is large. The contribution term targets condition (a) by upweighting task-relevant samples that are better explained by $\bm\theta^{*}$; the diversity term targets condition (b) by improving the centered feature covariance spectrum.
 
\paragraph{Connecting to the DRIVE weighting rule.}
With the diversity score fixed, one exponential weighting step satisfies
\[
w_i(\lambda)\;\propto\;p_i\exp\!\big((\lambda a_i+(1-\lambda)t_i+\eta h_i)/\tau\big).
\]
For each sample, let \(\bm J_i=\nabla_{\bm{\theta}} f(\bm{x}_i;\bm{\theta}_{\mathrm{pre}})\) denote the model-output Jacobian and let \(\bm B_i=\nabla_{\bm{o}}^2\ell(\bm{o},\bm y_i)|_{\bm{o}=f(\bm{x}_i;\bm{\theta}_{\mathrm{pre}})}\) denote the output-space loss Hessian. For a task-relevant direction $\bm v\in\mathcal S$, define the per-sample contribution to $\bm v^\top\bm G_{\bm w}\bm v$ as
\[
q_i(\bm v)\;:=\;\bm v^\top\bm J_i^\top\bm B_i\bm J_i\bm v\;\ge 0.
\]
Differentiating the softmax weights and simplifying yields
\[
\frac{d}{d\lambda}\sum_{i=1}^N w_i(\lambda)\,q_i(\bm v)
\;=\;
\tau^{-1}\operatorname{Cov}_{w(\lambda)}\!\big(q_i(\bm v),\,a_i-t_i\big).
\]
Therefore, when task-relevant conflicting samples have larger task-adaptation scores, i.e.,
\[
\operatorname{Cov}_{w(\lambda)}\!\big(q_i(\bm v),\,t_i-a_i\big)\;>\;0,
\]
decreasing $\lambda$ increases the effective spectral mass along $\bm v$, which enlarges $\bm v^\top\bm G_{\bm w}\bm v$ and, through Assumption~\ref{ass:spec_comp}, supports a larger $m_{\bm w}$ and a smaller projected fitting gap toward $\bm\theta^{*}$.
 
The diversity term $\log|\bm{\Sigma}_{\bm w}|$ further discourages collapse of the centered feature covariance: as in App.~\ref{app:proof_linear_regime}, its eigenvalue derivative is $1/\lambda_j(\bm{\Sigma}_{\bm w})$, so it places stronger pressure on under-covered feature directions with small eigenvalues, encouraging a more balanced feature spectrum. Under Assumption~\ref{ass:spec_comp}, this can increase $\lambda_{\min}(\bm{\Sigma}_{\bm w})$, hence $m_{\bm w}$, and improve the local condition number $M_{\bm w}/m_{\bm w}$ whenever $\kappa+L_H R$ is controlled.
 
In summary, DRIVE provides a local nonlinear guarantee with explicit Hessian remainder control: contribution reallocation reduces the residual gradient toward the ideal target $\bm\theta^{*}$ along task-relevant conflicting directions, while diversity regularization improves feature-space coverage and supports better local conditioning. This result is local in $\mathcal B_R$ and does not claim a global non-convex guarantee.
\section{More Experimental Details}
\label{full_details}
In this section, we provide a comprehensive overview of our experimental setup to ensure reproducibility. We begin by detailing the implementation and configurations of the baseline methods in \ref{app:baselines}. Subsequently, we provide descriptions of the diverse downstream vision datasets utilized in our experiments in \ref{sec:evaluation_details}. We then outline the specific evaluation metrics and benchmarks employed for Large Language Models in \ref{sec:llm_eval_details}. Finally, we present the implementation details, including hyperparameter settings and training configurations for both vision and language models, in \ref{app:impl_details}.

\subsection{Baseline Details}
\label{app:baselines}

We compare DRIVE against SOTA methods spanning parameter-space regularization, PEFT, sample selection, sample reweighting, and token-level loss reweighting. All are implemented under the \textit{data-oblivious} setting (no access to pre-training data).

\textbf{Standard Fine-Tuning (Standard FT):} Updates all parameters using standard cross-entropy loss. It typically suffers severe catastrophic forgetting, serving as the lower bound for general capability preservation.

\textbf{Linear Probing (LP):} Freezes the pre-trained backbone and only trains the task-specific head. While maximizing retention, it often yields sub-optimal task adaptation.

\textbf{$\ell_2$-Regularization~\cite{kirkpatrick2017overcoming}:} Constrains parameter drift from pre-trained weights $\theta_\text{pre}$ via an $\ell_2$ penalty: $\mathcal{L}(\boldsymbol{\theta}) = \sum \ell_i(\boldsymbol{\theta}) + \lambda \|\boldsymbol{\theta} - \boldsymbol{\theta}_\text{pre}\|_2^2$. Higher $\lambda$ favors retention over adaptation.

\textbf{WiSE-FT \cite{wortsman2022robust}:} Linearly interpolates between pre-trained ($\boldsymbol{\theta}_\text{pre}$) and fine-tuned ($\boldsymbol{\theta}_{\text{ft}}$) weights: $\boldsymbol{\theta}_{\text{wise}} = (1-\alpha)\boldsymbol{\theta}_\text{pre} + \alpha \boldsymbol{\theta}_{\text{ft}}$. We set $\alpha = 0.5$ as pre-training validation data is unavailable.

\textbf{LoRA \cite{hu2022lora}:} Injects trainable rank decomposition matrices into Transformer layers while freezing the backbone. Limiting trainable parameters inherently mitigates forgetting \cite{biderman2024lora}. We set rank $r=64$.

\textbf{DEFT-UCS \cite{das2024deft}:} A dataset pruning framework that selects a representative ``core-set" via K-Means clustering on pre-trained embeddings. We use the recommended ``hard sampling" strategy. Unlike DRIVE's continuous weighting, DEFT-UCS performs discrete data sub-sampling prior to training.

\textbf{FLOW \cite{sanyal2025upweighting}:} Dynamically weights samples inversely to their pre-trained loss: $w_i \propto \exp(-\ell_i(\boldsymbol{\theta}_\text{pre}) / \tau)$, upweighting ``easy" samples to anchor the model. FLOW lacks explicit task adaptation mechanisms and ignores diversity.

\textbf{TALR \cite{lin2026sft}:} Token-Adaptive Loss Reweighting dynamically rescales token-level losses according to predicted token probabilities, down-weighting low-probability hard tokens to reduce their disproportionate impact on general capability degradation.

\subsection{Vision Datasets and Evaluation Details}
\label{sec:evaluation_details}
 
For vision tasks, we primarily utilize Top-1 Accuracy as the evaluation metric. We report the Target Accuracy on the fine-tuning datasets to assess adaptation performance. To monitor catastrophic forgetting, we evaluate the ImageNet-1K Retention Accuracy, where the pre-trained classification head is restored during inference~\cite{russakovsky2015imagenet}. The Overall Avg is calculated as the arithmetic mean of the Target Accuracy and the ImageNet-1K Retention Accuracy. For the pre-trained loss used to compute \(a_i\), we freeze the ImageNet-pretrained backbone and use a task-specific linear-probe head trained on the downstream training set.

\begin{itemize}[leftmargin=*,nosep]
\item ImageNet-1K~\cite{russakovsky2015imagenet}: Large-scale pre-training source with 1,000 diverse classes. \item CIFAR-10~\cite{krizhevsky2009learning}: Standard benchmark with 10 classes of 32×32 images. \item CIFAR-100~\cite{krizhevsky2009learning}: Fine-grained extension of CIFAR-10 with 100 classes. \item Caltech101~\cite{fei2004learning}: General object recognition benchmark with 101 categories. \item Flowers102~\cite{nilsback2008automated}: Fine-grained classification of 102 flower species. \item Cars~\cite{krause20133d}: Fine-grained dataset containing 196 car models. \item Dogs~\cite{khosla2011novel}: Fine-grained dataset covering 120 dog breeds. \item Food101~\cite{bossard2014food}: Large-scale dataset with 101 food categories. \item Beans~\cite{elfatimi2022beans}: Agricultural dataset for leaf disease classification. \item EuroSAT~\cite{helber2019eurosat}: Satellite imagery for land-use and land-cover detection. \end{itemize}

\subsection{LLM Datasets and Evaluation Details}
\label{sec:llm_eval_details}

We assess model performance across three dimensions: Commonsense Reasoning, General Capabilities (MMLU, MBPP) to monitor forgetting, and Target Tasks for mathematical reasoning. Code generation capabilities are further evaluated using HumanEval and the Magicoder dataset. For LLM training and contribution scoring, each sequence is treated as one sample, and sample-level cross-entropy is averaged over non-padding target tokens, avoiding length bias.

We employ the lm-evaluation-harness~\cite{gao12608602framework} library for standardized assessment. We report Pass@1 for code generation and Accuracy for other tasks. The specific configurations are: General Capabilities: MMLU is evaluated in a 5-shot setting. Commonsense Reasoning: HellaSwag, PIQA, OpenBookQA, ARC-Easy, and ARC-Challenge are evaluated in a 0-shot setting. Mathematical Reasoning: GSM8K is assessed using 5-shot prompting. Code Generation: We evaluate HumanEval in a 0-shot setting and MBPP in a 3-shot setting. For aggregation, we calculate the General Capability score as the average of the three general benchmarks, while the Overall Avg. is computed as the mean of this General Capability score and the Target Task accuracy.

\begin{itemize}[leftmargin=*,nosep]
    \item HellaSwag ~\cite{zellers2019hellaswag}: A commonsense reasoning benchmark for selecting a plausible continuation.
    
    \item ARC (Easy \& Challenge)~\cite{clark2018think}: Grade-school science questions. \textit{ARC-e} targets basic reasoning, while \textit{ARC-c} involves complex questions requiring deep inference.
    
    \item PIQA ~\cite{bisk2020piqa}: Evaluates physical commonsense through solutions to everyday physical goals.
    
    \item OpenBookQA ~\cite{mihaylov2018can}: Elementary science QA combining open-book facts with world knowledge.
    
    \item MMLU ~\cite{hendrycks2021measuring}: A comprehensive benchmark covering 57 subjects to evaluate knowledge breadth.
    
    \item MBPP~\cite{austin2021program}: Contains 974 basic Python problems requiring code satisfying provided test cases.
    
    \item GSM8K ~\cite{cobbe2021training}: 1,319 grade-school math word problems requiring multi-step arithmetic reasoning.

     \item MetaMathQA~\cite{yu2024metamath}: Bootstraps math questions into instruction data for reasoning.

    \item Magicoder ~\cite{wei2024magicoder}: Leverages open-source snippets to create coding problems for instruction tuning.
    
    \item HumanEval ~\cite{chen2021evaluating}: Evaluates code synthesis via 164 handwritten Python problems with unit tests.
\end{itemize}

\subsection{Implementation Details}
\label{app:impl_details}
\subsubsection{Vision Model Implementation Details}

For the vision benchmarks, we performed an extensive hyper-parameter search over a grid of learning rates (lr $\in \{5e-2, 1e-2, 5e-3, 1e-3, 5e-4, 1e-4\}$) for standard fine-tuning, linear probing, and our proposed method. Experiments were conducted using ResNet18 and ResNet50 architectures across nine downstream datasets: CIFAR-10, CIFAR-100, Flowers102, Caltech101, Stanford Dogs, Stanford Cars, Food101, Beans, and EuroSAT. We use a global batch size of 128 for all vision experiments. We selected the optimal learning rates using a held-out validation split from the downstream training data, and used the final test set only once for reporting.

\begin{table}[ht]
\vspace{-0.05in}
   \caption{Hyperparameter configurations for DRIVE across various datasets.}
    \centering
    \small
    \begin{tabular}{l|ccccccccc}
    \toprule
   Datasets & \scalebox{0.8}{CIFAR10} & \scalebox{0.8}{CIFAR100} & \scalebox{0.8}{Flowers} & \scalebox{0.8}{Caltech} & \scalebox{0.8}{Dogs} & \scalebox{0.8}{Cars} & \scalebox{0.8}{Food} & \scalebox{0.8}{Beans} & \scalebox{0.8}{EuroSAT} \\
    \midrule
     Epochs   & 20 & 25 & 25 & 30 & 30 & 30 & 30 & 20 & 20 \\
        \midrule
        ResNet18 & 1e-3 & 5e-3 & 5e-2 & 1e-2 & 5e-3 & 1e-2 & 1e-2 & 1e-2 & 1e-3 \\
    
        ResNet50 & 5e-4 & 1e-3 & 1e-2 & 1e-2 & 5e-3 & 1e-2 & 1e-2 & 1e-2 & 1e-3 \\
        \bottomrule
    \end{tabular}
\vspace{-10pt}
\end{table}

To ensure a fair comparison, we maintained consistent training configurations across all baselines. Specifically, we utilized an SGD optimizer with a momentum of 0.9 and a weight decay of $5 \times 10^{-4}$. The learning rate followed a Cosine Annealing schedule for all fine-tuning tasks, whereas linear probing employed a constant learning rate.

\subsubsection{LLMs Implementation Details}
All Large Language Models (LLMs) are fine-tuned using the AdamW optimizer paired with a cosine learning rate scheduler. To ensure a fair comparison, we enforce a uniform training configuration across all models (Gemma-2-2B, Llama-3.2-3B, and Llama-3.1-8B). 
Regarding dataset-specific configurations, for MetaMathQA, we set the global batch size to 128 with a maximum sequence length of 1024; for Magicoder, we use a global batch size of 64 and a maximum sequence length of 4096. 
Common settings include a warmup ratio of 0.03 and a training duration of 3 epochs. To ensure statistical reliability, all results are reported as the mean and standard deviation across five independent runs. 
All experiments were conducted on four NVIDIA A100 GPUs. For LoRA baselines, the rank is set to $r=64$. 
The learning rate is tuned for each method--model combination using validation performance. LoRA uses $2\mathrm{e}{-4}$ following the PEFT setting. All full-parameter methods, including Standard FT, FLOW, $\ell_2$-Reg., DEFT-UCS, and DRIVE, share the same grid $\{5\mathrm{e}{-6}, 1\mathrm{e}{-5}, 2\mathrm{e}{-5}\}$. DRIVE selects $2\mathrm{e}{-5}$ for Gemma-2-2B and Llama-3.2-3B, and $1\mathrm{e}{-5}$ for Llama-3.1-8B.

\section{Additional Experimental Results and Analysis}
\label{app:extended_experiments}

In this appendix, we provide a comprehensive empirical analysis to complement the main text. We begin by detailing fine-grained performance across diverse downstream vision datasets in Appendix~\ref{app:extended_vision_experiments} and verifying architectural generalization on Vision Transformers in Appendix~\ref{app:vit_experiments}. Subsequently, Appendix~\ref{app:extended_llm_experiments} extends the LLM evaluation to code generation tasks. We then validate our design choices through in-depth sensitivity analyses: we examine the magnitude and temporal schedule of the balance parameter $\lambda$ in Appendix~\ref{app:lambda_schedule} and Appendix~\ref{app:tau_impact}, respectively, followed by the prior regularization threshold $\tau$ in Appendix~\ref{app:regularization_tau}, the diversity coefficient $\eta$ in Appendix~\ref{app:eta_impact}, and component-wise contributions on vision benchmarks in Appendix~\ref{app:ablation_vision}. Finally, we report the computational efficiency of DRIVE in Appendix~\ref{appendix:efficiency}, demonstrating its scalability compared to baselines.

\subsection{Analysis of Vision Tasks}
\label{app:extended_vision_experiments}

To evaluate the robustness of DRIVE, we present the fine-grained performance on six diverse downstream vision datasets: Caltech101, Stanford Dogs, Stanford Cars, Food101, Beans, and EuroSAT.

\textbf{Target Task Adaptation.} Table~\ref{tab:combined_acc} details the top-1 accuracy on each specific target dataset after fine-tuning. DRIVE consistently demonstrates strong performance across both ResNet-18 and ResNet-50 architectures. Notably, on fine-grained classification tasks such as Stanford Cars and Food101, DRIVE achieves significant improvements over competing baselines. For instance, on the Food101 dataset with ResNet-50, DRIVE outperforms FLOW, an anti-forgetting oriented weighting method, while maintaining a superior adaptation score compared to other methods.

\textbf{Anti-forgetting Capability.} Table~\ref{tab:combined_acc} reports the Top-1 ImageNet-1K accuracy (the pre-training task) after fine-tuning. A major challenge in fine-tuning is the catastrophic drop in pre-trained performance, which is particularly severe in Standard FT. DRIVE effectively mitigates this issue, maintaining significantly higher retention accuracy compared to Standard FT. Furthermore, DRIVE demonstrates anti-forgetting capabilities that are competitive with or superior to $\ell_2$-Regularization, while simultaneously achieving higher accuracy on the target tasks.

\textbf{Consolidated Performance Analysis.} Table~\ref{tab:dataset_means} presents the average accuracy across the six datasets, reflecting the overall balance between representation preservation and task adaptation. Across both ResNet-18 and ResNet-50 architectures, DRIVE achieves the highest average accuracy on the majority of datasets, including Food101 and EuroSAT. While methods like Linear Probing preserve pre-trained knowledge perfectly, they often fail to adapt effectively to new tasks, resulting in lower overall average scores compared to DRIVE. This demonstrates that DRIVE achieves a robust stability-plasticity trade-off, successfully resolving the dilemma.

\begin{table*}[t]
  \centering
  \small
  \definecolor{highlightblue}{HTML}{E0F7FA}
  \caption{Combined performance comparison across six datasets. For each dataset, \textbf{Tgt} denotes the target task accuracy (higher is better for adaptation), and \textbf{Gen} denotes the ImageNet-1K accuracy (higher is better for anti-forgetting)~(Accuracy in \%). Best results are \textbf{bold}, second-best results are \underline{underlined}.}
  \vspace{-7pt}
  \label{tab:combined_acc}
  \resizebox{\textwidth}{!}{
    \begin{tabular}{ll|cc|cc|cc|cc|cc|cc}
    \toprule
    \multirow{2}{*}{} & \multirow{2}{*}{\textbf{Method}} & \multicolumn{2}{c|}{\textbf{Caltech101}} & \multicolumn{2}{c|}{\textbf{Dogs}} & \multicolumn{2}{c|}{\textbf{Cars}} & \multicolumn{2}{c|}{\textbf{Food101}} & \multicolumn{2}{c|}{\textbf{Beans}} & \multicolumn{2}{c}{\textbf{EuroSAT}} \\
    \cmidrule(lr){3-4} \cmidrule(lr){5-6} \cmidrule(lr){7-8} \cmidrule(lr){9-10} \cmidrule(lr){11-12} \cmidrule(lr){13-14}
          &        & \textbf{Gen} & \textbf{Tgt} & \textbf{Gen} & \textbf{Tgt} & \textbf{Gen} & \textbf{Tgt} & \textbf{Gen} & \textbf{Tgt} & \textbf{Gen} & \textbf{Tgt} & \textbf{Gen} & \textbf{Tgt} \\
    \midrule
    \multirow{7}[0]{*}{\rotatebox{90}{\textbf{ResNet18}}} & Linear probing & \textbf{69.76} & 88.63 & \textbf{69.76} & 78.55 & \textbf{69.76} & 50.21 & \textbf{69.76} & 60.03 & \textbf{69.76} & 85.78 & \textbf{69.76} & 91.35 \\
          & Standard FT & 35.79 & \textbf{94.89} & 58.01 & 80.30 & 1.84  & \textbf{87.64} & 0.23  & 74.03 & 48.95 & \textbf{97.66} & 24.79 & \underline{96.79} \\
          & $\ell_2$-Regularization & 45.16 & \textbf{94.89} & 60.11 & 81.33 & 20.79 & \underline{83.79} & 10.20 & \underline{75.69} & 40.77 & \underline{96.88} & 39.09 & 95.42 \\
          & WiSE-FT & 62.96 & 90.60 & 66.89 & 75.42 & 32.57 & 65.46 & 20.79 & 70.30 & 44.87 & 96.41 & 50.75 & 90.18 \\
          & DEFT-UCS & 59.69 & 89.59 & 59.98 & 79.62 & 28.54 & 60.31 & 10.77 & 70.94 & 41.97 & 95.31 & 50.32 & 90.01 \\
          & FLOW  & \underline{65.72} & 90.46 & 66.31 & \underline{82.22} & \underline{57.65} & 61.85 & \underline{53.79} & 73.43 & \underline{65.79} & 94.06 & \underline{66.00} & 96.26 \\
          \rowcolor{highlightblue} \cellcolor{white}
          & \textbf{DRIVE~(Ours)} & 64.91 & \underline{93.90} & \underline{66.95} & \textbf{84.82} & 55.89 & 76.72 & 53.61 & \textbf{77.48} & 64.78 & 94.84 & 65.78 & \textbf{97.17} \\
    \midrule
    \multirow{7}[0]{*}{\rotatebox{90}{\textbf{ResNet50}}} & Linear probing & \textbf{80.10} & 92.04 & \textbf{80.10} & 85.40 & \textbf{80.10} & 42.78 & \textbf{80.10} & 53.08 & \textbf{80.10} & 86.41 & \textbf{80.10} & 92.35 \\
          & Standard FT & 42.53 & \textbf{97.73} & 69.25 & 87.15 & 0.92  & \textbf{90.27} & 0.42  & 76.80 & 53.79 & \textbf{98.75} & 34.79 & \underline{97.04} \\
          & $\ell_2$-Regularization & 45.73 & \underline{95.82} & \underline{78.96} & 82.86 & 14.61 & \underline{89.97} & 17.63 & 76.11 & 60.21 & 95.31 & 66.46 & 94.21 \\
          & WiSE-FT & \underline{74.76} & 92.71 & 74.56 & 74.30 & 10.84 & 63.79 & 27.87 & 72.33 & 70.64 & 92.34 & 71.54 & 95.43 \\
          & DEFT-UCS & 60.78 & 90.68 & 72.54 & 80.67 & 26.87 & 67.15 & 22.95 & 71.23 & 64.56 & 90.47 & 67.64 & 92.33 \\
          & FLOW  & 71.63 & 92.79 & 75.57 & \underline{89.63} & \underline{67.98} & 70.76 & 55.78 & \underline{77.68} & \underline{77.19} & 95.00 & \underline{75.32} & 96.37 \\
          \rowcolor{highlightblue} \cellcolor{white}
          & \textbf{DRIVE~(Ours)} & 71.03 & 94.96 & 73.04 & \textbf{89.94} & 66.91 & 77.68 & \underline{67.40} & \textbf{79.37} & 76.34 & \underline{96.41} & 74.55 & \textbf{98.44} \\
    \bottomrule
    \end{tabular}%
  }
\end{table*}

\begin{table*}[ht]
  \centering
  \small
  \caption{Comparison of Average Accuracy~(\%) across six downstream vision datasets. This metric reflects the overall balance between representation preservation and task adaptation. Best results are \textbf{bold}, second-best results are \underline{underlined}.}
  \vspace{-7pt}
  \label{tab:dataset_means}
  \definecolor{highlightblue}{HTML}{E0F7FA}
    \begin{tabular}{ll|cccccc}
    \toprule
    \textbf{Model} & \textbf{Method} & \textbf{Caltech101} & \textbf{Dogs} & \textbf{Cars} & \textbf{Food101} & \textbf{Beans} & \textbf{EuroSAT} \\
    \midrule
    \multirow{7}{*}{\textbf{ResNet18}} 
    & Linear probing         & \underline{79.20} & 74.16 & \underline{59.99} & \underline{64.90} & 77.77 & 80.56 \\
    & Standard FT            & 65.34 & 69.16 & 44.74 & 37.13 & 73.31 & 60.79 \\
    & $\ell_2$-Regularization & 70.03 & 70.72 & 52.29 & 42.95 & 68.83 & 67.26 \\
    & WiSE-FT                & 76.78 & 71.16 & 49.02 & 45.55 & 70.64 & 70.47 \\
    & DEFT-UCS               & 74.64 & 69.80 & 44.43 & 40.86 & 68.64 & 70.17 \\
    & FLOW                   & 78.09 & \underline{74.27} & 59.75 & 63.61 & \textbf{79.93} & \underline{81.13} \\
    \rowcolor{highlightblue} \cellcolor{white}
    & \textbf{DRIVE (Ours)}   & \textbf{79.41} & \textbf{75.89} & \textbf{66.31} & \textbf{65.55} & \underline{79.81} & \textbf{81.48} \\
    \midrule
    \multirow{7}{*}{\textbf{ResNet50}} 
    & Linear probing         & \textbf{86.07} & \textbf{82.75} & 61.44 & 66.59 & 83.26 & \underline{86.23} \\
    & Standard FT            & 70.13 & 78.20 & 45.60 & 38.61 & 76.27 & 65.92 \\
    & $\ell_2$-Regularization & 70.78 & 80.91 & 52.29 & 46.87 & 77.76 & 80.34 \\
    & WiSE-FT                & \underline{83.74} & 74.43 & 37.32 & 50.10 & 81.49 & 83.49 \\
    & DEFT-UCS               & 75.73 & 76.61 & 47.01 & 47.09 & 77.52 & 79.99 \\
    & FLOW                   & 82.21 & \underline{82.60} & \underline{69.37} & \underline{66.73} & \underline{86.10} & 85.85 \\
    \rowcolor{highlightblue} \cellcolor{white}
    & \textbf{DRIVE (Ours)}   & 83.00 & 81.49 & \textbf{72.30} & \textbf{73.39} & \textbf{86.38} & \textbf{86.50} \\
    \bottomrule
    \end{tabular}
    \vspace{-14pt}
\end{table*}

\subsection{Experiments on Vision Transformers}

\label{app:vit_experiments}

To verify the architectural generalization of DRIVE beyond Convolutional Neural Networks (CNNs), we extended our evaluation to the Vision Transformer (ViT) architecture. Specifically, we fine-tuned a pre-trained ViT-B/16\cite{dosovitskiy2021an} model on the Food101 dataset, a challenging fine-grained classification benchmark.

As presented in Table~\ref{tab:vit_food101}, the results verify that DRIVE generalizes effectively to Transformer-based architectures, offering a superior trade-off compared to existing baselines on the challenging fine-grained Food101 benchmark.

While Standard FT achieves high target adaptation, it suffers from severe catastrophic forgetting. Specifically, the ImageNet-1K accuracy drops by over 24\% compared to the pre-trained baseline, confirming that the high plasticity of ViTs leads to significant drift from the general manifold without regularization.

\begin{wraptable}{r}{0.50\textwidth}

\centering
\setlength{\tabcolsep}{4pt}
\caption{Performance comparison (Accuracy in \%) of ViT-B/16 on Food101. Best in \textbf{bold}, second in \underline{underlined}.}
\vspace{-7pt}
\label{tab:vit_food101}
\scriptsize
\definecolor{highlightblue}{HTML}{E0F7FA}
\resizebox{\linewidth}{!}{%
\begin{tabular}{lccc}
\toprule
Method & IN-1K & Food101 & Avg. \\
\midrule
Pre-trained & 82.95 & - & - \\
\midrule
Standard FT & 58.42 & \textbf{92.24} & 75.33 \\
Linear Probe & \textbf{82.95} & 83.93 & \underline{83.44} \\
$l_2$-Reg. & 57.94 & \underline{91.78} & 74.86 \\
WiSE-FT & 55.38 & 90.64 & 73.01 \\
DEFT-UCS & 67.35 & 90.15 & 78.75 \\
FLOW & \underline{76.72} & 89.86 & 83.29 \\
\rowcolor{highlightblue}
\textbf{DRIVE (Ours)} & 76.15 & 91.45 & \textbf{83.80} \\
\bottomrule
\end{tabular}%
}
\vspace{-0.25in}
\end{wraptable}

DRIVE successfully mitigates this issue, recovering the general accuracy by nearly 18\% compared to Standard FT. Crucially, this substantial gain in stability comes with negligible cost to adaptation, as DRIVE maintains a competitive target accuracy comparable to full fine-tuning. DRIVE demonstrates better adaptability than rigid baselines. It surpasses Linear Probe by more than 7.5\% on the target task, proving that DRIVE effectively learns task-specific features where simple probing fails. Furthermore, DRIVE outperforms the baseline FLOW by approximately 1.6\% in target accuracy, suggesting that our contribution-based weighting captures fine-grained patterns more effectively than loss-based strategies.

\subsection{Experiments on Language Models}
\label{app:extended_llm_experiments}

Table~\ref{tab:code1} presents results of fine-tuning LLMs on Magicoder. We evaluate target-task performance using HumanEval and general capability retention using Commonsense Reasoning (CS), MMLU, and GSM8K. DRIVE consistently achieves the highest overall average across evaluated models, demonstrating a strong balance between stability and plasticity. Particularly on Llama-3.2-3B, DRIVE improves the overall score by 2.86\% over LoRA and by 0.98\% over FLOW, while substantially improving general capability retention over Standard FT. Although some retention-oriented baselines, such as TALR on Llama-3.2-3B, obtain slightly higher general capability, they achieve lower HumanEval accuracy. Overall, DRIVE attains the best overall score across all evaluated models.

\begin{table*}[t]
      \caption{Performance comparison of models fine-tuned on the Magicoder dataset across various benchmarks. Best results are \textbf{bolded} and second-best results are \underline{underlined}.}
    \vspace{-7pt}
    \label{tab:code1}
    \scriptsize
    \definecolor{highlightblue}{HTML}{E0F7FA}
    \centering
    \setlength{\tabcolsep}{12pt}
    \renewcommand{\arraystretch}{1}
    \begin{adjustbox}{max width=\textwidth}
\begin{tabular}{@{}c|c|cccc|c|>{\columncolor{white}[\tabcolsep][0pt]}c@{}}
        \toprule
        \multirow{2}{*}[-0.6ex]{\centering\textbf{Model}} &
        \multirow{2}{*}[-0.6ex]{\centering\textbf{Method}} &
        \multicolumn{4}{c|}{\textbf{General Capability Acc.~(\%)}} &
        \textbf{Target Acc.~(\%)} &
        \textbf{Overall~(\%)} \\
        \cmidrule(lr){3-6} \cmidrule(lr){7-7} \cmidrule(lr){8-8}
        & & \textbf{CS} & \textbf{MMLU} & \textbf{GSM8K}
        & \textbf{Avg} & \textbf{HumanEval} & \textbf{Avg} \\
        \midrule

        \multirow{9}{*}{\rotatebox{90}{\textbf{Gemma-2-2B}}}
        & Pre-trained
        & 57.23$_{\pm{0.00}}$ & 53.14$_{\pm{0.00}}$
        & 23.65$_{\pm{0.00}}$ & 44.67$_{\pm{0.00}}$
        & 20.12$_{\pm{0.00}}$ & 32.40$_{\pm{0.00}}$ \\
        \cmidrule(lr){2-8}
        & Standard FT
        & 55.71$_{\pm{1.17}}$ & 50.64$_{\pm{0.41}}$
        & 20.96$_{\pm{2.04}}$ & 42.44$_{\pm{0.94}}$
        & \textbf{35.49$_{\pm{0.51}}$} & 38.97$_{\pm{0.62}}$ \\
        & WiSE-FT
        & 56.86$_{\pm{1.59}}$ & 49.50$_{\pm{0.39}}$
        & 22.87$_{\pm{2.18}}$ & 43.08$_{\pm{0.86}}$
        & 30.73$_{\pm{0.70}}$ & 36.91$_{\pm{0.34}}$ \\
        & $\ell_2$-Reg
        & 54.72$_{\pm{1.35}}$ & 49.45$_{\pm{0.40}}$
        & 20.64$_{\pm{2.15}}$ & 41.60$_{\pm{1.16}}$
        & 30.49$_{\pm{0.43}}$ & 36.05$_{\pm{0.49}}$ \\
        & LoRA
        & 56.27$_{\pm{1.65}}$ & 51.74$_{\pm{0.40}}$
        & 21.49$_{\pm{2.87}}$ & 43.17$_{\pm{1.26}}$
        & 32.32$_{\pm{0.61}}$ & 37.75$_{\pm{0.69}}$ \\
        & DEFT-UCS
        & 51.46$_{\pm{1.35}}$ & 50.65$_{\pm{0.40}}$
        & 20.65$_{\pm{2.46}}$ & 40.92$_{\pm{1.13}}$
        & 33.05$_{\pm{0.51}}$ & 36.99$_{\pm{0.57}}$ \\
        & FLOW
        & 57.84$_{\pm{1.22}}$ & \textbf{53.05$_{\pm{0.42}}$}
        & 23.91$_{\pm{2.09}}$ & 44.93$_{\pm{0.96}}$
        & 33.66$_{\pm{0.67}}$ & 39.30$_{\pm{0.55}}$ \\
        & TALR
        & \underline{58.02$_{\pm{1.20}}$}
        & \underline{52.92$_{\pm{0.41}}$}
        & \underline{24.37$_{\pm{2.08}}$}
        & \underline{45.10$_{\pm{0.97}}$}
        & 33.54$_{\pm{0.43}}$
        & \underline{39.32$_{\pm{0.35}}$} \\
        \rowcolor{highlightblue} \cellcolor{white}
        & DRIVE
        & \textbf{58.24$_{\pm{1.18}}$} & 52.56$_{\pm{0.40}}$
        & \textbf{24.94$_{\pm{2.11}}$}
        & \textbf{45.25$_{\pm{0.28}}$}
        & \underline{35.24$_{\pm{0.51}}$}
        & \textbf{40.25$_{\pm{0.37}}$} \\
        \midrule

        \multirow{9}{*}{\rotatebox{90}{\textbf{Llama-3.2-3B}}}
        & Pre-trained
        & 54.62$_{\pm{0.00}}$ & 56.54$_{\pm{0.00}}$
        & 27.75$_{\pm{0.00}}$ & 46.30$_{\pm{0.00}}$
        & 28.66$_{\pm{0.00}}$ & 37.48$_{\pm{0.00}}$ \\
        \cmidrule(lr){2-8}
        & Standard FT
        & 52.46$_{\pm{1.17}}$ & 52.18$_{\pm{0.39}}$
        & 25.37$_{\pm{2.19}}$ & 43.34$_{\pm{0.90}}$
        & \textbf{43.41$_{\pm{0.67}}$} & 43.38$_{\pm{0.62}}$ \\
        & WiSE-FT
        & 50.79$_{\pm{1.69}}$ & 52.84$_{\pm{0.40}}$
        & 25.67$_{\pm{2.03}}$ & 43.10$_{\pm{0.86}}$
        & 37.80$_{\pm{0.43}}$ & 40.45$_{\pm{0.42}}$ \\
        & $\ell_2$-Reg
        & 54.58$_{\pm{1.32}}$ & 52.94$_{\pm{0.41}}$
        & \textbf{28.57$_{\pm{2.21}}$} & 45.36$_{\pm{0.45}}$
        & 40.12$_{\pm{0.51}}$ & 42.74$_{\pm{0.26}}$ \\
        & LoRA
        & 52.78$_{\pm{1.18}}$ & 53.79$_{\pm{0.40}}$
        & 24.09$_{\pm{1.33}}$ & 43.55$_{\pm{0.55}}$
        & 40.12$_{\pm{0.51}}$ & 41.84$_{\pm{0.47}}$ \\
        & DEFT-UCS
        & 50.75$_{\pm{1.28}}$ & 48.15$_{\pm{0.42}}$
        & 22.14$_{\pm{1.98}}$ & 40.35$_{\pm{0.81}}$
        & \underline{41.46$_{\pm{0.61}}$} & 40.91$_{\pm{0.64}}$ \\
        & FLOW
        & \underline{56.59$_{\pm{1.18}}$}
        & \underline{55.91$_{\pm{0.39}}$}
        & 25.78$_{\pm{2.19}}$
        & \underline{46.09$_{\pm{0.88}}$}
        & 41.34$_{\pm{0.51}}$ & 43.72$_{\pm{0.33}}$ \\
        & TALR
        & \textbf{56.72$_{\pm{1.17}}$}
        & \textbf{55.96$_{\pm{0.40}}$}
        & 26.22$_{\pm{2.10}}$
        & \textbf{46.30$_{\pm{0.80}}$}
        & 41.22$_{\pm{0.55}}$
        & \underline{43.76$_{\pm{0.52}}$} \\
        \rowcolor{highlightblue} \cellcolor{white}
        & DRIVE
        & 55.79$_{\pm{1.21}}$ & 55.49$_{\pm{0.40}}$
        & \underline{26.66$_{\pm{2.08}}$}
        & 45.98$_{\pm{0.65}}$
        & \textbf{43.41$_{\pm{0.51}}$}
        & \textbf{44.70$_{\pm{0.39}}$} \\
        \midrule

        \multirow{9}{*}{\rotatebox{90}{\textbf{Llama-3.1-8B}}}
        & Pre-trained
        & 61.59$_{\pm{0.00}}$ & 65.32$_{\pm{0.00}}$
        & 50.95$_{\pm{0.00}}$ & 59.29$_{\pm{0.00}}$
        & 34.15$_{\pm{0.00}}$ & 46.72$_{\pm{0.00}}$ \\
        \cmidrule(lr){2-8}
        & Standard FT
        & 52.02$_{\pm{1.25}}$ & 57.28$_{\pm{0.41}}$
        & 40.41$_{\pm{2.05}}$ & 49.90$_{\pm{0.88}}$
        & \textbf{65.49$_{\pm{0.55}}$} & 57.70$_{\pm{0.50}}$ \\
        & WiSE-FT
        & 55.79$_{\pm{1.45}}$ & 60.52$_{\pm{0.40}}$
        & 43.68$_{\pm{1.95}}$ & 53.33$_{\pm{0.73}}$
        & 60.73$_{\pm{0.55}}$ & 57.03$_{\pm{0.49}}$ \\
        & $\ell_2$-Reg
        & 56.69$_{\pm{1.38}}$ & 62.17$_{\pm{0.39}}$
        & 43.79$_{\pm{2.10}}$ & 54.22$_{\pm{0.92}}$
        & 60.61$_{\pm{0.55}}$ & 57.42$_{\pm{0.63}}$ \\
        & LoRA
        & 55.44$_{\pm{1.22}}$ & 59.33$_{\pm{0.41}}$
        & 45.44$_{\pm{1.88}}$ & 53.40$_{\pm{0.58}}$
        & 61.22$_{\pm{0.33}}$ & 57.31$_{\pm{0.39}}$ \\
        & DEFT-UCS
        & 57.58$_{\pm{1.30}}$ & 62.83$_{\pm{0.40}}$
        & 42.17$_{\pm{1.98}}$ & 54.19$_{\pm{0.87}}$
        & 57.20$_{\pm{0.67}}$ & 55.70$_{\pm{0.19}}$ \\
        & FLOW
        & 59.24$_{\pm{1.19}}$ & 63.29$_{\pm{0.39}}$
        & \underline{49.64$_{\pm{2.15}}$}
        & 57.39$_{\pm{0.97}}$
        & 62.80$_{\pm{0.43}}$ & 60.10$_{\pm{0.50}}$ \\
        & TALR
        & \underline{59.58$_{\pm{1.20}}$}
        & \underline{63.34$_{\pm{0.39}}$}
        & \textbf{49.73$_{\pm{2.10}}$}
        & \underline{57.55$_{\pm{0.77}}$}
        & 62.68$_{\pm{0.67}}$
        & \underline{60.12$_{\pm{0.63}}$} \\
        \rowcolor{highlightblue} \cellcolor{white}
        & DRIVE
        & \textbf{60.51$_{\pm{1.32}}$}
        & \textbf{63.46$_{\pm{0.40}}$}
        & 49.51$_{\pm{2.02}}$
        & \textbf{57.83$_{\pm{0.59}}$}
        & \underline{64.27$_{\pm{0.33}}$}
        & \textbf{61.05$_{\pm{0.34}}$} \\
        \bottomrule
    \end{tabular}
    \end{adjustbox}
    \vspace{-8pt}
\end{table*}

\subsection{Impact of Balance Parameter $\lambda$ on Language Tasks}
\label{app:lambda_schedule}

The parameter $\lambda \in [0, 1]$ controls the trade-off between anti-forgetting and task adaptation. Table~\ref{tab:lambda} illustrates this spectrum on Llama-3.2-3B.
\begin{itemize}[leftmargin=*,nosep]
    \item \textbf{Task Adaptation Focus ($\lambda=0.2$):} When $\lambda$ is low, the model prioritizes the target task, achieving a high GSM8K score. However, this comes at the cost of general capabilities, with significant drops in Commonsense (CS) and MMLU scores.
    \item \textbf{Anti-forgetting Focus ($\lambda=0.8$):} Conversely, a high $\lambda$ preserves general capabilities but restricts the model's plasticity, leading to lower performance on the target task.
    \item \textbf{Optimal Balance:} The dynamic schedule used in DRIVE (decaying $\lambda$) or a balanced fixed value effectively navigates this trade-off, achieving the best overall averaged performance.
\end{itemize}

\begin{table*}[h]
\caption{Comparison of different $\lambda$ on Llama-3.2-3B. Higher $\lambda$ favors general capabilities, while lower $\lambda$ benefits target adaptation. $\lambda=0.8$ achieves the best overall balance.}
\label{tab:lambda}
\vspace{-7pt}
\centering
\small
\setlength{\tabcolsep}{8pt}
\begin{tabular}{@{}c|cccc|c|c@{}}
\toprule
\multirow{2}{*}[-0.6ex]{\centering\textbf{$\lambda$ value}} & \multicolumn{4}{c|}{\textbf{General Capability Acc.~(\%)}} & \textbf{Target Acc.~(\%)} & \textbf{Overall~(\%)} \\
\cmidrule(lr){2-5} \cmidrule(lr){6-6} \cmidrule(lr){7-7} 
& \textbf{CS} & \textbf{MMLU} & \textbf{MBPP} & \textbf{Avg} & \textbf{GSM8K} & \textbf{Avg} \\
\midrule
$\lambda=0.2$ & 42.23& 55.55& 39.20& 45.66& \textbf{63.81} & 54.74 \\
$\lambda=0.5$ & 54.95& 55.64& 40.40 & 50.33& 60.82& 55.58\\
$\lambda=0.8$ & \textbf{55.79} & \textbf{55.86} & \textbf{41.64}& \textbf{51.10}& 60.29& \textbf{55.70}\\
\bottomrule
\end{tabular}
\vspace{-18pt}
\end{table*}

\subsection{Temporal Dynamics of $\lambda$: Increasing vs. Decaying Schedule}
\label{app:tau_impact}

To validate the temporal importance of the trade-off parameter, we conduct an ablation study using a linearly increasing schedule for $\lambda$ ($0 \to 1$). This configuration prioritizes task adaptation during the initial phase of training, serving as a counter-factual baseline to our proposed decay strategy ($1 \to 0$). The results across three representative LLMs are summarized in Table~\ref{tab:lambda_increase}.

Compared to the standard DRIVE results (see Table~\ref{tab:main-language}), this reverse schedule yields consistently sub-optimal outcomes. For instance, the Overall Avg score for Gemma-2-2B drops to 54.92\%, while Llama-3.1-8B achieves only 66.75\%. These findings underscore that ``initial anchoring" to the pre-trained distribution is indispensable. In the early stages of fine-tuning, when gradient updates are typically largest, a low $\lambda$ fails to provide sufficient regularization (via the KL-divergence term) against the prior distribution. This allows the model to deviate prematurely from the general knowledge manifold, a distortion that subsequent increases in $\lambda$ cannot fully rectify. Consequently, maintaining a high $\lambda$ at the onset is crucial for resolving the stability-plasticity dilemma.

\begin{table*}[t]
\vspace{12pt}
\caption{Performance on GSM8K and general capabilities using a linearly increasing schedule for $\lambda$. All accuracies are reported in \%.}
\label{tab:lambda_increase}
\vspace{-7pt}
\centering
\small
\setlength{\tabcolsep}{8pt} 
\begin{tabular}{@{}c|cccc|c|c@{}}
\toprule
\multirow{2}{*}[-0.6ex]{\centering\textbf{Model}} & \multicolumn{4}{c|}{\textbf{General Capability Acc.~(\%)}} & \textbf{Target Acc.~(\%)} & \textbf{Overall~(\%)} \\
\cmidrule(lr){2-5} \cmidrule(lr){6-6} \cmidrule(lr){7-7}
& \textbf{CS} & \textbf{MMLU} & \textbf{MBPP} & \textbf{Avg} & \textbf{GSM8K} & \textbf{Avg} \\
\midrule
Gemma-2-2B   & 56.57 & 52.59 & 29.00 & 46.05& 63.79 & 54.92 \\
Llama-3.2-3B & 54.28 & 54.23& 37.00& 48.50& 63.12 & 55.81 \\
Llama-3.1-8B & 58.48 & 62.29 & 45.80 & 55.52 & 77.98 & 66.75 \\
\bottomrule
\end{tabular}
\end{table*}

\subsection{Impact of Prior Regularization $\tau$}
\label{app:regularization_tau}

The hyperparameter $\tau$ controls weight smoothing. In our implementation, $\tau$ is set by the percentile of pre-training loss for scale invariance. Table~\ref{tab:tau} evaluates different percentiles on Llama-3.2-3B. The 10th percentile gives the best overall performance in this LLM setting. A sharper distribution better preserves the contribution-diversity signal while still assigning non-zero weights to other samples. Larger $\tau$ smooths weights toward the prior. With a uniform prior, overly large $\tau$ makes the weights closer to uniform, weakens DRIVE's weighting signal, and can approach Standard FT-like behavior, degrading the overall trade-off.

\begin{table*}[t]
\vspace{-0.1in}
\caption{Comparison of different $\tau$ percentiles on Llama-3.2-3B. All accuracies are reported in \%.}
\label{tab:tau}
\vspace{-7pt}
\centering
\small
\setlength{\tabcolsep}{10pt}
\renewcommand{\arraystretch}{1.2}
\begin{tabular}{@{}c|cccc|c|c@{}}
\toprule
\multirow{2}{*}[-0.6ex]{\centering\textbf{$\tau$-Percentile(\%)}} &
\multicolumn{4}{c|}{\textbf{General Capability Acc.~(\%)}} &
\textbf{Target Acc.~(\%)} &
\textbf{Overall~(\%)} \\
\cmidrule(lr){2-5} \cmidrule(lr){6-6} \cmidrule(lr){7-7}
& \textbf{CS} & \textbf{MMLU} & \textbf{MBPP} & \textbf{Avg}
& \textbf{GSM8K} & \textbf{Avg} \\
\midrule
10 & \textbf{55.92} & 55.71 & \textbf{41.00} & \textbf{50.88} & \textbf{64.50} & \textbf{57.69} \\
30 & 55.19 & 55.80 & 38.80 & 49.93 & 61.90 & 55.92 \\
50 & 54.91 & 54.34 & 38.80 & 49.35 & 60.80 & 55.08 \\
70 & 55.06 & 55.65 & 39.60 & 50.10 & 59.14 & 54.62 \\
90 & 53.13 & \textbf{55.90} & 40.00 & 49.68 & 58.01 & 53.85 \\
\bottomrule
\end{tabular}
\vspace{-14pt}
\end{table*}

\subsection{Impact of Diversity Regularization $\eta$ on Language Tasks}
\label{app:eta_impact}
The diversity coefficient $\eta$ controls the strength of the covariance determinant term. Table~\ref{tab:eta} studies the sensitivity of DRIVE to this coefficient on language tasks.
\begin{itemize}[leftmargin=*,nosep]
    \item \textbf{Effect of Diversity Strength:} General capability metrics vary non-monotonically as $\eta$ increases. Moderate diversity regularization improves the general capability average, peaking at $\eta=0.75$, while overly large values reduce CS, MMLU, and MBPP. This suggests diversity regularization is beneficial when balanced, but excessive emphasis can hurt capability preservation.
    \item \textbf{Robustness:} DRIVE remains relatively robust across a range of $\eta$ values, with the best overall performance observed at $\eta=1.00$. This setting achieves the highest GSM8K accuracy and overall average, indicating that a moderate-to-strong diversity coefficient provides the best balance between target adaptation and general capability retention.
\end{itemize}

\begin{table*}[h]
\vspace{0pt}
\caption{Comparison of different $\eta$ on Llama-3.2-3B. }
\label{tab:eta}
\vspace{-7pt}
\centering
\small 
\setlength{\tabcolsep}{10pt}
\renewcommand{\arraystretch}{1.2} 
\begin{tabular}{@{}c|cccc|c|c@{}}
\toprule
\multirow{2}{*}[-0.6ex]{\centering\textbf{$\eta$ value}} &\multicolumn{4}{c|}{\textbf{General Capability Acc.~(\%)}} & \textbf{Target Acc.~(\%)} &  \textbf{Overall~(\%)}  \\
\cmidrule(lr){2-5} \cmidrule(lr){6-6}   \cmidrule(lr){7-7}  
 & \textbf{CS} & \textbf{MMLU} & \textbf{MBPP} &\textbf{Avg}  &\textbf{GSM8K} &\textbf{Avg} \\
\midrule
\textbf{0.25} & 55.08 & 56.30 & 40.88 & 50.75 & 59.67 & 55.21 \\
\textbf{0.50} & 54.56 & 56.67 & \textbf{41.52} & 50.92 & 58.65 & 54.79 \\
\textbf{0.75} & 55.67 & \textbf{56.79} & 41.00 & \textbf{51.15} & 60.61 & 55.88 \\
\textbf{1.00} & \textbf{55.92} & 55.71 & 41.00 & 50.88 & \textbf{64.50} & \textbf{57.69} \\
\textbf{1.25} & 55.19 & 55.80 & 38.80 & 49.93 & 61.90 & 55.92 \\
\textbf{1.5}  & 54.26 & 54.21 & 37.00 & 48.49 & 61.50 & 55.00 \\
\bottomrule
\end{tabular}
\vspace{-14pt}
\end{table*}

\subsection{Ablation Study on Vision Tasks}
\label{app:ablation_vision}

To verify the generalizability of our findings beyond LLMs, we conducted an ablation study on the ResNet-50 vision model (Table~\ref{tab:ablation_vision}). The results indicate that removing any individual component of DRIVE compromises the overall performance, with the full method consistently achieving the highest average accuracy across both benchmarks (72.34\% on CIFAR-100 and 83.85\% on Flowers102). Specifically, the diversity regularization ($\eta=0$) proves critical for adaptation; its removal leads to a substantial decline in target accuracy, such as a 7.27\% drop on CIFAR-100 and a 3.74\% decrease on Flowers102. This confirms that expanding the feature volume is necessary to capture task-specific nuances effectively. Additionally, the uniform-prior KL regularization remains the foundation of the framework, as its removal causes the most significant performance degradation, reducing the average accuracy on CIFAR-100 by approximately 9.5\% due to weight collapse.

Furthermore, the method navigates the stability-plasticity dilemma through the balance parameter. While the setting without anti-forgetting ($\lambda=0$) yields the highest target accuracy (86.40\% on CIFAR-100 and 93.67\% on Flowers102), it incurs a retention cost, reducing ImageNet performance to 57.49\% and 71.98\%, respectively. Conversely, DRIVE maintains equilibrium. It secures competitive adaptation performance (83.03\% on CIFAR-100 and 92.67\% on Flowers102) while retaining more knowledge than the plasticity-only baseline. This demonstrates that DRIVE optimizes the joint trade-off, ensuring reliable adaptation without erasing pre-trained representations.

\begin{table}[ht]
    \vspace{-0.15in}
    \caption{Ablation study of DRIVE components on ResNet-50 (Accuracy in \%).}
    \label{tab:ablation_vision}
    \centering
    \small
    \renewcommand{\arraystretch}{1.1}
    \setlength{\tabcolsep}{5pt} 
    \begin{tabular}{c|cccccc} 
        \toprule
        \multirow{2.5}{*}{\textbf{Method}} & \multicolumn{3}{c}{CIFAR-100} & \multicolumn{3}{c}{Flowers102} \\
        \cmidrule(lr){2-4} \cmidrule(lr){5-7}
        & IN-1K & Target Acc. & Avg. & IN-1K & Target Acc. & Avg. \\
        \midrule
        DRIVE w/o Diversity ($\eta=0$) & 
        \textbf{65.79} & 75.76 & 70.78 & 
        \textbf{77.93} & 88.93 & 83.43 \\
        
        DRIVE w/o KL Prior  & 
        55.93 & 69.71 & 62.82 & 
        66.94 & 87.52 & 77.23 \\
        
        DRIVE w/o anti-forgetting ($\lambda=0$) & 
        57.49 & \textbf{86.40} & 71.95 & 
        71.98 & \textbf{93.67} & 82.83 \\
        
        DRIVE w/o task-adaptation ($\lambda=1$) & 
        63.49 & 80.91 & 72.20 & 
        76.26 & 90.73 & 83.50\\
        
        \textbf{DRIVE} & 
        61.64 & 83.03 & \textbf{72.34} & 
        75.02 & 92.67 & \textbf{83.85} \\
        \bottomrule
    \end{tabular}
    \vspace{-14pt}
\end{table}

\subsection{Computational Efficiency Analysis}
\label{appendix:efficiency}

We first evaluate the computational overhead on vision tasks using CIFAR-100 with ResNet-18 and ResNet-50 on a single NVIDIA A100 GPU. As shown in Table~\ref{tab:vision_time_efficiency}, we measure the average GPU time per epoch. Compared to Standard FT, FLOW introduces slight latency (3.3\%--4.6\%). DRIVE incurs a moderate increase (11.8\%--14.7\%) due to projection and determinant calculations required for diversity regularization, yet remains efficient for standard backbones.

\begin{table}[h]
\centering
\small
\vspace{-0.15in}
\caption{Training efficiency on CIFAR-100, reported as GPU seconds per epoch.}
\label{tab:vision_time_efficiency}
\vspace{-0pt}
\begin{tabular}{llcc}
\toprule
Model & Method & GPU Time / Epoch (s) & Overhead (\%) \\ \midrule
\multirow{3}{*}{ResNet-18} & Standard FT & 280 & - \\
 & FLOW & 293 & +4.64\% \\
 & DRIVE (Ours) & 313 & +11.79\% \\ \midrule
\multirow{3}{*}{ResNet-50} & Standard FT & 389 & - \\
 & FLOW & 402 & +3.34\% \\
 & DRIVE (Ours) & 446 & +14.65\% \\ \bottomrule
\end{tabular}
\vspace{-10pt}
\end{table}

We further evaluate scalability on MetaMathQA using Gemma-2B (single A100) and Llama-3-8B (dual A100s). Table~\ref{tab:llm_time_efficiency} reports the average GPU time per epoch. Both methods incur manageable costs: FLOW adds approximately 10\% overhead for computing loss distributions, while DRIVE adds roughly 17\% due to simultaneous forward passes and covariance computation. Notably, the Johnson-Lindenstrauss transform ($d=4096 \to k=128$) ensures that high-dimensional geometric calculations remain linear, preventing prohibitive scaling on larger models.

\begin{table}[h]
\centering
\small
\vspace{-8pt}
\caption{Training efficiency on MetaMathQA, reported as GPU hours per epoch.}
\label{tab:llm_time_efficiency}
\begin{tabular}{llcc}
\toprule
Model & Method & GPU Time / Epoch (h) & Overhead (\%) \\ \midrule
\multirow{3}{*}{Gemma-2-2B} & Standard FT & 5.45 & - \\
 & FLOW & 6.00 & +10.09\% \\
 & DRIVE (Ours) & 6.38 & +17.06\% \\ \midrule
\multirow{3}{*}{Llama-3.1-8B} & Standard FT & 19.42 & - \\
 & FLOW & 21.36 & +9.99\% \\
 & DRIVE (Ours) & 22.75 & +17.15\% \\ \bottomrule
\end{tabular}
\vspace{-14pt}
\end{table}

\subsection{Paired Statistical Analysis of the Main LLM Results}
\label{app:paired_significance}

We assess the statistical reliability of the Overall improvements in Table~\ref{tab:main-language} using original training-log scores from five runs with matched random seeds and data orders. DRIVE is paired with the strongest baseline for each model: TALR for Gemma-2-2B and Llama-3.2-3B, and LoRA for Llama-3.1-8B. No additional training or reruns are performed.

For each matched seed, we compute DRIVE's Overall improvement and apply a two-sided paired $t$-test against zero mean improvement. Displayed Overall scores are calculated from rounded component means, and $\Delta$ is their difference. The 95\% confidence intervals, raw $p$-values, and small-sample-corrected paired effect sizes Hedges' $g_z$ are computed from the original unrounded paired scores. Table~\ref{tab:paired_significance} reports these statistics and Holm-adjusted $p$-values across the three model-wise tests; confidence intervals are reported for individual comparisons without multiplicity adjustment.

\begin{table}[t]
\centering
\caption{Paired statistical analysis of DRIVE against the strongest baseline for each LLM, using five runs with matched random seeds and data orders. Overall scores are percentages; differences and confidence intervals are expressed in percentage points. Positive $\Delta$ favors DRIVE.}
\label{tab:paired_significance}
\small
\setlength{\tabcolsep}{4pt}
\renewcommand{\arraystretch}{1.15}
\resizebox{\linewidth}{!}{%
\begin{tabular}{@{}llrrrrrrr@{}}
\toprule
\textbf{Model}
& \textbf{Baseline}
& \shortstack{\textbf{Baseline}\\\textbf{Overall}}
& \shortstack{\textbf{DRIVE}\\\textbf{Overall}}
& \textbf{$\Delta$}
& \textbf{95\% CI}
& \textbf{Raw $p$}
& \shortstack{\textbf{Holm-adjusted}\\\textbf{$p$}}
& \textbf{Hedges' $g_z$} \\
\midrule
Gemma-2-2B & TALR & 53.87 & 55.55 & +1.68 & $[0.85,\,2.51]$ & 0.0049 & 0.0148 & 2.00 \\
Llama-3.2-3B & TALR & 56.26 & 57.69 & +1.43 & $[0.46,\,2.42]$ & 0.0152 & 0.0152 & 1.45 \\
Llama-3.1-8B & LoRA & 65.70 & 67.27 & +1.57 & $[0.77,\,2.37]$ & 0.0055 & 0.0148 & 1.95 \\
\bottomrule
\end{tabular}%
}
\end{table}

All three paired comparisons favor DRIVE, with individual 95\% confidence intervals excluding zero. All improvements remain significant after Holm correction (adjusted $p \leq 0.0152$). Hedges' $g_z$ ranges from 1.45 to 2.00, indicating large improvements relative to the variability of paired differences. These results support the statistical reliability of DRIVE's reported 1.43--1.68 percentage-point Overall gains over the strongest evaluated baselines under this five-seed protocol.

\section{Analysis of Contribution Assessment and Diversity Mechanisms}
\label{app:mechanism_analysis}

We analyze DRIVE's contribution assessment, diversity modeling, and covariance approximation. Sec.~\ref{sec:appendix_shapley} studies sample-level stability--plasticity conflict via Shapley-style analysis. Sec.~\ref{app:diversity_metrics} discusses diversity metrics and DRIVE-Cluster. Sec.~\ref{sec:appendix_ensemble} compares covariance approximations. Sec.~\ref{sec:appendix_hard_sample_dynamics} examines hard-sample adaptation dynamics; Sec.~\ref{app:k_sensitivity} reports sensitivity to projection dimension \(k\).

\subsection{Shapley Analysis of Anti-Forgetting vs. Task Adaptation}
\label{sec:appendix_shapley}

To validate the \textit{Stability-Plasticity Dilemma} at the sample level and justify our contribution-based weighting, we conducted an analysis using the Shapley value framework. We utilize the Gradient Inner Product (GIP) \cite{koh2017understanding, wang2025data} as an efficient approximation of the One-Run Shapley value to quantify the contribution of each training sample toward two goals: (1) preserving pre-trained knowledge (Anti-forgetting, $a_i$) and (2) learning the new task (Task Adaptation, $t_i$).

\subsubsection{Experimental Setup}
\label{sec:exp_shapley_setup}

\textbf{Full-Batch Validation Gradient Calculation.}
We use full validation sets at every training step to ensure accurate estimation of the optimization direction. Let $\mathcal{D}_{\mathrm{val}}^{\mathrm{target}}$ and $\mathcal{D}_{\mathrm{val}}^{\mathrm{retain}}$ denote the complete validation sets for the target task (e.g., Flowers102) and the retention task (e.g., ImageNet subset), respectively.
At each training step $t$, we compute the exact average gradient over these datasets:
\begin{align}
\mathbf{g}_t^{\mathrm{target}}
&:= \nabla_{\theta}\left[\frac{1}{|\mathcal{D}_{\mathrm{val}}^{\mathrm{target}}|}\sum_{(x,y)\in \mathcal{D}_{\mathrm{val}}^{\mathrm{target}}}
\ell\bigl(f(x;\theta_t),y\bigr)\right],
\label{eq:g_target_full}\\
\mathbf{g}_t^{\mathrm{retain}}
&:= \nabla_{\theta}\left[\frac{1}{|\mathcal{D}_{\mathrm{val}}^{\mathrm{retain}}|}\sum_{(x,y)\in \mathcal{D}_{\mathrm{val}}^{\mathrm{retain}}}
\ell_{\text{retain}}\bigl(f(x;\theta_t),y\bigr)\right].
\label{eq:g_retain_full}
\end{align}

\paragraph{Cumulative Sample Contribution.}
Following \cite{wang2025data}, we quantify sample utility by summing marginal contributions along the optimization trajectory. Let $\mathcal{T}_i$ be steps where sample $x_i$ appears in the mini-batch. We define Shapley-style anti-forgetting ($a_i^{\mathrm{shap}}$) and task-adaptation ($t_i^{\mathrm{shap}}$) contributions as cumulative alignment between the sample gradient and full-batch validation gradients:
\begin{equation}
a_i^{\mathrm{shap}} := \sum_{t \in \mathcal{T}_i} \underbrace{\eta_t \left\langle \nabla_{\theta}\ell(f(x_i;\theta_t), y_i),\; \mathbf{g}_t^{\mathrm{retain}} \right\rangle}_{\text{Step-wise Contribution to Retention}},
\qquad
t_i^{\mathrm{shap}} := \sum_{t \in \mathcal{T}_i} \underbrace{\eta_t \left\langle \nabla_{\theta}\ell(f(x_i;\theta_t), y_i),\; \mathbf{g}_t^{\mathrm{target}} \right\rangle}_{\text{Step-wise Contribution to Adaptation}}.
\label{eq:ai_cumulative}
\end{equation}
Here, $\eta_t$ denotes the learning rate \cite{wang2025data}. The inner product measures geometric alignment, where a positive value indicates the sample facilitates loss minimization on the corresponding validation set.

\subsubsection{Empirical Results and Analysis}
\label{sec:exp_shapley_results}

\textbf{Validation of Loss-based Proxies.} In the main text, we proposed loss-based metrics to define the contribution scores (e.g., $a_i \propto -\ell_\text{pre}$, $t_i \propto \ell_{\text{cur}}$). To verify these heuristic proxies, we analyzed their correlation with theoretically grounded Shapley values calculated via Eq.~\eqref{eq:ai_cumulative}. Our empirical analysis reveals a significant positive correlation. The loss-based anti-forgetting metric $a_i$ strongly correlates with the reference $a_i^{\text{shap}}$ (Pearson's $r=0.6493$), corroborating that samples with lower pre-training loss generally contribute disproportionately to parameter stability. The strong positive correlation between $t_i$ and $t_i^{\text{shap}}$ ($r=0.8679$) confirms high-loss samples drive adaptation. This validates using our efficient loss-based proxy in DRIVE for scalable training.

\textbf{Disentangling Loss from Forgetting.}
While correlated ($r \approx 0.65$), loss magnitude alone does not fully explain forgetting; we posit gradient conflict as the primary driver. To verify this, we isolated the top 40\% high-loss samples from Flowers102 and partitioned them by Gradient Inner Product (GIP) with the retention task. This yielded a Conflicting Group ($GIP < 0$, 293 images) that opposes retention, and a Synergistic Group ($GIP > 0$, 117 images) geometrically aligned with it.

\begin{wraptable}{r}{0.38\textwidth} 
\vspace{-0.1in} 
\centering
\small
\caption{Impact of Gradient Alignment. Synergistic samples prevent forgetting despite high loss.}
\label{tab:grad_conflict}
\vspace{-7pt}
\begin{tabular}{lcc}
\toprule
Subset & Target (\%) & Retain (\%) \\
\midrule
Pre-trained & - & 80.10 \\
Conflicting & 84.85& 38.42 \\
Synergistic & 83.76 & 79.15\\
\bottomrule
\end{tabular}
\vspace{-18pt} 
\end{wraptable}
As summarized in Table~\ref{tab:grad_conflict}, we evaluate performance by training on distributions combining the remaining 60\% low-loss samples with either the Synergistic or Conflicting high-loss subgroups. Fine-tuning on the Synergistic mixture achieves competitive target accuracy while maintaining ImageNet performance comparable to the pre-trained baseline.

In contrast, the Conflicting mixture triggers severe forgetting. This dichotomy confirms that high-loss samples can be learned with limited forgetting when their gradients align with the retention task.

\paragraph{Conflict Analysis via Min-Max Normalization.}
To intuitively visualize the trade-off landscape independent of raw gradient magnitudes, we project the accumulated Shapley values into the unit interval $[0, 1]$ using Min-Max normalization. Let $\mathbf{a}^{\text{shap}}$ and $\mathbf{t}^{\text{shap}}$ be the vectors containing the contributions of all samples. The normalized scores are defined as:
\begin{equation}
\widetilde a_i := \frac{a_i^{\text{shap}} - \min(\mathbf{a}^{\text{shap}})}{\max(\mathbf{a}^{\text{shap}}) - \min(\mathbf{a}^{\text{shap}})}, \qquad
\widetilde t_i := \frac{t_i^{\text{shap}} - \min(\mathbf{t}^{\text{shap}})}{\max(\mathbf{t}^{\text{shap}}) - \min(\mathbf{t}^{\text{shap}})}.
\end{equation}

\begin{wrapfigure}{r}{0.4\textwidth}
    \vspace{-0.20in}
    \centering
    \includegraphics[width=0.4\textwidth]{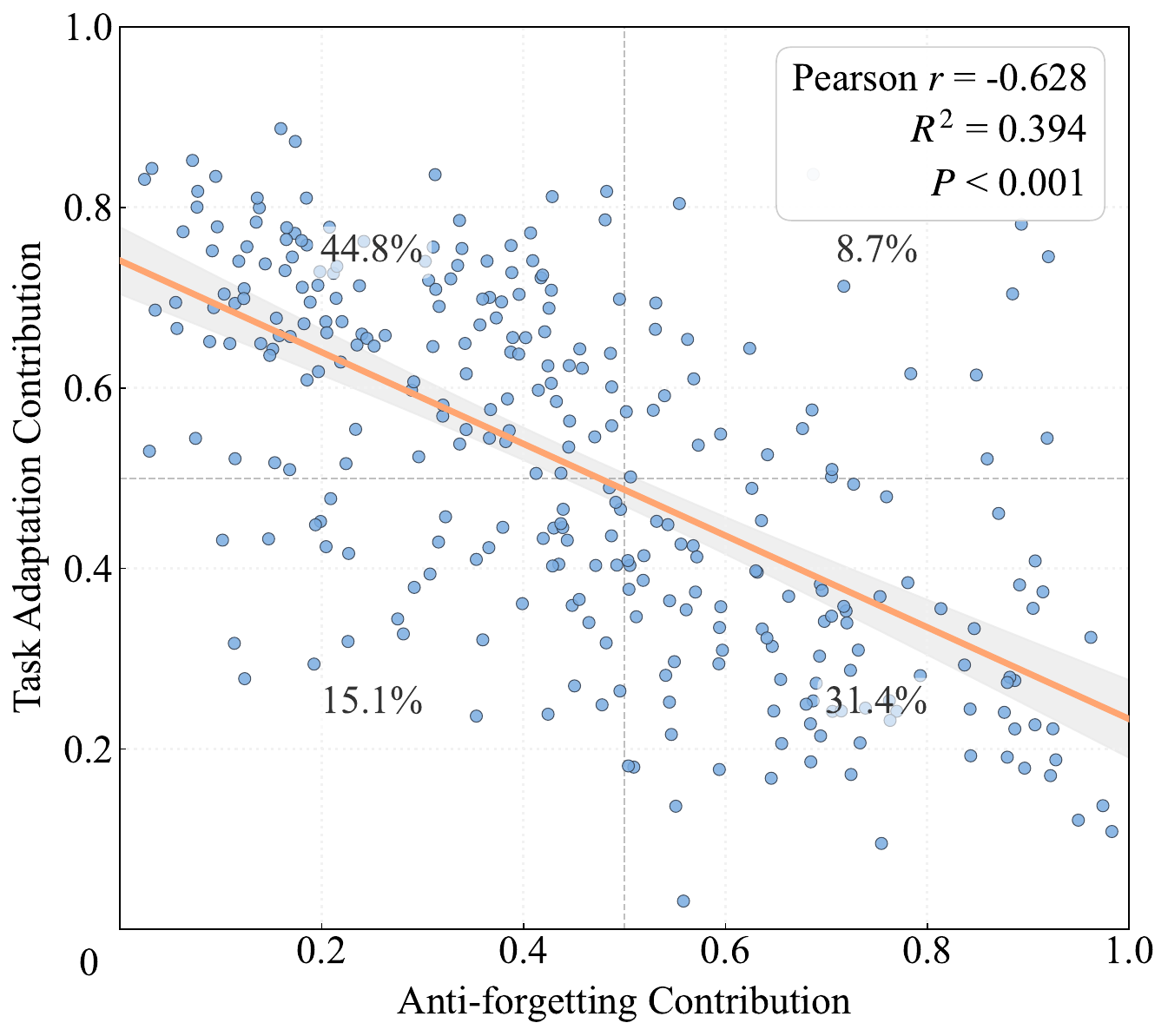}
    \vspace{-0.2in}
    \caption{Scatter plot of normalized Anti-forgetting ($\widetilde a_i$) vs. Task Adaptation ($\widetilde t_i$) contributions. }
    \label{fig:shapley_scatter}
    \vspace{-0.2in} 
\end{wrapfigure}
Fig.~\ref{fig:shapley_scatter} plots the normalized task-adaptation contribution $\widetilde{t}_i$ versus the anti-forgetting contribution $\widetilde{a}_i$. The regression analysis reveals a pronounced negative correlation (Pearson's $r=-0.628$, $R^2=0.394$), supporting a strong empirical conflict between the two objectives: samples conducive to adaptation often tend to be less favorable for retention. The data distribution is heavily skewed toward the off-diagonal quadrants: 44.8\% of samples cluster in the Adaptation-Biased region (top-left) and 31.4\% in the Retention-Biased region (bottom-right). Conversely, the Co-contributing region (top-right) remains sparse, containing only 8.7\% of the data. This scarcity of synergistic samples underscores the stability-plasticity dilemma and motivates DRIVE's dynamic trade-off mechanism, which navigates the conflicting landscape more effectively than naive uniform averaging.

\subsection{Sample Diversity}
\label{app:diversity_metrics}

Sample diversity is a multi-faceted construct that ensures the model explores the parameter space effectively while avoiding overfitting to redundant patterns \cite{song2024scaling, stap2025effect}. In this section, we expand on the metrics used to quantify diversity in recent literature. 

\subsubsection{Diversity in Fine-tuning}

\paragraph{1. Semantic and Clustering-based Metrics.}
These metrics rely on the distance between samples in an embedding space to ensure coverage \cite{das2024deft, li2025unlocking}.
    \cite{li2025unlocking} identify representative core sets via clustering-based strategies to improve sample efficiency \cite{li2025unlocking}. While effective, these methods are often discrete and non-differentiable.
DEFT-UCS utilizes core-set selection for text-editing tasks to reduce data size while preserving performance \cite{das2024deft}. However, it lacks the continuous, joint optimization framework provided by DRIVE.

\paragraph{2. Gradient-space Diversity.}
Diversity in the gradient space is often a more direct proxy for optimization robustness \cite{aljundi2019gradient, xia2024less}. \cite{xia2024less} demonstrate that selecting data based on gradient alignment can target specific downstream capabilities. DRIVE captures a similar signal through the task-adaptation term $t_i$, which prioritizes samples essential for mastering the new task. GSS prioritizes samples whose gradients maximally align with the feasible regions of previous tasks to mitigate forgetting \cite{aljundi2019gradient}.

\paragraph{3. Multi-level Diversity (Macro to Micro).}
Recent trends categorize diversity across hierarchies of the instruction dataset \cite{li2025macro}.
Studies show that increasing response diversity (micro-diversity) is more beneficial for human alignment than simply expanding prompt quantity (macro-diversity) \cite{song2024scaling}.
Methods like DEITA prioritize semantic complexity and quality over volume, effectively reducing the sample size needed for fine-tuning \cite{liu2024what, zhang2025survey}.

While existing methods often rely on selection-based or discrete metrics, DRIVE integrates diversity directly into the objective function via a differentiable covariance determinant, providing a unified and scalable solution.

\subsubsection{Details of DRIVE-Cluster}
\label{app:DRIVE_cluster}

In the replacement study, we compare DRIVE with DRIVE-Cluster, a heuristic baseline that enforces diversity through rigid geometric constraints. This variant corresponds to the DRIVE framework with $\eta=0$, where the diversity requirement is encoded directly into the prior distribution rather than a dynamic regularizer.

\paragraph{Formulation.}
Given $K$ non-empty clusters $\{\mathcal{C}_k\}_{k=1}^{K}$ obtained via K-Means clustering, we construct a cluster-balanced prior $\bm{p}=\{p_i\}_{i=1}^N$ over the entire dataset. To ensure each semantic cluster holds equal probability mass ($1/K$) initially, the prior is defined as:
\begin{equation}
    p_i = \frac{1}{K|\mathcal{C}_k|}, \quad \forall i \in \mathcal{C}_k.
\end{equation}
The weight optimization is then performed on the entire dataset by minimizing the KL-divergence against this global prior, subject to a single global normalization constraint:
\begin{equation}
\begin{aligned}
\max_{\bm w} \quad & \sum_{i=1}^N w_i c_i - \tau KL(\bm{w}||\bm{p}) \quad \text{s.t.}  & \sum_{i=1}^N w_i = 1, \quad w_i \ge 0, \quad \forall i.
\end{aligned}
\end{equation}
This formulation unifies all samples into a single optimization landscape, allowing them to compete globally based on their contribution scores, modulated by the cluster-balanced prior.

\paragraph{Closed-form Solution.}
By introducing a Lagrange multiplier for the global constraint $\sum w_i = 1$, we derive the optimal solution. The stationarity condition yields:
\begin{equation}
w_i^* \propto p_i \exp(c_i / \tau).
\end{equation}
Enforcing the global normalization condition, we obtain the standard Softmax form:
\begin{equation}
w_i^* = \frac{p_i \exp(c_i / \tau)}{\sum_{j=1}^N p_j \exp(c_j / \tau)}.
\end{equation}
This result highlights that DRIVE-Cluster effectively weights samples by scaling their exponential contribution with the prior $p_i$. Since $p_i \propto 1/|\mathcal{C}_k|$, samples from dense (redundant) clusters are naturally down-weighted in the global competition, ensuring that the final weight distribution maintains semantic diversity without requiring explicit cluster-level hard constraints.


\subsection{Analysis of Covariance Approximation Strategies}
\label{sec:appendix_ensemble}

In Section~\ref{sec:efficient_reweighting_finetuning}, we employ randomized dimensionality reduction via the Johnson-Lindenstrauss (JL) transform to approximate the high-dimensional feature space $\mathbb{R}^d$ in a lower-dimensional subspace $\mathbb{R}^k$. This design choice balances computational efficiency with geometric fidelity. In this section, we justify this choice by comparing it against two alternative strategies: a computationally cheaper diagonal approximation and a theoretically more robust ensemble projection method.

\subsubsection{Comparison with Diagonal Approximation}
A computationally cheaper alternative to our randomized projection strategy is the diagonal approximation, which assumes feature independence by setting off-diagonal elements of the weighted covariance matrix $\bm{\Sigma}_w$ to zero. This approach is widely adopted in methods like EWC \cite{kirkpatrick2017overcoming} and scalable Laplace approximations \cite{ritter2018scalable} to reduce the storage and inversion cost from $\mathcal{O}(d^3)$ to $\mathcal{O}(d)$:
\begin{equation}
    \log |\text{diag}(\bm{\sigma}^2)| = \sum_{j=1}^d \log(\sigma_j^2).
\end{equation}
However, this approximation creates a critical failure mode for diversity weighting: it cannot distinguish between diverse orthogonal features and redundant collinear features, provided their marginal variances are identical. Since DRIVE aims to explicitly penalize geometric redundancy (i.e., feature correlation), preserving off-diagonal interactions via random projection is theoretically necessary to enforce D-Optimality. Consequently, despite its linear complexity, the diagonal approximation is unsuitable for our objective.

\subsubsection{Empirical Analysis of Single vs. Ensemble Random Projections}
While a single random projection preserves pairwise distances with high probability, ensembling multiple independent projections could theoretically reduce the approximation variance. We investigate whether this potential gain justifies the added computational cost.

\paragraph{Experimental Setup.}
To evaluate the impact of ensembling, we compare the performance of DRIVE using a single Gaussian random projection matrix $\mathbf{R} \in \mathbb{R}^{d \times k}$ against an ensemble of $M$ independent matrices $\{\mathbf{R}_m\}_{m=1}^M$. For the ensemble method, the diversity regularization term averages the log-determinants across $M$ projections:
\begin{equation}
    \mathcal{J}(w) = \frac{1}{M} \sum_{m=1}^M \eta \log \left| \sum_{i=1}^N w_i (\mathbf{R}_m^\top \bar{\bm{x}}_i)(\mathbf{R}_m^\top \bar{\bm{x}}_i)^\top \right|
\end{equation}
We conducted experiments on ResNet-50 fine-tuned on CIFAR-100 and Flowers102, setting $k=128$ and varying $M \in \{1, 5, 10\}$. Other hyperparameters remain consistent with the main experiments.
\begin{table}[h]
\centering
\small
\caption{Performance of DRIVE with varying ensemble sizes on CIFAR-100 and Flowers102.}
\label{tab:ensemble_ablation}
\begin{tabular}{l|ccc|ccc}
\toprule
\multirow{2}{*}{\textbf{Method}} & \multicolumn{3}{c|}{\textbf{CIFAR-100}} & \multicolumn{3}{c}{\textbf{Flowers102}}  \\
& \textbf{Target} & \textbf{IN-1K} & \textbf{Avg.} & \textbf{Target} & \textbf{IN-1K} & \textbf{Avg.}\\
\midrule
Single Proj. ($M=1$) & 83.03 & 61.64 & 72.34 & 92.67 & 75.02 & 83.85 \\
\midrule
Ensemble ($M=5$) & 83.08 & \textbf{61.67} & \textbf{72.38} & 92.65 & \textbf{75.09} & 83.87  \\
Ensemble ($M=10$) & \textbf{83.10} & 61.65 & \textbf{72.38} & \textbf{92.71} & 75.05 & \textbf{83.88}  \\
\bottomrule
\end{tabular}
\end{table}

As observed in Table~\ref{tab:ensemble_ablation}, increasing the ensemble size from $M=1$ to $M=10$ yields negligible performance gains. Specifically, on CIFAR-100, the average accuracy improves by only 0.04\%, and on Flowers102, the improvement is marginal with a gain of just 0.03\%. This performance plateau suggests that a single Gaussian projection ($M=1$) is sufficiently robust for capturing the geometric structure required by our diversity regularizer. The variance introduced by a single projection does not destabilize the optimization trajectory. Therefore, given the lack of significant benefits from ensembling, we use a single random projection in our main experiments to maintain model simplicity.

\subsection{Empirical Verification of Dynamic Adaptation on Hard Samples}
\label{sec:appendix_hard_sample_dynamics}

To empirically validate FLOW's under-fitting behavior on challenging samples, we conducted an experiment monitoring the optimization dynamics on the most difficult subset of the data.

\begin{wrapfigure}{r}{0.5\textwidth}
    \vspace{-0.3in} 
    \centering
    \includegraphics[width=0.5\textwidth]{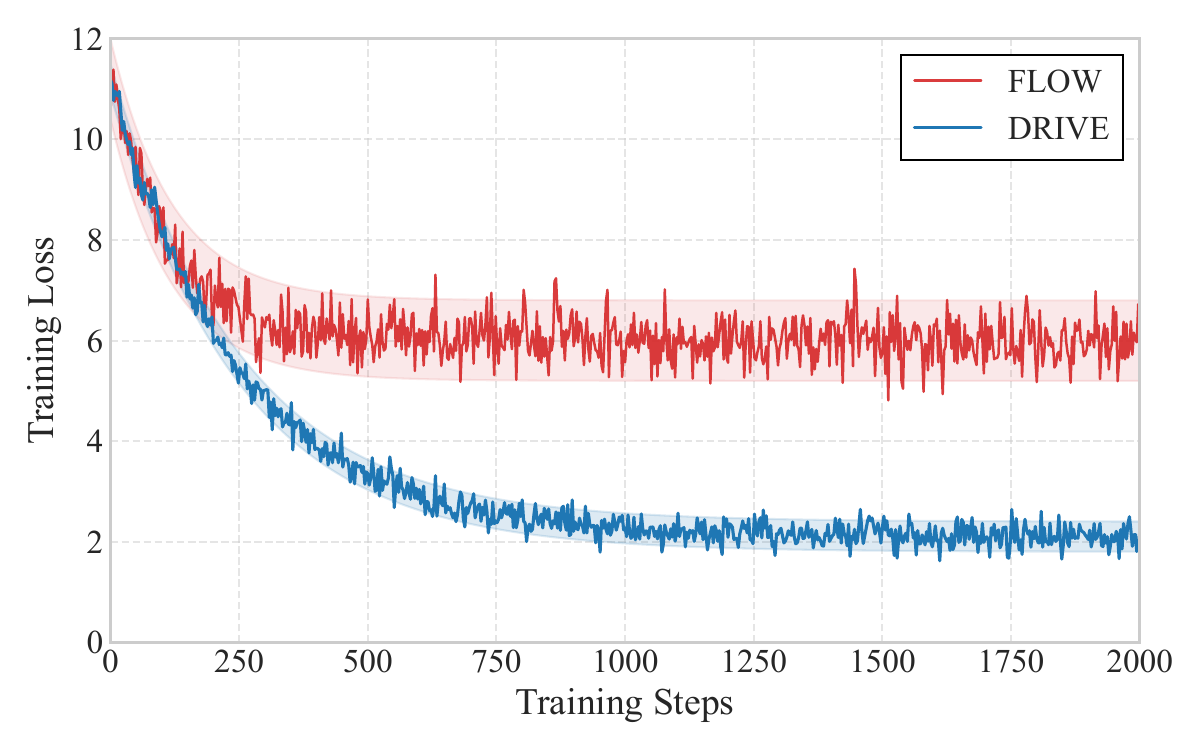}
    \vspace{-0.25in}
    \caption{Loss Dynamics on $\mathcal{D}_{hard}$. FLOW exhibits high-variance stalling (jagged), while DRIVE achieves stable adaptation (smooth).}
    \label{fig:hard_sample_loss}
    \vspace{-0.2in}
\end{wrapfigure}

\textbf{Setup.} We fine-tune a ResNet-50 backbone on CIFAR-100. Before fine-tuning, we select the top 10\% training samples (approximately 5{,}000 images) that incur the largest loss under the pre-trained model $\bm{\theta}_{\mathrm{pre}}$, and define this subset as a fixed hard-sample probe set $\mathcal{D}_{\mathrm{hard}}$ (not used for model selection or hyperparameter tuning). During training, for both FLOW and DRIVE, we monitor the raw (unweighted) mean loss on $\mathcal{D}_{\mathrm{hard}}$ over time.

As illustrated in Figure~\ref{fig:hard_sample_loss}, the distinct trajectories confirm our theoretical analysis. The loss curve for FLOW remains nearly flat, stagnating at a high value ($\approx$ 6.0). Notably, the curve exhibits severe fluctuations and instability. This indicates that while the model attempts to update, the ill-conditioned landscape causes the optimization to oscillate without making effective progress. By assigning negligible weights ($w_i \approx 0$) to high-loss samples, FLOW effectively freezes the model parameters in the ``conflicting subspace,'' preventing the gradient updates required to master these samples.

In contrast, DRIVE exhibits a steady, smooth, and significant decrease, with the final loss dropping to approximately 2.0. This substantial reduction serves as direct evidence that the model has successfully ``learned'' the underlying patterns of these hard samples.

To provide a mechanistic explanation for the optimization stalling observed in FLOW, we monitored the condition number $\kappa(\Sigma_w) = \lambda_{\max}/\lambda_{\min}$ of the weighted covariance matrix.
Our empirical observations reveal a stark contrast in the optimization geometry: FLOW exhibits spectral collapse, resulting in an ill-conditioned landscape with $\kappa(\Sigma_w)$ reaching the order of $10^6$. This anisotropy traps the optimization in a steep valley, causing the gradients to bounce ineffectively between valley walls (zig-zagging behavior).
In contrast, thanks to the determinant-based diversity regularization, DRIVE maintains a better-conditioned geometry, stabilizing $\kappa(\Sigma_w)$ at the order of $10^3$.
This reduction by three orders of magnitude ($10^6 \to 10^3$) empirically validates  that DRIVE enforces D-Optimality, thereby flattening the loss landscape and ensuring robust convergence.

\begin{wraptable}{r}{0.32\textwidth}
\vspace{-0.15in}
\centering
\scriptsize
\caption{Sensitivity to $k$ on Gemma-2-2B / MetaMathQA. Relative time is normalized by $k=32$.}
\label{tab:k_sensitivity}
\vspace{-7pt}
\setlength{\tabcolsep}{3pt}
\renewcommand{\arraystretch}{1.03}
\begin{tabular}{ccccc}
\toprule
$k$ & Gen. & Target & Overall & Time \\
\midrule
32  & 46.72 & 54.18 & 50.45 & 1.00 \\
64  & 47.01 & 57.60 & 52.31 & 1.30 \\
128 & 47.64 & 63.46 & 55.55 & 1.70 \\
256 & 47.69 & 63.79 & 55.74 & 2.50 \\
512 & 47.71 & 64.32 & 56.02 & 4.10 \\
\bottomrule
\end{tabular}
\vspace{-0.15in}
\end{wraptable}
\subsection{Sensitivity to the Projection Dimension $k$}
\label{app:k_sensitivity}

We conduct a lightweight sensitivity study on Gemma-2-2B fine-tuned on MetaMathQA, varying the sketch dimension $k$ while keeping all other settings fixed. As shown in Table~\ref{tab:k_sensitivity}, overly small sketches ($k=32,64$) lead to clear drops in both target accuracy and overall performance, indicating that the weighted local covariance geometry is compressed too aggressively. In contrast, increasing the sketch size to $k=128$ yields a substantial performance jump, suggesting that this scale is sufficient to preserve the key local geometry used by DRIVE. Further increasing $k$ to $256$ or $512$ brings only marginal gains, while the training cost rises sharply. Therefore, $k=128$ emerges as the practical elbow point, offering the best trade-off between approximation fidelity and efficiency, and thus serves as a well-justified default choice in DRIVE.

\section{Limitations}
\label{sec:limitations}

While DRIVE demonstrates consistent gains across vision models and LLMs, we note several limitations. First, our formal guarantees are established in linear and locally linearized nonlinear regimes. The local nonlinear analysis extends the purely linear case by retaining explicit Hessian remainder control, but it still relies on smoothness and local-approximation assumptions, and is valid only within a bounded neighborhood of the pre-trained weights. Extending these results to the fully non-convex global dynamics of large transformers remains an open challenge; our support beyond the local regime is currently empirical.

Second, to ensure tractability, we employ the Johnson--Lindenstrauss projection and estimate diversity with a finite FIFO buffer. Although the projection preserves dominant low-rank pairwise structures, it may attenuate diversity signals when the effective intrinsic dimension is high. The FIFO buffer further captures local rather than global diversity and may miss distant redundant samples in the data stream. Moreover, the method's effectiveness relies on reasonable alignment between the pre-trained representation and the target domain.

Finally, diversity introduces computational overhead from weighted covariance and determinant estimation. While dimensionality reduction keeps this cost moderate (e.g., $\sim$17\% for LLMs) in our reported settings, further acceleration would be beneficial for strictly constrained training budgets.

\section*{Broader Impacts}

DRIVE improves the efficiency and reliability of data-oblivious fine-tuning by assigning adaptive sample weights, which can reduce unnecessary training effort and make model adaptation more accessible. At the same time, more efficient fine-tuning may also lower the barrier to harmful, biased, or low-quality downstream adaptation. Our experiments use public datasets and standard benchmarks, and do not collect user data or involve human-subject studies. 

\section*{Reproducibility}

We will release anonymized code, configuration files, training scripts, and evaluation scripts for reproducing the main tables and diagnostic experiments. All reported experiments use public datasets and publicly available base models, with implementation details provided in the main text and appendices.

\section*{Licenses and Assets}

We use publicly available datasets, benchmarks, and base models, following their respective licenses and terms of use. The proposed method does not introduce a new dataset or release new model checkpoints. The released assets consist of anonymized code, configurations, and scripts for reproducing the experiments.


\end{document}